\documentclass[mnsc,nonblindrev]{informs3} 
\OneAndAHalfSpacedXI 

\usepackage{endnotes}
\usepackage[ruled]{algorithm}
\usepackage{algorithmicx,algpseudocode,amsfonts,subfigure}
\usepackage{graphicx,hyperref,comment,appendix}
\usepackage[capitalize,noabbrev]{cleveref}
\usepackage{tikz}
\usepackage{subfigure,caption,microtype,booktabs} 

\usepackage{amsthm,amsmath,amssymb,soul,bm}
\usepackage{comment,enumitem}
\setlist{leftmargin=*}
\usepackage{soul}

\newtheorem{assumption}{Assumption}
\newtheorem{condition}{Condition}
\newtheorem{theorem}{Theorem}
\newtheorem{corollary}{Corollary}
\newtheorem{lemma}{Lemma}
\newtheorem{proposition}{Proposition}
\newtheorem{remark}{Remark}

\renewcommand{\Re}{\mathbb{R}}
\newcommand{\ie}{\textit{i.e.}}
\newcommand{\satexwk}{\texttt{SELECTwK}}
\newcommand{\eg}{\textit{e.g.}}
\newcommand{\E}{\mathbb{E}}
\renewcommand{\Pr}{\text{Pr}}
\newcommand{\regret}{\texttt{Regret}}

\usepackage{natbib}
\bibpunct[, ]{(}{)}{,}{a}{}{,}%
\def\bibfont{\small}%
\usepackage{bm,comment}

\ECRepeatTheorems

\EquationsNumberedThrough    

\allowdisplaybreaks
\begin{document}


\RUNAUTHOR{}

\RUNTITLE{Satisficing Regret Minimization in Bandits with Knapsacks}



\TITLE{Satisficing Regret Minimization in Bandits with Knapsacks}

\ARTICLEAUTHORS{

\AUTHOR{Qing Feng}
\AFF{School of Operations Research and Information Engineering, Cornell University\\
	\EMAIL{qf48@cornell.edu}}

\AUTHOR{Tianyi Ma}
\AFF{School of Operations Research and Information Engineering, Cornell University\\
	 \EMAIL{tm693@cornell.edu}}

\AUTHOR{Ruihao Zhu}
\AFF{SC Johnson College of Business, Cornell University\\
	\EMAIL{ruihao.zhu@cornell.edu}} 

} 

\ABSTRACT{%
We study satisficing regret minimization in bandits with knapsacks (BwK), where a learner sequentially selects randomized policies, \ie, distributions over arms, to obtain mean rewards above a satisficing threshold as frequently as possible while respecting finite resource budgets.
The performance is measured by satisficing regret, the cumulative deficits of the selected policies' mean rewards below this threshold.
Resource constraints introduce two difficulties: an algorithm can incur linear satisficing regret even with small standard regret as deficits to the threshold cannot be compensated, and random consumption can cause early resource depletion even when policies are feasible in expectation.
To address these challenges, we propose \satexwk, a general algorithmic template that uses a BwK learning oracle algorithm with sub-linear expected (standard) regret upper bound to search for and test candidate satisficing policies (\ie, policies whose mean rewards exceed satisficing threshold and whose expected resource consumption is feasible over the time horizon). 
Our regret analysis, supported by an LP-duality argument, establishes that \satexwk~attains constant expected satisficing regret when a satisficing policy exists.
We further introduce a stronger regret condition that bounds the expected sum of positive reward deficits to the optimal one at each individual time step, under which \satexwk~achieves a sharper constant bound together with a sub-linear bound in the time horizon.
As a complement, \satexwk~still preserves the oracle's sub-linear expected standard regret bound if the satisficing threshold is unattainable. 
Finally, we establish the oracle conditions and corresponding satisficing regret bound for existing finite-armed BwK and dynamic pricing algorithms, and conduct numerical experiments on dynamic pricing with limited inventory to demonstrate the empirical performance of \satexwk.
}%



\maketitle

%

\section{Introduction}

Bandits with knapsacks (BwK) provide a framework for sequential decision-making under uncertainty and resource constraints. In BwK, a learner repeatedly selects an arm, receives a noisy reward, and consumes a random amount of each resource. The mean rewards and resource consumptions of the arms are initially unknown and must be learned from the observed outcomes. The learner aims to maximize the expected cumulative reward while keeping the total consumption of each resource within its budget. To achieve this, the learner must balance exploring different arms to learn their outcomes and exploiting the information collected to raise rewards, while accounting for the resources that remain available for future decisions.

While maximizing cumulative reward is a natural objective, in many applications a decision-maker may be satisfied with performance that meets a prescribed target. This preference reflects the notion of \textit{satisficing} \citep{Lufkin21}, coined by \cite{Simon56}, where an acceptable solution is sufficient without requiring the best possible one. In this paper, we study satisficing bandits with knapsacks, where the learner seeks to select policies, \ie, distributions over arms, whose mean rewards exceed a prescribed satisficing threshold as frequently as possible while respecting finite resource budgets. This objective finds its place in a variety of decision-making settings where satisfactory performance must be maintained with limited resources, as illustrated by the following examples:

\begin{itemize}
\item \textbf{Cloud Service:} Consider an operator choosing configurations for a cloud service. In each period, the operator deploys a configuration, such as a resource allocation or batching rule, and observes its service quality and resource consumption. The operator seeks to meet a prescribed service-level objective (SLO), such as meeting 99\% availability measured over a calendar week or having at least 95\% of requests receive a response within 100 milliseconds \citep{Google26}. Configurations meeting this target are acceptable, whereas pursuing higher service quality may require additional resources without providing a meaningful improvement to users.
Moreover, the operator may have a fixed total computing budget over the planning horizon, and a configuration that pursues an optimal service quality may consume resources too quickly and influence decisions in later periods. 

\item \textbf{Inventory Management:} 
Consider a retailer choosing stocking quantities for successive selling periods. Product shortages can lead to lost sales and discourage customers from returning \citep{fitzsimons2000stockout,jing2011stockout}. 
However, improving availability generally imposes pressure on purchasing budget as large orders in earlier periods leave fewer funds for subsequent purchases \citep{Oracle26}. Searching for the optimal stocking quantity that maximizes service level is also very costly. Retailers may thus work toward an acceptable service-level target, such as a 90\%--95\% availability, rather than seeking the highest possible one \citep{levy2023retailing}, while respecting the purchasing budget over the planning horizon. 

\end{itemize}

To measure performance against the targets illustrated by these examples, we use \textit{satisficing regret}, defined as the expected cumulative positive deficit of the selected policies' mean rewards relative to the satisficing threshold.
In the setting without resource constraints, \cite{Michel23} establish constant expected satisficing regret for finite-armed bandits when an arm with mean reward above the threshold exists.
\cite{feng2025satisficing} further show that a blackbox learning oracle algorithm with sub-linear expected standard regret can be used to attain constant expected satisficing regret across different bandit problem classes when the optimal mean reward exceeds the threshold.
When the threshold is unattainable, their algorithm preserves the oracle's expected standard regret bound.

However, extending these results to bandits with knapsacks introduces two difficulties.
First, a resource-feasible algorithm can attain a near-optimal cumulative reward while repeatedly selecting policies (distributions of arms) whose mean rewards fall below the satisficing threshold.
The reason is that high rewards in some periods can compensate for low rewards in other periods when computing standard regret, whereas satisficing regret penalizes each shortfall separately.
Moreover, even identifying a \textit{satisficing policy}, \ie, a policy with mean reward exceeding satisficing threshold and expected resource consumption feasible along time horizon, is not sufficient to ensure constant satisficing regret.
Randomness in resource consumption may still cause resources to be depleted before the end of the horizon.
Therefore, besides identifying policies with satisficing mean rewards, the learner must account for their resource consumption and protect against early depletion.

In this paper, we propose \satexwk, a general algorithmic template that uses a bandits-with-knapsacks learning oracle to identify and test candidate satisficing policies while controlling resource consumption. 
We establish constant expected satisficing regret in the \textit{realizable} case, \ie, satisficing policies exist, under the condition that the learning oracle attains sub-linear expected standard regret, 
and further investigate how a stronger guarantee on the learning oracle improves the constant dependence.

\subsection{Main Contribution}
In this work, we propose \satexwk, a novel algorithmic template for \underline{S}atisficing R\underline{E}gret Minimization via Samp\underline{L}ing and Low\underline{E}r \underline{C}onfidence Bound \underline{T}esting \underline{w}ith \underline{K}napsack. More specifically:
\begin{itemize}
\item In Section~\ref{sec:prob_formulate}, we formulate a satisficing regret minimization problem for bandits with resource constraints. The learner needs to select randomized policies whose mean rewards exceed the satisficing threshold, \ie, a satisficing policy, as frequently as possible while keeping the total consumption of each resource within a resource budget. We define satisficing regret as the expected cumulative deficits of the selected policies' mean rewards to the satisficing threshold, and introduce an LP relaxation to upper bound this optimal expected cumulative reward. In the realizable case, the learner aims to maintain low cumulative satisficing regret. Otherwise, in the non-realizable case, the learner aims to maintain sub-linear standard regret. We further discuss why resource constraints fundamentally change the problem. Sub-linear standard regret does not imply low satisficing regret, since rewards above the LP optimal reward can compensate for low rewards at other time steps. Moreover, always selecting satisficing policies does not ensure constant satisficing regret, since the randomness in resource consumption may cause early resource depletion.

\item We describe the design of \satexwk~in Section~\ref{sec:algorithm_description}. For a given class of bandit-with-knapsack problem and a corresponding learning algorithm with a sub-linear standard regret bound, \satexwk~first adjusts the resource budget with a margin to hedge against consumption randomness, then it runs in rounds with the help of the learning algorithm. In each round, \satexwk~first runs the learning algorithm for a number of time steps and allocate the adjusted resources to the algorithm by time. Then, it computes the average of the policy trajectory and proposes a candidate policy by mixing null arm (which deterministicly incur zero reward and zero resource consumption) into the average policy. After it conducts forced sampling by pulling this candidate policy for a pre-specified number of time steps, it continues to pull this policy and keeps monitoring the lower confidence bound (LCB) of its mean reward and the upper confidence bound (UCB) of its mean resource consumption. \satexwk~terminates the current round and starts the next as long as the LCB falls below the satisficing level or any UCB exceeds the per-time-step resource budget.

\item In Section~\ref{sec:regret_analysis}, we analyze the regret of \satexwk~when the learning oracle has a sub-linear expected standard regret bound. We establish a lemma that uses LP duality to control the expected positive part of the oracle's cumulative standard regret. With this lemma, we show that \satexwk~attains constant expected satisficing regret in the realizable case. In the non-realizable case, \satexwk~attains a sub-linear expected standard regret bound that matches the learning oracle's bound up to logarithmic factors.

\item In Section~\ref{sec:Improved Regret Bound}, we introduce a strong regret condition that controls the expected sum of positive reward deficits at individual time steps relative to the LP optimal reward. Under this stronger condition, we address two limitations of the satisficing regret bound established in Section~\ref{sec:regret_analysis}. Specifically, \satexwk~attains a sharper constant expected satisficing regret bound in terms of the exceeding gap (\ie, the difference between the optimal reward and the threshold value), together with a bound that is sub-linear in the time horizon and remains valid when the exceeding gap is small.

\item In Section~\ref{sec:examples}, we instantiate \satexwk~for three classes of bandit optimization problems with resource constraints: finite-armed bandits with knapsacks, dynamic pricing with continuous prices under inverse demand, and dynamic pricing with finite prices and unknown demand. We establish expected strong regret bounds and the resulting expected constant satisficing regret bound for the first two classes. For the third class, we provide the corresponding bound under the expected standard regret condition.

\item In Section~\ref{sec:synthetic_experiment}, we conduct numerical experiments to demonstrate the performance of \satexwk~on a synthetic instance of dynamic pricing on single resource with selling rate linear in price. We use TS-fixed, an algorithm based on Thompson sampling designed for finite-armed bandits with knapsack proposed by \cite{ferreira2018online}, as the learning oracle and the benchmark. We consider two cases where the optimal reward is reached by the price that exhausts time budget, \ie, the resource constraint does not bind at the optimum, and is reached by the price that exhausts resource budget, \ie, the resource constraint binds at the optimum. Our results reveal that in both cases, \satexwk~does achieve constant expected satisficing regret when it is realizable and its standard regret is comparable to that of the benchmark when it is non-realizable.  

\end{itemize}

\subsection{Related Works}\label{sec:related_work}

The most relevant works are \cite{Michel23} and \cite{feng2025satisficing}, which establish constant expected satisficing regret when the optimal mean reward exceeds the threshold. \cite{Michel23} focus on finite-armed bandits, whereas \cite{feng2025satisficing} propose an algorithmic template that uses a learning oracle with sub-linear standard regret to achieve constant expected satisficing regret across different bandit problem classes. Related finite-armed results include \cite{BubeckPR13} and \cite{GarivierM19}, who establish constant regret under prior knowledge about the optimal mean reward, and \cite{TamatsukuriT19}, who establish constant regret when there is exactly one satisficing arm. Bounded satisficing regret has also been established in multi-objective bandits by \cite{HuyukT21} and finite-state Markov decision processes by \cite{hajiabolhassan2026online}. More recently, \cite{rouyer2024understanding} refine the dependence of satisficing regret on the exceeding gap, and \cite{zhang2026peril} show that constant satisficing regret may no longer be attainable in nonstationary bandits.

The notion of satisficing has also been studied under different regret definitions. An early formulation of satisficing regret, given by \cite{RussoV22}, considers a time-discounted setting. \cite{ReverdyL14,AbernethyAZ16} maximize the probability that realized rewards exceed a threshold, and a more general Bayesian setting is further considered by \cite{ReverdySL16}. Instead, our satisficing regret is defined using mean rewards. The notion of satisficing is also studied in the pure exploration setting through thresholding bandits \citep{LocatelliGC16,MukherjeeNSR17} and good arm identification \citep{KanoHSM19,li2025identifying}. 

Resource constraints complicate the problem as actions will affect the resources available for future decisions, \ie, repeatedly selecting a policy with high mean reward may deplete resources before the end of the time horizon. In the dynamic and stochastic setting, \cite{kleywegt1998dynamic} study whether to accept randomly arriving items when their reward and size distributions are known. Bandit models introduce the additional need to learn rewards and consumption costs: \cite{tran2012knapsack} consider unknown rewards with known pulling costs, whereas \cite{ding2013multi} allow the costs to be random and unknown. \cite{combes2015bandits} further derive regret lower bounds and asymptotically optimal algorithms for bandits with budget.

\cite{BKS18} introduce the general bandit with knapsack (BwK) framework, in which each arm pull generates a random reward and consumes multiple resources. A randomized policy over arms is used to balance reward and resource consumption, and an LP relaxation is then introduced. Algorithm with near-optimal worst-case regret is provided by \cite{BKS18}, and \cite{flajolet2015logarithmic} and \cite{li2021symmetry} further establish instance-dependent logarithmic regret bounds for BwK. Beyond cumulative regret, \cite{sankararaman2021bandits} bound the number of rounds in which mean reward of the selected policy falls below the optimal one by a given amount. In our work, we study cumulative satisficing regret relative to a given threshold and establish a constant expected bound in realizable cases.

Our work is also closely related to the literature on dynamic pricing with unknown demand and limited inventory. In the single-product setting, \cite{besbes2009dynamic} and \cite{babaioff2015dynamic} study demand learning with limited inventory, while \cite{wang2014close} close the performance gap between parametric and nonparametric learning. With multiple products, pricing decisions must also account for shared resource capacities, leading to the blind network revenue management problem studied by \cite{besbes2012blind}. In the multi-product setting, \cite{chen2019nonparametric} develop a nonparametric self-adjusting control, while \cite{ferreira2018online} combine Thompson sampling with linear programming for network revenue management. Primal-dual methods provide another approach to guiding resource allocation. \cite{chen2022primal} develop a primal-dual learning algorithm for personalized pricing with an inventory constraint, and \cite{miao2024demand} introduce demand balancing in primal-dual optimization for blind network revenue management. While these works use dual variables to guide resource allocation, we use LP duality to bound the expected negative part of cumulative standard regret.

\section{Problem Formulation}\label{sec:prob_formulate}

\subsection{Notations}
 We define $[n]$ to be the set $\{1,2,\ldots,n\}$ for any positive integer $n.$ For $k\in [1, \infty]$, we use $\| x\|_k$ to denote the $\ell_k$-norm of a vector $ x\in\Re^d.$ We use $\log_a(\cdot)$ to denote the logarithm with base $a.$ When $a$ is left unspecified, this is the natural logarithm. We adopt the asymptotic notations $O(\cdot),\Omega(\cdot),$ and $\Theta(\cdot)$ as defined in \citep{CLRS09}. When logarithmic factors are omitted, we write them as $ \tilde{O}(\cdot),\tilde{\Omega}(\cdot),$ and $\tilde{\Theta}(\cdot)$. We also use $\textup{polylog}(y)$ to denote a polynomial in $\log(y)$. With some abuse, these notations are used when we try to avoid the clutter of writing out constants explicitly. The indicator variable is denoted as $\bm{1}[\cdot].$ 

 \subsection{Model}

\noindent\textbf{Bandits with Resource Constraints.} Consider a class of bandit optimization problems with resource constraints $(\mathcal{X},\mathcal{R},\mathcal{C})$. Here, $\mathcal{X}$ is the arm set, resources are indexed by $[d]$, $\mathcal{R}\subseteq\{r:\mathcal{X}\to\mathbb{R}\}$ is the class of admissible reward functions, and $\mathcal{C}\subseteq\{c:\mathcal{X}\to\mathbb{R}^d \}$ is the class of admissible expected resource consumption functions. Pulling arm $X\in\mathcal{X}$ in a time step brings a random reward with mean $r(X)$ and a random consumption with mean $c_j(X)$ for each resource $j\in[d]$. The learner knows $r\in\mathcal{R}$ and $c\in\mathcal{C}$, but the exact reward function $r$ and resource consumption function $c$ is initially unknown to the learner.

At the beginning of the time horizon with length $T$, the learner has $B_j$ units of resource $j$ for any $j\in[d]$. In each time step $t\in[T]$, the learner decides a \textit{policy}, \ie, a distribution of arms $\pi_t\in\Delta(\mathcal{X})$. (Here we use $\Delta(\mathcal{X})$ to denote the set of all distributions over the arm set $\mathcal{X}$.) Then an arm $X_t$ is sampled from the distribution $\pi_t$, and the learner pulls $X_t$, and observes and receives an i.i.d. noisy reward $Y_t\in[0,1]$ with mean $r(X_t)$. At the same time, for each $j\in[d]$, an i.i.d. resource consumption $C_t^j(X_t)\in[0,1]$ with mean $c_{j}(X_t)$ is deducted from the remaining resource $j$. If the total consumption of any resource exceeds the initial supply, the learner has to set $\pi_t=\bm{1}_{0}$, \ie, deterministically pull the null arm that incurs zero reward and zero resource consumption.

\noindent\textbf{Satisficing and Standard Regret.} For simplicity, we use $r(\pi)=\mathbb{E}_{X\sim\pi}[r(X)]$ to denote the mean reward of a policy $\pi$, and $c_j(\pi)=\mathbb{E}_{X\sim\pi}[c_j(X)]$ as the mean consumption for resource $j$ of a policy $\pi$. Similar to the setting without resource constraints, in the setting with resource constraints, the learner is given a satisficing threshold value $S$, and the goal of the learner is to select policies with mean reward above the threshold $S$ as frequently as possible using limited resources. To achieve this goal, we measure the performance of the algorithm using \textit{satisficing regret}, defined as the expected cumulative deficit of the mean reward of the policy against the threshold value $S$. The formal definition of satisficing regret is given by
\begin{equation*}
\regret_S=\E\left[\sum_{t=1}^T \max\{S-r(\pi_t),0\}\right].
\end{equation*}

\begin{remark}[Definition of Satisficing Regret with Resource Constraints]

In the settings of BwK, we define the satisficing regret using the mean reward of policies instead of that of individual arms, which is different from the satisficing regret defined for the setting without resource constraints. The reason for this is as follows: The general goal of satisficing regret minimization with resource constraints is to identify satisficing policies whose mean reward is above the satisficing threshold value $S$ and its mean resource consumptions are within the initial resource budget divided by $T$. In many cases, satisficing policies cannot be achieved by deterministically pulling a single arm, thus it is often necessary for satisficing policies to randomize over multiple arms. Our results can be transformed to the satisficing regret defined based on mean reward of individual arms under certain condition, \eg, when the arm set $\mathcal{X}$ and reward functions $r$ are convex, as assumed in~\cite{babaioff2015dynamic}.

\end{remark}

Besides satisficing regret, we also consider the standard regret for bandits with resource constraints. Since the optimal dynamic policy is difficult to derive even with full knowledge of reward and resource consumption, we define the per-time-step LP relaxation on the value of optimal dynamic policy, given by
\begin{equation}
\begin{aligned}
Z^*(T,\bm{B})=&\underset{\pi\in\Delta(\mathcal{X})}{\max}\ r(\pi)\\
&\text{s.t.}\ c_j(\pi)\leq B_j/T,\ \forall j\in[d].
\label{eq:problem_formulation-lp}
\end{aligned}
\end{equation}
Lemma 3.1 in \cite{BKS18} shows that $T\cdot Z^*(T,\bm{B})$ is an upper bound on the cumulative expected reward of the optimal dynamic policy. Following notations in the literature of multi-armed bandits and BwK, we define the standard (cumulative) regret as the difference between the expected reward of an algorithm and the upper bound on the optimal expected reward, given by
\begin{equation}\label{eq:standard_regret}
\regret=T\cdot Z^*(T,\bm{B})-\E\left[\sum_{t=1}^T r(\pi_t)\right].
\end{equation}

If $S\leq Z^*(T,\bm{B})$, then there exists a policy $\pi$ such that $r(\pi)\geq S$ and $c_j(\pi)\leq B_j/T$ for all $i\in[d]$, thus we consider this case \textit{realizable}, and in this case our goal is to attain a constant satisficing regret. Otherwise if $S>Z^*(T,\bm{B})$, then such policy does not exist, thus we consider this case \textit{non-realizable}, and in this case our goal is to maintain a sub-linear standard regret.

We assume that the budget for resource scales linearly with the time horizon. Such assumption is common in the literature of BwK \citep{BKS18, lyu2023bandits}.

\begin{assumption}[Linear in $T$ Resource]
There exists a constant $b>0$ such that $B_j\geq bT$ for all $j\in[d]$.
\label{assum:linear-resource}
\end{assumption}

\subsection{Model Discussion} \label{sec:model_discussion}

After introducing the problem formulation, we would like to provide some discussions on the impact of resource constraints to the problem. Specifically, with the presence of resource constraints, some properties that are straightforward in the unconstrained setting no longer hold with the presence of resource constraints, suggesting that the additional resource constraints fundamentally change the problem thus require new analysis.

{\noindent\bf 1. Satisficing Regret Is Not a Weaker Regret Notion than Standard Regret.}
In the setting without resource constraints, satisficing regret is upper bounded by standard regret whenever the satisficing level does not exceed the optimal mean reward. Specifically, letting $r^*=\max_{X\in\mathcal{X}}r(X)$, we have $\max\{S-r(\pi_t),0\}\leq r^*-r(\pi_t)$ for every selected policy $\pi_t$ when $S\leq r^*$. Therefore, a sub-linear expected standard regret bound directly implies a sub-linear satisficing regret bound. However, this relation no longer holds with the presence of resource constraints. The reason for this is that the optimal reward $Z^*(T,\bm{B})$ in the LP relaxation does not upper bound the mean reward of every policy selected by a feasible algorithm. Specifically, a feasible algorithm only needs to satisfy the cumulative resource constraints and may select a policy $\pi$ with $r(\pi)>Z^*(T,\bm{B})$ at individual time steps, even though such a policy has $c_j(\pi)>B_j/T$ for at least one resource $j$ and is therefore infeasible for the LP relaxation. When computing standard regret, mean rewards above $Z^*(T,\bm{B})$ in these time steps can compensate for low mean rewards in other time steps along the trajectory. In contrast, satisficing regret penalizes each shortfall below the threshold $S$ separately.

For instance, consider an even time horizon $T$, a single resource with budget $B=T/2$, and two arms: an arm with deterministic reward and resource consumption both equal to $1$, and the null arm with zero reward and zero resource consumption. The optimal reward in the LP relaxation is $Z^*(T,B)=1/2$, and we set the satisficing threshold to $S=1/4$, so that the instance is realizable. Consider an algorithm that selects the null arm at every odd time step and the other arm at every even time step. This algorithm consumes exactly $T/2$ units of resource, with the budget depleted only at the final time step, and attains a cumulative reward of $T/2$. Therefore, its standard regret is zero. However, each of the $T/2$ selections of the null arm incurs a satisficing regret of $1/4$, resulting in a total satisficing regret of $T/8$. This example shows that even an algorithm with zero standard regret can incur linear satisficing regret in the realizable case.

{\noindent\bf 2. Always Selecting Satisficing Policies Does Not Imply Constant Regret.} In the setting without constraints (\eg, \citealt{feng2025satisficing}), the algorithm only needs to select satisficing arms, \ie, arms whose mean rewards are above the threshold $S$, as frequently as possible. However, only focusing on selecting arms or policies with satisficing mean rewards does not work with the presence of resource constraints. This is because even if the mean reward of a policy is satisficing, the expected resource consumption of the policy may be large, and always selecting such policies may cause some resource to be depleted early, forcing the learner to select the null arm afterwards and potentially incurring linear in $T$ regret. Therefore under resource constraints, the learner has to keep track of not only the mean reward but also the mean resource consumption of a policy. We define a policy $\pi$ to be satisficing w.r.t. resource $\bm{B}$ if its mean reward $r(\pi)\geq S$, and its expected resource consumption does not exceed $\bm{B}/T$, \ie, $c_j(\pi)\leq B_j/T$ for all $j\in[d]$.

While intuitively the learner needs to select satisficing policies (w.r.t. resource $\bm{B}$) as frequently as possible, we would like to remark that due to randomness in resource consumption, always selecting satisficing policies w.r.t. $\bm{B}$ does not guarantee a constant satisficing regret. The reason for this is that even though always selecting satisficing policies w.r.t. $\bm{B}$ can satisfy the resource constraints in expectation, the randomness in resource consumption may cause some resource to be depleted early with positive probability. Furthermore, since the variance of resource consumption often scale with $T$
, the expected number of time steps when some resource is depleted can scale with $T$, resulting in a satisficing regret that also scales with $T$. This is in sharp contrast to the setting without resource constraints, where always selecting satisficing arms incurs zero satisficing regret. Therefore we conclude that the randomness in resource consumption brings further challenge to the problem, thus besides identifying and selecting satisficing policies w.r.t. resources $\bm{B}$, new techniques are needed to hedge against randomness in resource consumptions.

\section{\satexwk: Satisficing Regret Minimization with Resource Constraints}\label{sec:algorithm_description}

In this section, we present our algorithm template \satexwk~for satisficing regret minimization in bandits with resource constraints. We first provide a detailed description of our algorithm in Section~\ref{sec:algorithm_template}. Then, based on the discussion of why our problem cannot be solved by existing algorithms, we will discuss the novelty of our algorithm in Section~\ref{sec:algorithm_novelty}.

\subsection{Algorithm Template}\label{sec:algorithm_template}

Similar to \cite{feng2025satisficing}, in order to efficiently identify candidate policies with low regret, we assume access to a learning algorithm $\texttt{ALG}(\cdot)$ for the corresponding class of bandits with knapsack (BwK) problem $(\mathcal{X},\mathcal{R},\mathcal{C})$. At a high level, we assume that under any instance of the problem class $(\mathcal{X},\mathcal{R},\mathcal{C})$, any time horizon $t$ and resource $\bm{B}$, the learning algorithm $\texttt{ALG}(t,\bm{B})$ incurs a regret sub-linear in $t$. The formal definition is provided in Condition~\ref{condition:std_regret_bound}.

\begin{condition}[Standard Regret Bound]
For problem class $(\mathcal{X},\mathcal{R},\mathcal{C})$, any time horizon $t\geq 1$ and any $\bm{B}$ such that $B_j\geq bt/2$ for all $j\in[d]$, there exists a sequence of learning algorithms $\texttt{ALG}$ such that the regret of $\texttt{ALG}(t,\bm{B})$ under any reward function $r\in\mathcal{R}$ is upper bounded by
\begin{equation*}
\E\left[t\cdot Z^*(t,\bm{B})-\sum_{s=1}^t r(\pi_s)\right] \leq C_1t^{\alpha}\log^{\beta}(t).
\end{equation*}
where $C_1\geq 1,1/2\leq\alpha<1,\beta\geq 0$ are constants only dependent on the problem class.
\label{condition:std_regret_bound}
\end{condition}
\noindent However, directly extending the analysis based on Condition~\ref{condition:std_regret_bound} in \cite{feng2025satisficing} to our setting with resource constraints is challenging, for which we will provide a detailed discussion in the forthcoming \cref{sec:regret_analysis}.

Given a class of BwK problem $(\mathcal{X},\mathcal{R},\mathcal{C})$ and its learning algorithm $\texttt{ALG}(t,\bm{B})$, we are now ready to introduce our algorithm template \satexwk. Specifically, \satexwk~runs in the following steps.

\noindent\textbf{Step 0. Pre-processing:} To hedge against consumption randomness for the whole algorithm, we first compute the adjusted resource $B'_j=B_j-\tilde{\Theta}(\sqrt{T})$. For the rest of the paper, we say a policy $\pi$ to be satisficing if $\pi$ is satisficing w.r.t. the adjusted resources $\bm{B}'$, \ie, $r(\pi)\geq S$ and $c_j(\pi)\leq B_j'/T$ for all $j\in[d]$, unless otherwise stated.

After pre-processing, \satexwk~runs in rounds to iteratively identify and test candidate satisficing policies. In each round $i$, \satexwk~proceeds with the following steps:

\noindent\textbf{Step 1. Identify a Candidate Satisficing Policy:} Letting $\gamma_i=2^{-i(1-\alpha)/\alpha}$, \satexwk~runs the learning algorithm $\texttt{ALG}(t_i,\bm{b}_i)$ for the first $t_i=\lceil \gamma_i^{-1/(1-\alpha)} \rceil = \lceil2^{i/\alpha} \rceil$ steps of round $i$ with resources $\bm{b}_{i}=\bm{B}'\cdot t_i/T$. After observing the policy selection of $\texttt{ALG}(t_i,\bm{b}_i)$, \satexwk~identifies a candidate satisficing policy using the following two steps. First, the algorithm computes an average of the policies selected by $\texttt{ALG}(t_i,\bm{b}_i)$. Then, a candidate satisficing policy $\hat{\pi}_i$ is computed by taking the mixture of the average policy with the null arm $0$;

\noindent \textbf{Step 2. Forced Sampling on the Candidate Policy} To validate whether $\hat{\pi}_i$ is truly a satisficing policy, we need to collect enough data from pulling $\hat{\pi}_i$ before performing any statistical test. However, because $\hat{\pi}_i$ is newly derived, it is likely that no data about $\hat{\pi}_i$ is collected when running the learning algorithm. To circumvent this, we perform a forced sampling on $\hat{\pi}_i$ by selecting it for $T_i=\lceil \gamma^{-2}_i\rceil$ times, and in this way we collect $T_i$ samples of realized reward of the policy $\hat{\pi}_i$; 

\noindent\textbf{Step 3. Lower Confidence Bound and Upper Confidence Bound Tests:} In this step, \satexwk~kicks off the LCB-UCB test as more data is collected from pulling $\hat{\pi}_i$, where we use LCB tests to validate whether $r(\hat{\pi}_i)\geq S$, and use UCB tests to validate whether $c_j(\hat{\pi}_i)\leq B_j'/T$ for all $j\in[d]$. We use $k$ (initialized to 0) to denote the number of additional samples of noisy reward generated by selecting $\hat{\pi}_i$ up to the current time step, $\hat{r}^{\text{tot}}_i$ to denote the running total reward and $\hat{c}_{ji}^{\text{tot}}$ to denote the running total resource consumption of resource $j$ collected from Step 2 and the current step so far. At the same time, it keeps comparing the LCB of $\hat{\pi}_i$'s mean reward, given by $\hat{r}^{\text{tot}}_i/(T_i+k) - \sqrt{4\log(T_i+k)/(T_i+k)}$, against the threshold value $S$, and also compares the UCB of $\hat{\pi}_i$'s mean resource consumption for each $j\in[d]$, given by $\hat{c}^{\text{tot}}_{ji}/(T_i+k) + \sqrt{4\log(d(T_i+k))/(T_i+k)}$, against the $B_j'/T$, after acquiring every piece of new data. \satexwk~terminates this round and enters the next whenever either the LCB of reward is less than $S$, or the UCB of any resource $j$ consumption is more than $B_j'/T$.

\begin{algorithm}[!ht]
\SingleSpacedXI
\caption{General Satisficing Bandits with Knapsacks (\satexwk)}
\label{alg:randomized_s_regret}
\begin{algorithmic}
\State \textbf{Input:} time horizon $T$, satisficing level $S$, resources $\bm{B}$ for $j\in[d]$, null arm $0$
\State Set $\gamma_i=2^{-i(1-\alpha)/\alpha}$, $B_j'=B_j-\frac{10\alpha}{1-\alpha}\sqrt{T\log(2dT^2)}$
\For{$i=1,2,\dots$}
\State Set $t_i\leftarrow\lceil 1/\gamma_i^{1/(1-\alpha)}\rceil$, $\bm{b}_i=\bm{B}'\cdot {t_i}/{T}$, $T_i\leftarrow\lceil\gamma_i^{-2}\rceil$
\State Run $\texttt{ALG}(t_i,\bm{b}_i)$, let $\{\tilde{\pi}_s\}_{s\in[t_i]}$ be the sequence of policy pulls
\State Compute
\begin{equation*}
w_i=\min\left\{1, \frac{10T}{\min_jB_j'}\sqrt{\frac{\log(8dT_i)}{T_i}}\right\}
\end{equation*}
\State Set 
\begin{equation*}
\hat{\pi}_i=\dfrac{1-w_i}{t_i}\sum_{s=1}^{t_i} \tilde{\pi}_s+ w_i\bm{1}_{0}
\end{equation*}
\State Set $T_i\leftarrow\lceil\gamma_i^{-2}\rceil$, let $t$ be the current time step, and pull policy $\hat{\pi}_i$ for the next $T_i$ time steps
\State Set $k\leftarrow0$, $\hat{r}_i^\text{tot}\leftarrow\sum_{s=t}^{t+T_i-1}Y_s$, and $\hat{c}_{ji}^\text{tot}\leftarrow\sum_{s=t}^{t+T_i-1}C_s^j$ for $j\in[d]$
\While{$LCB_r(\hat{\pi}_i)\geq S$ and $UCB_j(\hat{\pi}_i)\leq B'_j/T$ for $j\in[d]$}
    \State Set $k\leftarrow k+1$
    \State Pull $\hat{\pi}_i$, then observe reward $Y_t$ and costs $C^j_t$, where $t$ is the current time step
    \State Update $\hat{r}_i^\text{tot}\leftarrow\hat{r}_i^\text{tot}+Y_t$, and $\hat{c}_{ji}^\text{tot}\leftarrow\hat{c}_{ji}^\text{tot}+C_t^j$ for $j\in[d]$
    \State Set
        \begin{equation*}
        \begin{aligned}
            LCB_r(\hat{\pi}_i) &\leftarrow \frac{\hat{r}_i^\text{tot}}{T_i+k} - \sqrt{\frac{4\log(T_i+k)}{T_i+k}}\\
            UCB_j(\hat{\pi}_i) &\leftarrow \frac{\hat{c}_{ji}^\text{tot}}{T_i+k} + \sqrt{\frac{4\log(d(T_i+k))}{T_i+k}}
        \end{aligned}
        \end{equation*}
\EndWhile
\EndFor
\end{algorithmic}
\end{algorithm}

\subsection{Novelty of \satexwk}\label{sec:algorithm_novelty}

In this section, we discuss the key novelties of our algorithm. We first discuss why existing algorithms such as \cite{feng2025satisficing} cannot effectively handle our setting with resource constraints. \cite{feng2025satisficing} proposes an algorithmic framework for satisficing regret minimization without resource constraints. Specifically, the algorithm iteratively identifies a candidate satisficing arm by randomly drawing an arm from the trajectory of a bandit learning algorithm. Then the algorithm validates whether the candidate arm is truly satisficing by first conducting forced sampling and then run LCB tests on the arm's mean reward. However, the additional resource constraints bring additional challenges to the problem, both in terms of identifying satisficing policies and validating satisficing policies, and existing algorithms no longer work in this setting.

{\noindent\bf 1. Challenges of Identifying Satisficing Policies:} In the setting with resource constraints, a natural idea is still to make use of a learning algorithm with sub-linear expected (standard) regret in the corresponding BwK problem class. However, uniformly drawing a policy selected by the learning algorithm as in \cite{feng2025satisficing} is clearly unable to guarantee finding a satisficing policy w.r.t. resource $\bm{B}$ with large enough probability. The reason for this is that the learning algorithm may select policies with high reward and high resource consumptions along with others with low reward and low resource consumptions. While algorithm achieves near-optimal cumulative reward with limited resources, it is likely that none of the policies is satisficing in terms of reward and resource consumption at the same time.

One immediate alternative way of identifying candidate satisficing policies is to take the average over the policies selected by the learning algorithm. Unfortunately, this approach is also unable to guarantee finding a satisficing policy w.r.t. resource $\bm{B}$. The reason for this is that any learning algorithm for BwK only guarantees the \textit{realized} resource consumptions to be within the designated resources, and does not guarantee the cumulative mean resource consumptions of the selected policies to also be within corresponding resources. When the realized resource consumptions happen to be significantly lower than the cumulative mean consumptions, directly taking the average over the selected policies may result in excessive resource consumptions in the average policy. Therefore directly taking the average of the selected policies does not guarantee obtaining a satisficing policy with large enough probability.

As a side note, a sub-linear standard regret bound does not directly guarantee that the averaged policy has a satisficing mean reward with sufficiently high probability, since the standard regret of an oracle trajectory can be negative. To bypass this challenge, we need to control the expected positive part of the oracle's cumulative standard regret. In the forthcoming \cref{sec:regret_analysis}, we establish this control under the same oracle condition with the help of an LP-duality argument.

{\noindent\bf 2. Challenges of Validating Satisficing Policies:} To ensure fast removal of policies that do not satisfy the resource constraints, \ie, $c_j(\pi)>B_j/T$ for some $j\in[d]$, we need to conduct UCB tests on the mean consumption of each resource $j\in[d]$ of the candidate policy. When the resource consumption of a policy is marginally smaller than $B_j/T$ for some $j\in[d]$, the policy can be easily removed by the UCB tests even though it satisfies the resource constraints. Therefore besides identifying a satisficing policy, we would also like to ensure the resource consumptions of the policy to be lower than $\bm{B}/T$ by a significant margin, so that with large enough probability, the policy will not be removed in the UCB tests. The stronger requirements for the mean resource consumption present even more challenge in constructing a proper candidate satisficing policy, and existing literature is unable to effectively resolve this challenge.

\noindent\textbf{3. Challenges of Exploiting Satisficing Policies:} Finally, as mentioned in \cref{sec:model_discussion}, due to randomness in resource consumption, always selecting a satisficing policy  w.r.t. $\bm{B}$ does not guarantee a constant satisficing regret bound. Even if the algorithm is able to quickly identify a satisficing policy w.r.t. $\bm{B}$, always selecting the policy until depletion of some resource could still incur a regret that scales in $T$ because of possibly early depletion of resources.

In light of all the challenges mentioned above, we implement multiple new techniques in \satexwk. In what follows, we provide a detailed discussion on these new techniques in each step of the algorithm.

\noindent\textbf{Step 0.} In the pre-processing step, \satexwk~reserves $\tilde{\Theta}(\sqrt{T})$ units of resource to hedge against randomness in resource consumption. Specifically, we first computes the adjusted resource $B_j'=B_j-\tilde{\Theta}(\sqrt{T})$, and for the rest of the algorithm, we search for and validate satisficing policies w.r.t. $\bm{B}'$ instead of the original resource capacities $\bm{B}$. Since the realized resource consumption deviates from the cumulative mean consumption of selected policies by $\tilde{O}(\sqrt{T})$ with high probability, as long as the algorithm maintains an expected cumulative resource consumption within $\bm{B}'$, we can ensure that the realized resource consumption does not exceed $\bm{B}$ with high probability. Therefore by using a more conservative target resource capacity $\bm{B}'$, we resolve our third challenge above and guarantee that resources will not be depleted with high probability.

\noindent\textbf{Step 1.} In order to efficiently identify satisficing policies with minimal regret, we use a learning algorithm with sub-linear expected (standard) regret in the corresponding BwK problem class. Specifically, in round $i$, we run the oracle in $t_i$ time steps, the resource capacity allocated to the oracle is $t_i/T\cdot \bm{B}'$, proportional to the number of time steps of running the oracle. 

To generate a candidate policy from the trajectory of the oracle, \satexwk~first takes the average of the policies in the trajectory. Then a candidate policy is computed by taking the mixture of the average policy and the null arm $\bm{1}_{0}$, with the weight of the null arm being $w_i=\tilde{\Theta}(\gamma_i)$. Here, taking the mixture with the null arm $\bm{1}_{0}$ mainly serves the two purposes.
\begin{enumerate}
\item As mentioned in the first challenge above, directly taking the average of selected policies may result in excessive expected resource consumption of the constructed policy, and by mixing the average policy with the null arm, we reduce the expected resource consumption of the constructed policy and effectively reduce the possibility of having excessive expected resource consumption of the constructed policy.
\item By mixing the null arm, we are able to ensure that the expected resource consumption of the candidate satisficing policy to be smaller than $\bm{B}'/T$ by a margin of at least $\tilde{\Omega}(\gamma_i)$ with a large enough probability. With this we are able to resolve the second challenge above, since a policy whose expected resource consumption smaller than $\bm{B}'/T$ is more likely to survive in the LCB-UCB test if the policy is truly satisficing.
\end{enumerate}
Finally, as $i$ gets large, $\gamma_i$ converges to $0$, and the weight of the null arm $w_i$ will also converge to $0$. This suggests that as the rounds proceed, mixing the null arm will have a diminishing impact on the expected revenue of the candidate policy.

\noindent\textbf{Step 2.} In order to collect enough information on policy reward and resource consumption, we conduct a forced sampling of the candidate policy until the confidence interval is of order $\tilde{\Theta}(T_i^{-1/2})=\tilde{\Theta}(\gamma_i)$. Similar to Step 2 of the algorithm in \cite{feng2025satisficing}, the purpose of this step is to collect enough data for the candidate policy, so that a small enough confidence interval on the policy's expected revenue and resource consumption can be constructed. In this way, the policy will be less likely to be removed by the upcoming LCB-UCB tests if it is truly satisficing w.r.t. resources $\bm{B}'$.

\noindent\textbf{Step 3.} In order to validate whether the candidate policy $\hat{\pi}_i$ is truly satisficing, we conduct a series of LCB tests for the policy's mean reward against the satisficing threshold $S$, and UCB tests for the policy's mean resource consumptions against $\bm{B}'/T$. Either when the LCB of the expected revenue falls below $S$, or the UCB of the expected consumption of any resource $j\in[d]$ exceeds $B_j'/T$, the algorithm immediately terminates the current round and starts a new round. Intuitively, the length of running the oracle $t_i$ increases as the rounds proceed, the expected revenue of the candidate policy $r(\hat{\pi}_i)$ will gradually converge to $Z^*(T,\bm{B}')$, thus if $Z^*(T,\bm{B}')>S$, the candidate policy will have a higher probability of having an expected revenue above $S$ under mild conditions of the learning oracle (we will provide a detailed discussion on this in the forthcoming \cref{sec:regret_analysis}). In terms of resource consumption, by construction in Step 1, the mean resource consumptions of the candidate policy $\hat{\pi}_i$ is smaller than $\bm{B}'/T$ by a margin of at least $\tilde{\Omega}(\gamma_i)$ with a large enough probability. Furthermore, in Step 2, we have collected $T_i$ samples of realized resource consumptions of $\hat{\pi}_i$, the length of the confidence interval of mean resource consumption is $\tilde{\Theta}(1/\sqrt{T_i})=\tilde{\Theta}(\gamma_i)$. Therefore with a large enough probability, the UCB of mean resource consumptions will not exceed $\bm{B}'/T$, thus the probability of the candidate being removed due to excessive resource consumption is small. 

\section{Regret Analysis}\label{sec:regret_analysis}

In this section, we analyze the regret of \satexwk~under Condition~\ref{condition:std_regret_bound}, which requires the learning oracle to attain a sub-linear expected standard regret bound. We show that \satexwk~attains constant expected satisficing regret in the realizable case, and sub-linear expected standard regret in the non-realizable case.

To establish these results, we need to first bound the expected positive reward deficit of the candidate policy relative to the LP optimal reward. In the setting without resource constraints, \cite{feng2025satisficing} show that Condition~\ref{condition:std_regret_bound} is sufficient to attain constant expected satisficing regret in the realizable case. However, directly extending the analysis in \cite{feng2025satisficing} to our setting with resource constraints is challenging. In the setting without resource constraints, the standard regret of every oracle trajectory is nonnegative. Therefore, a sub-linear expected standard regret bound can be used together with Markov's inequality to show that the candidate arm has a satisficing mean reward with high probability once the oracle runs for sufficiently many time steps. In our setting, however, the standard regret $R_t=tZ^*(t,\bm{B})-\sum_{s=1}^t r(\pi_s)$ of some trajectories can be negative. Consequently, when taking expectation, a very large positive regret on some trajectories can be offset by negative regret on others, and then \(\mathbb{E}[\max\{R_t,0\}]\) cannot be directly bounded based on an upper bound on $\mathbb{E}[R_t]$. Thus Markov's inequality is no longer applicable. To address this difficulty, we need to control the contribution of trajectories with negative standard regret.

We make use of the resource constraints satisfied by the learning oracle to control this contribution. Intuitively, if the average of cumulative reward along an oracle trajectory exceeds the LP optimal reward, the cumulative mean consumption along this trajectory must exceed the allocated budget in at least one resource. Since the realized resource consumption remains within the allocated budget, this can occur only when the realized consumption of that resource falls below its cumulative mean consumption. Using LP duality, we relate the magnitude of negative standard regret to these deviations in resource consumption and bound its expected contribution by leading term $\sqrt{t_i}$. Combining this bound with Condition~\ref{condition:std_regret_bound} gives an upper bound on the expected positive part of standard regret, as stated in the following lemma.

\begin{lemma}
Under Condition~\ref{condition:std_regret_bound}, for every round $i$,
\begin{equation*}
\E\left[\max\left\{\sum_{s=1}^{t_i}r(\tilde\pi_s) - t_iZ^*(T,\bm B'),0\right\}\right]\leq \frac{T}{\min_jB_j'}\sqrt{t_i},
\end{equation*}
therefore, 
\begin{equation*}
\E\left[\max\left\{t_iZ^*(T,\bm B') - \sum_{s=1}^{t_i}r(\tilde\pi_s),0\right\}\right]\leq C_1t_i^\alpha\log^\beta(t_i)+\frac{T}{\min_jB_j'}\sqrt{t_i}.
\end{equation*}
\label{lemma:positive_standard_regret}
\end{lemma}

The proof of Lemma~\ref{lemma:positive_standard_regret} is provided in Section~\ref{sec:positive_standard_regret} of the Appendix. When $T$ is sufficiently large, the lemma implies that the expected positive part of cumulative standard regret is bounded by $\tilde{O}(t_i^\alpha)$. Consequently, after averaging the oracle trajectory and accounting for the reward loss from mixing in the null arm, the candidate policy is obtained and satisfies
$$
\mathbb{E}\!\left[\max\left\{Z^*(T,\bm{B}')-r(\hat{\pi}_i),0\right\}\right]
\leq \tilde{O}(t_i^{\alpha-1}).
$$
We can then apply Markov's inequality to this nonnegative quantity. Specifically, when $Z^*(T,\bm{B}')>S$, the mean reward of the candidate satisficing policy $r(\hat{\pi}_i)$ exceeds the satisficing threshold $S$ with high probability for any large enough $i$. Therefore, with the help of Lemma~\ref{lemma:positive_standard_regret}, Condition~\ref{condition:std_regret_bound} alone is sufficient to obtain the reward guarantee for the candidate policy needed for our analysis. In the realizable case, we define the exceeding gap $\Delta_S^*$ as the gap between the per-time-step upper bound $Z^*(T,\bm{B})$ and the satisficing threshold $S$, given by 
$$
\Delta_S^*=Z^*(T,\bm{B})-S.
$$

\noindent{\bf Satisficing Regret Bound in the Realizable Case under Condition~\ref{condition:std_regret_bound}:} We use the following theorem to provide a constant satisficing regret bound when the learning oracle has a sub-linear expected standard regret bound (\ie, Condition~\ref{condition:std_regret_bound}).

\begin{theorem}\label{thm:real_bwk_weak}
    Under Condition~\ref{condition:std_regret_bound}, if $Z^*(T,\bm{B})\geq S$, then the satisficing regret of \satexwk~is upper bounded by 
    \begin{equation}\label{eq:sat_bwk_weak}
        \texttt{Regret}_S\leq {C_1}^{\frac{2-\alpha}{1-\alpha}} \left(\dfrac{1}{\Delta_S^*}\right)^{\frac{1}{1-\alpha}} \cdot \textup{polylog}\left(\dfrac{C_2d}{\Delta_S^*}\right).
    \end{equation}
\end{theorem}

\begin{proof}[Proof Sketch.]
A complete proof of \cref{thm:real_bwk_weak} is provided in Section~\ref{sec:thm:real_bwk_weak} of the Appendix. To prove the theorem, we proceed with the following steps.

\noindent\textbf{Step 1. Sufficient Resources with High Probability:} We first prove that \satexwk~does not deplete any resource with high probability. We define the event $\mathcal{E}$ as 
\begin{equation*}
    \mathcal{E}_i= \left\{\forall\;j\in[d],\, \hat{c}_{ji}^{\text{tot}}(\tau_i)\leq\frac{\tau_iB_j'}{T} + 2\sqrt{T_i\log(2dT^2)}\right\},\quad \mathcal{E}=\bigcap_{i\leq i_{\max}}\mathcal{E}_i.
\end{equation*}
Here, $\tau_i$ is the total number of time steps of forced sampling and LCB-UCB test in round $i$, $i_{\max}$ is the total number of rounds in time horizon $T$, and we use $\hat{r}^{\text{tot}}_i(t)$ and $\hat{c}_{ji}^{\text{tot}}(t)$ as the realized total reward and cost after repeatedly selecting $\hat{\pi}_i$ for $t$ time steps. We use the following proposition to prove that if the event $\mathcal{E}$ holds, then no resource will be depleted when running \satexwk. Intuitively, this is because the algorithm reserves $\tilde{O}(\sqrt{T})$ resources to hedge against randomness in resource consumption, thus as long as the realized resource consumption does not significantly exceed the expected consumption of the selected policies, as assumed in $\mathcal{E}$, then no resource will be depleted by the end of the time horizon.
\begin{proposition}\label{prop:clean_event}
    If $\mathcal{E}$ holds, then for every resource $j\in[d]$, the total resource consumption by the end of the time horizon is upper bounded by
    \begin{equation*}
        \sum_{i=1}^{i_{\max}} (b_{ji} + \hat{c}_{ji}^{\text{tot}}(\tau_i))\leq B_j.
    \end{equation*}
\end{proposition}
Here $b_{ji}=B'_jt_i/T$ is the allocated resource budget. Proof of this proposition is provided in Section~\ref{sec:prop:regret_bound_Sat_clean} of the Appendix. Then, we use the following proposition to show that $\mathcal{E}$ holds with probability at least $1-1/T$. 
\begin{proposition}\label{prop:clean_event_bound}
    The event $\mathcal{E}$ holds with probability at least $1-1/T$.
\end{proposition}
Proof of this proposition is provided in Section~\ref{sec:prop:regret_bound_Sat_cleanbound} of the Appendix.

\noindent\textbf{Step 2. Satisficing Regret Bound of Each Round:} We show an upper bound on the satisficing regret of round $i$, conditioned on the event $\mathcal{E}$ that resources are not depleted.
\begin{proposition}\label{prop:round_Sat_regret_weak}
    Under Condition~\ref{condition:std_regret_bound}, if $Z^*(T,\bm{B}')>S$, then conditioned on $\bigcap_{i'<i}\mathcal{E}_{i'}$, the expected satisficing regret incurred in round i is upper bounded by
    \begin{equation*}
        \left( \frac{6C_1+6}{(1-\alpha)^{\max\{\beta,1/2\}}} + \frac{33\sqrt{2}T}{\min_jB_j'}+2 \right)\gamma_i^{-1 / (1 - \alpha)}\log(4d/\gamma_i)^{\max\{\beta,1/2\}}.
    \end{equation*}
\end{proposition}
Proof of this proposition is provided in Section~\ref{sec:prop:round_Sat_regret_weak} of the Appendix. 

\noindent\textbf{Step 3. Exit Probability of Each Round:}
Denoting $\Delta_S':=Z^*(T,\bm{B}')-S$, we show that conditioned on $\mathcal{E}$, once \satexwk~runs for large enough number of rounds so that $\gamma_i=\tilde{O}(\Delta_S')$, the UCB-LCB test is unlikely to terminate a round and start a new round.
\begin{proposition}
    Under Condition~\ref{condition:std_regret_bound}, if $Z^*(T,\bm{B}')\geq S$, then conditioned on $\bigcap_{i'<i}\mathcal{E}_{i'}$, for every $i$ that satisfies
    \begin{equation*}
        64\left( \frac{2C_1+2}{(1-\alpha)^{\max\{\beta,1/2\}}} + \frac{11\sqrt{2}T}{\min_jB_j'} \right) \gamma_i\log(4d/\gamma_i)^{\max\{\beta,1/2\}}\leq\Delta_S',
    \end{equation*}
    the probability that round $i$ ends within the time horizon $T$ (conditioned on the event that round $i$ is started within the time horizon) is upper bounded by $1/8$.
    \label{prop:transit_Sat_prob_weak}
\end{proposition}
Since each round of \satexwk~runs independently, Proposition~\ref{prop:transit_Sat_prob_weak} indicates that the total number of rounds for \satexwk~can be upper-bounded by a shifted geometric random variable with success rate $7/8$. Proof of this proposition is provided in Section~\ref{sec:prop:transit_Sat_prob_weak} of the Appendix.

\noindent\textbf{Step 4. Finish the Proof:} 
To finish the proof, we first bound the gap between optimal rewards under the original resource constraint and the adjusted one.
\begin{proposition}\label{prop:optimal_gap}
    For any $T$, the difference between optimal reward $Z^*(T,\bm{B}')$ of allocating resource $B_j'$ and optimal reward $Z^*(T,\bm{B})$ of allocating resource $B_j$ is upper bounded by
    \begin{equation*}
        Z^*(T,\bm{B}) - Z^*(T,\bm{B}') \leq \frac{10\alpha}{b(1-\alpha)} \sqrt{\frac{\log(2dT^2)}{T}}.
    \end{equation*}
\end{proposition}
Proof of this proposition is provided in Section~\ref{sec:prop:optimal_gap} of the Appendix.


Combining \crefrange{prop:clean_event}{prop:optimal_gap}, we are able to establish the statement in Theorem~\ref{thm:real_bwk_weak}. We first consider the case where $T$ is sufficiently large that $10\alpha/(b(1-\alpha)) \cdot \sqrt{{\log(2dT^2)}/{T}} \leq {\Delta_S^*}/{2}$. By Proposition~\ref{prop:optimal_gap}, we have
\[
\Delta_S'
=Z^*(T,\boldsymbol B')-S
\geq \Delta_S^*-\frac{10\alpha}{b(1-\alpha)}
\sqrt{\frac{\log(2dT^2)}{T}}
\geq \frac{\Delta_S^*}{2}.
\]
Moreover, the definition of $\bm B'$ implies that $\min_j B_j'\geq bT/2$. Therefore, the term $T/\min_j B_j'$ in Propositions~\ref{prop:round_Sat_regret_weak} and \ref{prop:transit_Sat_prob_weak} is bounded by $2/b$, independently of $T$.

We first bound the contribution to expected satisficing regret from $\mathcal E$. Denote $i_0+1$ as the smallest integer such that the condition in Proposition~\ref{prop:transit_Sat_prob_weak} holds for every $i\geq i_0+1$. We show that $2^{i_0(1-\alpha)/\alpha}=1/\gamma_{i_0}=\tilde O(C_1/\Delta_S')$. Then Proposition~\ref{prop:round_Sat_regret_weak} implies that the contribution from the first \(i_0\) rounds is upper bounded by
$$
\sum_{i=1}^{i_0} \tilde O\!\left(C_1\gamma_i^{-\frac{1}{1-\alpha}}\right)
= \tilde O\!\left(C_1\sum_{i=1}^{i_0}2^{i/\alpha}\right)
= \tilde O\!\left(C_12^{i_0/\alpha}\right)
= \tilde O\!\left( C_1^{\frac{2-\alpha}{1-\alpha}} \left(\frac{1}{\Delta_S'}\right)^{\frac{1}{1-\alpha}} \right).
$$
Then we bound the contribution from rounds after \(i_0\). By Proposition~\ref{prop:transit_Sat_prob_weak}, the probability that \satexwk~enters round \(i_0+k\) and \(\mathcal E\) holds is at most \(8^{1-k}\). Therefore, by Proposition~\ref{prop:round_Sat_regret_weak}, the contribution from these rounds is upper bounded by
\begin{equation*}
\begin{aligned}
&\sum_{k=1}^{\infty} 8^{1-k}\,
O\!\left(C_1 \gamma_{i_0+k}^{-\frac{1}{1-\alpha}} \log\!\left(\frac{4d}{\gamma_{i_0+k}}\right)^{\max\{\beta,1/2\}} \right)\\
\leq & \tilde O\!\left(C_1 \gamma_{i_0}^{-\frac{1}{1-\alpha}} \right) 
\sum_{k=1}^{\infty} 8^{1-k}2^{k/\alpha}(k+1)^{\max\{\beta,1/2\}}
= \tilde O\!\left(C_1^{\frac{2-\alpha}{1-\alpha}} \left(\frac{1}{\Delta_S'}\right)^{\frac{1}{1-\alpha}} \right).
\end{aligned}
\end{equation*}
Here, the series is finite because \(\alpha\geq1/2\) implies \(8^{1-k}2^{k/\alpha}\leq8\cdot2^{-k}\). 

Combining the two parts, using \(\Delta_S'\geq\Delta_S^*/2\), and noting that Proposition~\ref{prop:clean_event_bound} bounds the contribution from \(\mathcal E^c\) by \(T\Pr(\mathcal E^c)\leq1\), we obtain
\[
\texttt{Regret}_S \leq
1+\tilde O\!\left(C_1^{\frac{2-\alpha}{1-\alpha}} \left(\frac{1}{\Delta_S'}\right)^{\frac{1}{1-\alpha}} \right)
= \tilde O\!\left(C_1^{\frac{2-\alpha}{1-\alpha}} \left(\frac{1}{\Delta_S^*}\right)^{\frac{1}{1-\alpha}} \right).
\]
Finally, if \(T\) is not large enough, then \(T\leq\tilde O((1/\Delta_S^*)^2)\). Since the expected satisficing regret is at most \(T\) and \(1/(1-\alpha)\geq2\), the bound \(\tilde O((1/\Delta_S^*)^{1/(1-\alpha)})\) also holds in this case.
\end{proof}

\cref{thm:real_bwk_weak} states that if the learning oracle has a sub-linear expected standard regret bound, then in the realizable case, \satexwk~attains a constant satisficing regret bound that depends on the exceeding gap $\Delta_S^*$ but independent from the length of the time horizon $T$. In other words, regardless of how long the learner makes decision over, the satisficing regret incurred by the learner is always upper bounded by a constant that only depends on $\Delta_S^*$.

\noindent{\bf Standard Regret Bound in the Non-Realizable Case:} We use the following theorem to show that in the non-realizable case, \satexwk~has a sub-linear standard regret bound as long as the learning oracle has a sub-linear expected standard regret bound.

\begin{theorem}\label{thm:real_bwk_std}
    Under Condition~\ref{condition:std_regret_bound}, if $Z^*(T,\bm{B})< S$, then the standard regret of our algorithm is upper bounded by
    \begin{equation}\label{eq:real_bwk}
        \texttt{Regret}\leq C_1T^{\alpha}\cdot\textup{polylog}(T),
    \end{equation}
\end{theorem}

\cref{thm:real_bwk_std} states that as long as the learning oracle has a $\tilde{O}(T^\alpha)$ expected standard regret bound as assumed in Condition~\ref{condition:std_regret_bound}, in the non-realizable case, \satexwk~is also able to attain a $\tilde{O}(T^\alpha)$ regret bound similar to that of the learning oracle. In other words, \satexwk~is robust in the non-realizable case and can attain a similar performance with the learning oracle used in the algorithm in terms of standard regret minimization. 

\begin{remark}[Limitations under Standard Regret Condition]\label{remark:limitations_std_condition}
Compared with the result of \cite{feng2025satisficing} in the setting without resource constraints, the expected satisficing regret bound of \satexwk~under Condition~\ref{condition:std_regret_bound} in the realizable case is weaker in two aspects. 

\noindent\textbf{Higher Order in $1/\Delta_S^*$}: the order of $1/\Delta_S^*$ increases from $\alpha/(1-\alpha)$ to $1/(1-\alpha)$. The reason for this is that a sub-linear expected standard regret bound does not guarantee a sub-linear expected satisficing regret when running the learning oracle. In particular, an oracle run of $t_i$ time steps could incur $\Theta(t_i)=\Theta(\gamma_i^{-1/(1-\alpha)})$ expected satisficing regret, whereas in the setting without resource constraints, its expected satisficing regret is bounded by $\tilde{O}(t_i^\alpha)=\tilde{O}(\gamma_i^{-\alpha/(1-\alpha)})$. Once the $\gamma_i$ becomes sufficiently small relative to $\Delta_S^*$, this larger regret incurred when running the oracle leads to the worse dependence from $\tilde{O}((1/\Delta_S^*)^{\alpha/(1-\alpha)})$ to $\tilde{O}((1/\Delta_S^*)^{1/(1-\alpha)})$ in the setting with resource constraints.

\noindent\textbf{Linear Regret under Small $\Delta_S^*$}: although the bound is independent of $T$ for any fixed $\Delta_S^*$, \satexwk~does not necessarily attain sub-linear expected satisficing regret when $\Delta_S^*$ is small relative to the time horizon. This is again because Condition~\ref{condition:std_regret_bound} does not rule out the expected satisficing regret growing linearly with the number of time steps spent running the learning oracle. \satexwk~may continue to start new rounds until the length of an oracle run becomes comparable to $T$, therefore incurring $\Theta(T)$ expected satisficing regret. 

\end{remark}

In light of the aforementioned limitations of regret bounds under Condition \ref{condition:std_regret_bound}, in the forthcoming Section~\ref{sec:Improved Regret Bound}, we propose a (slightly) stronger condition on the learning oracle, under which \satexwk~is able to attain an improved constant satisficing regret for fixed $\Delta_S^*$ as well as a sub-linear in $T$ satisficing regret for small $\Delta_S^*$.

\section{Improved Regret Bound}\label{sec:Improved Regret Bound}
In this section, we introduce a strong regret condition that controls the sum of positive reward deficits relative to the LP optimal reward. Based on this strong regret condition, we show that both limitations discussed in Remark~\ref{remark:limitations_std_condition} under Condition~\ref{condition:std_regret_bound} can be addressed: \satexwk~attains a better constant expected satisficing regret bound in terms of $\Delta_S^*$, together with a sub-linear bound in $T$ that remains valid when $\Delta_S^*$ is small.

\subsection{Strong Regret Condition}
Although Lemma~\ref{lemma:positive_standard_regret} provides a sub-linear bound on the expected positive part of cumulative standard regret, it does not directly provide a sub-linear bound on the satisficing regret incurred when running the learning oracle. The reason for this is that the lemma takes the positive part after summing over all time steps, whereas satisficing regret sums the positive deficits at each individual time step. Thus, the bound on the positive part of the sum, given by Lemma~\ref{lemma:positive_standard_regret}, does not provide a corresponding bound on the sum of the positive parts. To obtain a sharper satisficing regret bound, we therefore consider a stronger notion of regret that measures the cumulative positive deficit of the selected policies compared to $Z^*(t,\bm{B})$, given by
\begin{equation}\label{eq:strong_regret}
\sum_{s=1}^t\max\{Z^*(t,\bm{B})-r(\pi_s),0\}.
\end{equation}
We say that an algorithm has a sub-linear expected strong regret bound if the expected cumulative deficit compared to $Z^*(t,\bm{B})$ defined in \eqref{eq:strong_regret} is sub-linear in the time horizon $t$. The formal definition of the condition is provided in Condition~\ref{condition:strong_regret_bound}.

\begin{condition}[Strong Regret Bound]
For any instance of the problem class $(\mathcal{X},\mathcal{R},\mathcal{C})$, any time horizon $t\geq 1$ and any $\bm{B}$ such that $B_j\geq bt/2$ for all $j\in[d]$, the learning oracle $\texttt{ALG}(t,\bm{B})$ satisfies
\begin{equation*}
\E\left[\sum_{s=1}^t \max\left\{Z^*(t,\bm{B})-r(\pi_s),0\right\}\right]\leq C_1t^{\alpha}\log^{\beta}(t).
\end{equation*}
\label{condition:strong_regret_bound}
\end{condition}

Intuitively, Condition~\ref{condition:strong_regret_bound} dose not only require the oracle algorithm to have an asymptotically optimal cumulative reward, but also requires the policies selected by the algorithm to have a stable mean revenue not significantly lower than $Z^*(t,\bm{B})$ throughout the time horizon. In other words, Condition~\ref{condition:strong_regret_bound}~rules out algorithms where a majority of the cumulative reward is generated in only a fraction of the time horizon, and little reward is gained for the rest of the time horizon. While Condition~\ref{condition:strong_regret_bound}~is stronger than Condition~\ref{condition:std_regret_bound}, it is still mild and can be satisfied by many existing BwK algorithms.  For instance, in \cref{sec:strong_regret_bound_existing_algorithm} of the Appendix, we show that Condition~\ref{condition:strong_regret_bound}~ can be satisfied by many existing algorithms including \cite{wang2014close,BKS18}. 

With Condition~\ref{condition:strong_regret_bound}, we are also able to show that the mean reward of the candidate satisficing policy $r(\hat{\pi}_i)$ exceeds the satisficing threshold $S$ with a large enough probability for any large enough $i$. Specifically, by the convexity of $\max\{x,0\}$, Condition~\ref{condition:strong_regret_bound} implies that $\mathbb{E}[\max\{Z^*(t_i,\bm{b}_i)-r(\bar{\pi}_i),0\}]\leq \tilde{O}(t_i^{\alpha-1})$, where $\bar{\pi}_i$ is the average of the selected policies of the algorithm. Then by Markov's inequality, $Z^*(t_i,\bm{b}_i)-r(\bar{\pi}_i)\leq \tilde{O}(t_i^{\alpha-1})$ holds with a probability at least $3/4$. By definition of the candidate satisficing arm $\hat{\pi}_i$ we also have that $Z^*(t_i,\bm{b}_i)-r(\hat{\pi}_i)\leq \tilde{O}(t_i^{\alpha-1})$ holds with a probability at least $3/4$. Since $t_i$ increases as the rounds proceed, when $i$ is large enough, the probability that $r(\hat{\pi_i})\geq S$ will be at least $3/4$. We refer to the forthcoming \cref{prop:transit_Sat_prob} for details.

In brief, Condition~\ref{condition:strong_regret_bound} strengthens Condition~\ref{condition:std_regret_bound} by controlling the expected sum of positive deficits at individual time steps. It enables \satexwk~to attain a better constant expected satisficing regret bound, together with a sub-linear expected satisficing regret bound, to be discussed in Section~\ref{sec:regret_analysis_strong}.

\subsection{Regret Analysis under Strong Regret Condition}\label{sec:regret_analysis_strong}
\noindent{\bf Satisficing Regret Bound under Condition~\ref{condition:strong_regret_bound}:} In what follows, we use the following theorem to show that if we assume the learning oracle to satisfy Condition~\ref{condition:strong_regret_bound}, then \satexwk~attains a constant satisficing regret bound better than the regret bound in \cref{thm:real_bwk_weak}.

\begin{theorem}\label{thm:real_bwk_sat}
Under Condition~\ref{condition:strong_regret_bound}, if $Z^*(T,\bm{B})\geq S$, then the satisficing regret of \satexwk~is upper bounded by
\begin{equation}\label{eq:sat_bwk}
    \texttt{Regret}_S\leq\min\left\{ C_1^{\frac{1}{1-\alpha}}\left(\dfrac{1}{\Delta_S^*}\right)^{\frac{\alpha}{1-\alpha}}\cdot\textup{polylog}\left(\dfrac{C'_2d}{\Delta_S^*}\right),\,C_1T^{\alpha}\cdot\textup{polylog}(T)\right\}.
\end{equation}
\end{theorem}
\begin{proof}[Proof Sketch.]
A complete proof of Theorem~\ref{thm:real_bwk_sat} is given in Section~\ref{sec:thm:real_bwk_sat} of the Appendix. We follow proof pattern similar to Theorem~\ref{thm:real_bwk_weak}. The alternative condition on the learning oracle does not affect the implementation of the algorithm, thus \cref{prop:clean_event} and \cref{prop:clean_event_bound} remain valid in this case. In other words, \satexwk~will not exhaust any resource before the end of the time horizon, conditioned on $\mathcal{E}$, and this bounded resource consumption happens with high probability.

Denoting $\Delta_S':=Z^*(T,\bm{B}')-S$, we then show two critical results. In the first one, similar to Proposition~\ref{prop:round_Sat_regret_weak}, we show an upper bound on the satisficing regret of round $i$.
\begin{proposition}\label{prop:round_Sat_regret}
    Under Condition~\ref{condition:strong_regret_bound}, if $Z^*(T,\bm{B}')\geq S$, then conditioned on $\bigcap_{i'<i}\mathcal{E}_{i'}$, the satisficing regret incurred in round i is upper bounded by
    \begin{equation*}
        4\left(\frac{2C_1}{(1-\alpha)^\beta}+\frac{10\sqrt{2}T}{\min_jB_j'} \right)\gamma_i^{-\alpha / (1 - \alpha)}\log(4d/\gamma_i)^{\max\{\beta,1/2\}}.
    \end{equation*}
\end{proposition}
Proof of this proposition is provided in Section~\ref{sec:prop:round_Sat_regret} of the Appendix. We also show that once our algorithm runs for large enough number of rounds so that $\gamma_s=\tilde{O}(\Delta_S')$, it is unlikely to start a new round.
\begin{proposition}
    Under Condition~\ref{condition:strong_regret_bound}, if $Z^*(T,\bm{B}')\geq S$, then conditioned on $\bigcap_{i'<i}\mathcal{E}_{i'}$, for every $i$ that satisfies
    \begin{equation*}
        64\left(\frac{2C_1}{(1-\alpha)^\beta}+\frac{10\sqrt{2}T}{\min_jB_j'} \right)\gamma_i\log(4d/\gamma_i)^{\max\{\beta,1/2\}}\leq\Delta_S',
    \end{equation*}
    the probability that round $i$ ends within the time horizon $T$ (conditioned on the event that round $i$ is started within the time horizon) is upper bounded by $1/8$.
    \label{prop:transit_Sat_prob}
\end{proposition}
Proof of this proposition is provided in Section~\ref{sec:prop:transit_Sat_prob} of the Appendix. Proposition~\ref{prop:round_Sat_regret} is similar to Proposition~\ref{prop:round_Sat_regret_weak} and Proposition~\ref{prop:transit_Sat_prob} is similar to Proposition~\ref{prop:transit_Sat_prob_weak}. 
Following the same proof pattern in \cref{sec:regret_analysis}, we first consider \(T\) sufficiently large that \(\Delta_S'\geq\Delta_S^*/2\) and \(\min_j B_j'\geq bT/2\). Denote \(i_0+1\) as the smallest integer such that the condition in Proposition~\ref{prop:transit_Sat_prob} holds for every \(i\geq i_0+1\). The same argument as in \cref{sec:regret_analysis} gives \(1/\gamma_{i_0}=\tilde O(C_1/\Delta_S')\). Under the strong regret condition, Proposition~\ref{prop:round_Sat_regret} bounds the expected satisficing regret incurred in round \(i\) by \(\tilde O(C_1\gamma_i^{-\alpha/(1-\alpha)})=\tilde O(C_1 2^i)\). Therefore, combining Proposition~\ref{prop:round_Sat_regret} and \ref{prop:transit_Sat_prob}, the contribution to expected satisficing regret from \(\mathcal E\) is upper bounded by
\[
\begin{aligned}
&\sum_{i=1}^{i_0} \tilde O\!\left(C_1\gamma_i^{-\frac{\alpha}{1-\alpha}}\right)
+\sum_{k=1}^{\infty}8^{1-k} \tilde O\!\left(C_1\gamma_{i_0+k}^{-\frac{\alpha}{1-\alpha}}\right)\\
&=\tilde O(C_1 2^{i_0})
=\tilde O\!\left( C_1^{\frac{1}{1-\alpha}} \left(\frac{1}{\Delta_S'}\right)^{\frac{\alpha}{1-\alpha}} \right).
\end{aligned}
\]
Using \(\Delta_S'\geq\Delta_S^*/2\) and adding the contribution from \(\mathcal E^c\), which is at most \(1\), establishes the constant regret bound $\tilde{O}(C_1^{{1}/{(1-\alpha)}} \left({1}/{\Delta_S^*}\right)^{{\alpha}/{(1-\alpha)}})$ for sufficiently large \(T\).

We further show that the expected satisficing regret is bounded by \(\tilde O(C_1T^\alpha)\). Let \(i_{\max}\) denote the total number of rounds. Since \(t_i=\lceil2^{i/\alpha}\rceil\), we have \(2^{i_{\max}}=O(T^\alpha)\). As Proposition~\ref{prop:round_Sat_regret} also bounds the expected positive reward deficits relative to \(\min\{S,Z^*(T,\boldsymbol B')\}\), we take the difference between $Z^*(T,\boldsymbol B)$ and $Z^*(T,\boldsymbol B')$ by Proposition~\ref{prop:optimal_gap} into consideration and obtain
\[
\begin{aligned}
\texttt{Regret}_S &\leq 1+ T(Z^*(T,\boldsymbol B)-Z^*(T,\boldsymbol B'))
+\sum_{i=1}^{i_{\max}}\tilde O(C_1 2^i)\\
&=\tilde O(\sqrt T)+\tilde O(C_1T^\alpha) =\tilde O(C_1T^\alpha).
\end{aligned}
\]

Finally, if $T$ is not large enough, then $T=\tilde O((1/\Delta_S^*)^2)$. Applying the sub-linear bound gives expected satisficing regret at most $\tilde{O}(C_1T^\alpha)\leq\tilde O(C_1(1/\Delta_S^*)^{2\alpha})\leq \tilde{O}(C_1^{1/(1-\alpha)}(1/\Delta_S^*)^{\alpha/(1-\alpha)})$.
\end{proof}

Since Condition~\ref{condition:strong_regret_bound} implies Condition~\ref{condition:std_regret_bound}, Theorem~\ref{thm:real_bwk_std} guarantees that the \(\tilde{O}(T^\alpha)\) expected standard regret bound in the non-realizable case also holds under Condition~\ref{condition:strong_regret_bound}.

\cref{thm:real_bwk_sat} suggests that if we assume an expected strong regret bound of the learning oracle as in Condition~\ref{condition:strong_regret_bound}, \satexwk~is able to attain a better regret bound than the one established in \cref{thm:real_bwk_weak}, where we assume a sublinear expected standard regret bound. The regret established in \cref{thm:real_bwk_sat} improves from the the one in \cref{thm:real_bwk_weak} in two aspects: first, the constant satisficing regret bound in \cref{thm:real_bwk_sat} is $\tilde{O}((1/\Delta_S^*)^{\alpha/(1-\alpha)})$, a smaller constant rate than the $\tilde{O}(((1/\Delta_S^*)^{1/(1-\alpha)})$ established in \cref{thm:real_bwk_weak}; second, when $\Delta_S^*$ is small, under Condition~\ref{condition:strong_regret_bound}, \satexwk~can still attain a $\tilde{O}(T^\alpha)$ sub-linear satisficing regret bound similar to the regret bound of the learning oracle, while such a sub-linear satisficing regret bound is lacking under Condition~\ref{condition:std_regret_bound} with a very small $\Delta_S^*$.

Such an improvement is mainly due to smaller satisficing regret incurred when running the learning oracle under Condition~\ref{condition:strong_regret_bound}. As we have mentioned, if we assume Condition~\ref{condition:std_regret_bound} for the learning oracle, running the learning oracle for $t_i$ time steps could still incur $\Theta(t_i)$ regret in the realizable case. On the other hand, under Condition~\ref{condition:strong_regret_bound}, we are able to prove that the satisficing regret incurred when running the learning oracle in round $i$ is upper bounded by $\tilde{O}(t_i^{\alpha})$, a significant improvement from the potentially $\Theta(t_i)$ satisficing regret under Condition~\ref{condition:std_regret_bound}. Intuitively, by assuming Condition~\ref{condition:strong_regret_bound}, the learning oracle is able to distribute the gained reward more evenly throughout the time horizon. The condition is able to avoid an undesirable situation where a majority of the reward is gained in a small fraction of the time horizon, while minimal reward is gained for the rest of the time horizon, making a better satisficing regret bound possible when running the learning oracle.

\section{Examples}\label{sec:examples}


In what follows, we instantiate \satexwk~to three classes of bandit optimization problems with knapsacks: finite-armed bandits with knapsacks, dynamic pricing with continuous prices under inverse demand, and dynamic pricing with finite prices and unknown demand. For the first two classes, we establish that the corresponding oracles satisfy the strong regret condition and derive the resulting expected satisficing and standard regret bounds. For the final one, we state the expected satisficing and standard regret bounds based on the standard regret condition. These examples are not exhaustive, and similar regret bounds can be established for other classes of bandit optimization problems with knapsacks.

\noindent\textbf{1. Finite-Armed Bandits with Knapsacks:} 
Consider a finite arm set $\mathcal{X}$ with $K=|\mathcal{X}|$ arms and $d$ resources, where the resource budgets satisfy $B_j\geq bT$ for all $j\in[d]$. 
We use \texttt{BalancedExploration} in \cite{BKS18} as the learning oracle, with resource budgets normalized as described in Section~\ref{sec:balanced_exploration_strong} of the Appendix. 
By adapting its analysis to control the positive reward deficits at individual time steps, we show that this oracle has an expected strong regret bound of $O(\sqrt{dKt}\log(dKt)^2)$ over $t$ time steps for fixed $b$ as established in \cref{thm:strong_BwK} in the Appendix.
In particular, Condition~\ref{condition:strong_regret_bound} holds with $\alpha=1/2$. Combining this result with Theorems~\ref{thm:real_bwk_sat} and \ref{thm:real_bwk_std}, we obtain the following corollary.

\begin{corollary}\label{cor:finite_bwk}
By using \texttt{BalancedExploration} as $\texttt{ALG}$, if $Z^*(T,\bm{B})\geq S$, the expected satisficing regret of \satexwk~is bounded by
\begin{equation*}
\regret_S \leq 
\min\left\{ \frac{dK}{\Delta_S^*} \cdot\operatorname{polylog}\!\left(\frac{dK}{\Delta_S^*}\right),
\;
\sqrt{dKT}\cdot\operatorname{polylog}(T) \right\}.
\end{equation*}
If $Z^*(T,\bm{B})<S$, then the expected standard regret of \satexwk~is bounded by $\regret \leq \sqrt{dKT} \cdot \operatorname{polylog}(T)$.
\end{corollary}

\noindent\textbf{2. Dynamic Pricing with Limited Inventory:} Consider a seller with $B\geq bT$ units of a single product to sell over $T$ time steps. At each time step, the seller offers a price $p\in[\underline{p},\bar{p}]$, and a purchase occurs with probability $\lambda(p)$. Thus, the mean reward and resource consumption are $r(p)=p\lambda(p)$ and $c(p)=\lambda(p)$, respectively, with rewards normalized as in our model. We assume that demand is decreasing, demand and inverse demand (the inverse function of $\lambda(\cdot)$) are Lipschitz continuous, and revenue as a function of demand is twice differentiable and strictly concave, with curvature bounded away from zero. The precise regularity assumptions are stated in Proposition~\ref{proposition:CtG_assumption} of the Appendix.

We use a discrete-time version of \texttt{DPA} in \cite{wang2014close} as the learning oracle. This version replaces the continuous-time Poisson arrivals in the original model with Bernoulli purchases at discrete time steps. In Section~\ref{sec:discrete_dpa_strong} of the Appendix, we describe the adaptation and establish an expected strong regret bound of $O(\sqrt{t}\log(t)^6)$. Therefore, Condition~\ref{condition:strong_regret_bound} holds with $\alpha=1/2$ and $\beta=6$. Together with Theorems~\ref{thm:real_bwk_sat} and \ref{thm:real_bwk_std}, this gives the following corollary.

\begin{corollary}\label{cor:dynamic_pricing}
By using \texttt{Discrete DPA} as $\texttt{ALG}$, if $Z^*(T,B)\geq S$, the expected satisficing regret of \satexwk~is bounded by
\begin{equation*}
\regret_S \leq
\min\left\{ \frac{1}{\Delta_S^*} \cdot\operatorname{polylog}\!\left(\frac{1}{\Delta_S^*}\right),\;
\sqrt T\cdot\operatorname{polylog}(T) \right\}.
\end{equation*}
If $Z^*(T,B)<S$, then the expected standard regret of \satexwk~is bounded by $\regret\leq \sqrt{T} \cdot\operatorname{polylog}(T)$.
\end{corollary}

\noindent\textbf{3. Thompson Sampling with Fixed Inventory Constraints:} 
Consider dynamic pricing with limited inventory and a finite set of \(K\) admissible prices. We use TS-fixed by \cite{ferreira2018online} as the learning oracle. At each time step, TS-fixed samples demand parameters from their posterior distribution and solves an LP with fixed per-period inventory constraints to obtain a randomized pricing policy. Although we do not establish a strong regret bound for TS-fixed, a sub-linear expected standard regret bound satisfying Condition~\ref{condition:std_regret_bound}, \ie, standard regret upper bound $O(\sqrt{Kt}\operatorname{polylog}(Kt))$, is sufficient to obtain the following guarantees from Theorems~\ref{thm:real_bwk_weak} and \ref{thm:real_bwk_std}.
\begin{corollary}
By using TS-fixed as the learning oracle, if \(Z^*(T,B)>S\), the expected satisficing regret of \satexwk~is bounded by
$$
\regret_S \leq \frac{K^{3/2}}{(\Delta_S^*)^2} \cdot \operatorname{polylog}\!\left(\frac{K}{\Delta_S^*}\right).
$$
If \(Z^*(T,B)<S\), then its expected standard regret is bounded by $\regret \leq \sqrt{KT}\cdot\operatorname{polylog}(KT)$.
\end{corollary}

\section{Numerical Experiments}\label{sec:synthetic_experiment}
In this section, we conduct numerical experiments to test the performance of \satexwk~on dynamic pricing with limited supply. All experiments are run on a PC with 4-core CPU and 16GB of RAM.

We use the setting given by \cite{ferreira2018online}. There are $T$ buyers, equivalent to time budget, arriving one by one and making a decision whether to buy an item after comparing their private valuation to observed price. There are totally $B$ items, which is equivalent to resource budget with $d=1$. For each price $p\in[0,10]$, the expected selling rate $c(p)\in[0,1]$, which is equivalent to the expected resource consumption, is realized to be $0$ or $1$ based on buyer's decision.  

There are two cases where $Z^*(T,\textbf{B})$ is reached by exhausting resource or time budget \citep{besbes2009dynamic, wang2014close}. For the former one, the best price is attained at the boundary of the constraints, \emph{i.e.}, to set a price that sells all items before the end of the time horizon and apply $p_\infty$ (the null price) till the end. For the latter one, the best price is attained strictly within the feasible region, \emph{i.e.}, to set the price with highest $pc(p)$ through the time horizon and no resource will be exhausted.

We consider an instance of $c(p)=1-0.1p$ with available price set $\{1,2,\cdots,9,10 \}$, then highest per-buyer (per-timestep) reward is attained as $2.5$ for $p=5$. The budget $B$ increases linearly in $T$, and we control the value of $B/T$ to simulate for different cases.

\subsection{Exhausting Time Budget}
\textbf{Setup:} In this case, we set $B/T=0.6$, then the optimal reward $Z^*(T,\bm{B})=2.5$ is reached by price $p^*=\arg\max pc(p)$ that exhausts time budget instead of resource budget with high probability. The length of the time horizon is set to be $500$ to $6000$ with a stepsize of $500$. We conduct experiments for both the realizable case and the non-realizable case. For the realizable case, we set the satisficing level $S=1$; for the non-realizable case, we set the satisficing level $S=5$. In both cases, we compare \satexwk~against TS-fixed in \cite{ferreira2018online}. The experiment is repeated for 1000 times and we report the average results.

\textbf{Results:} The results of the realizable case are presented in Figure~\ref{fig:BwK_exhaust_time_sat}, and the results of the non-realizable case is presented in Figure~\ref{fig:BwK_exhaust_time_non}. From Figure~\ref{fig:BwK_exhaust_time_sat}, one can see that the satisficing regret of \satexwk~stabilizes at a constant level after $T$ exceeds 500, implying that \satexwk~attains a constant expected satisficing regret in the realizable case. At the same time, the satisficing regret of TS-fixed keeps increasing with $T$ and fails to be stable at a constant. Comparing \satexwk~and TS-fixed, the satisficing regret incurred by \satexwk~is also smaller than TS-fixed.

As for the non-realizable case, one can see from Figure~\ref{fig:BwK_exhaust_time_non} that \satexwk~achieves better empirical performance than TS-fixed. This is because \satexwk~exploits more aggressively on near-optimal arms during forced sampling while TS-fixed keeps exploring with posterior sampling.

\begin{figure}[!ht]
    \centering
    \subfigure[Satisficing Regret in the Realizable Case]{\includegraphics[height=4.3cm]{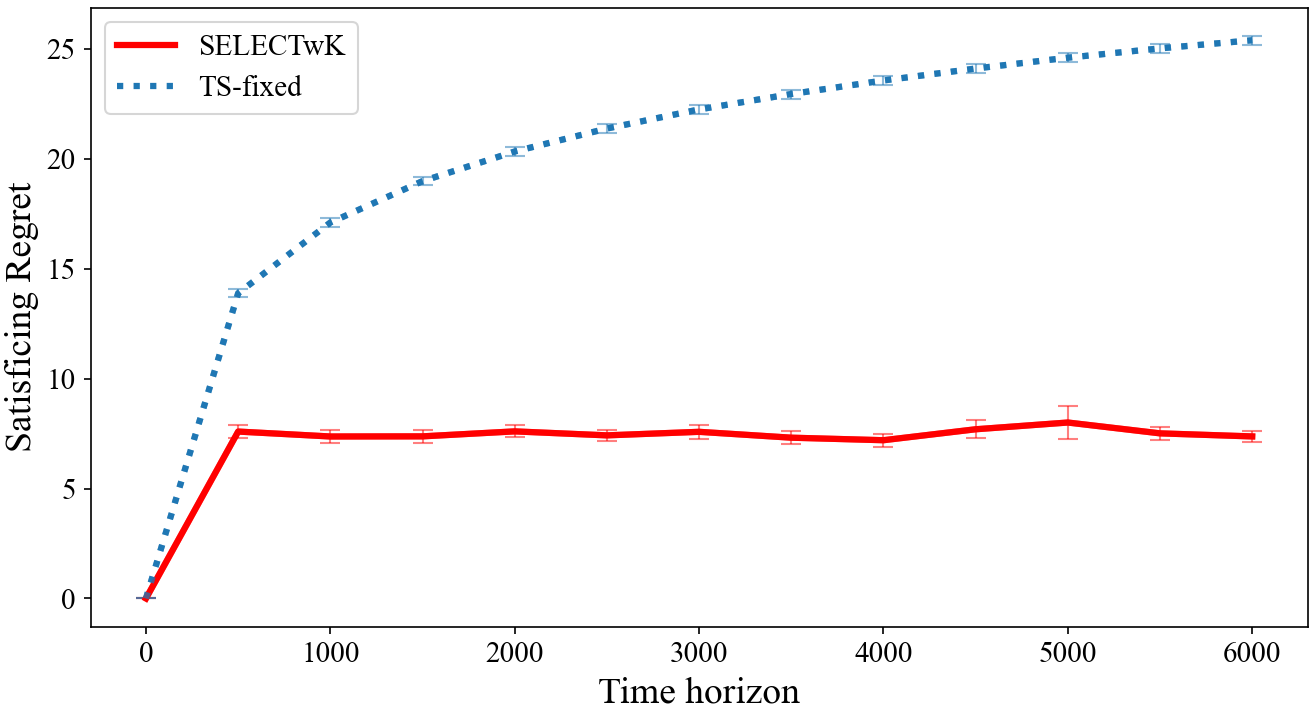}
    \label{fig:BwK_exhaust_time_sat}}
    \subfigure[Standard Regret in the Non-Realizable Case]{\includegraphics[height=4.3cm]{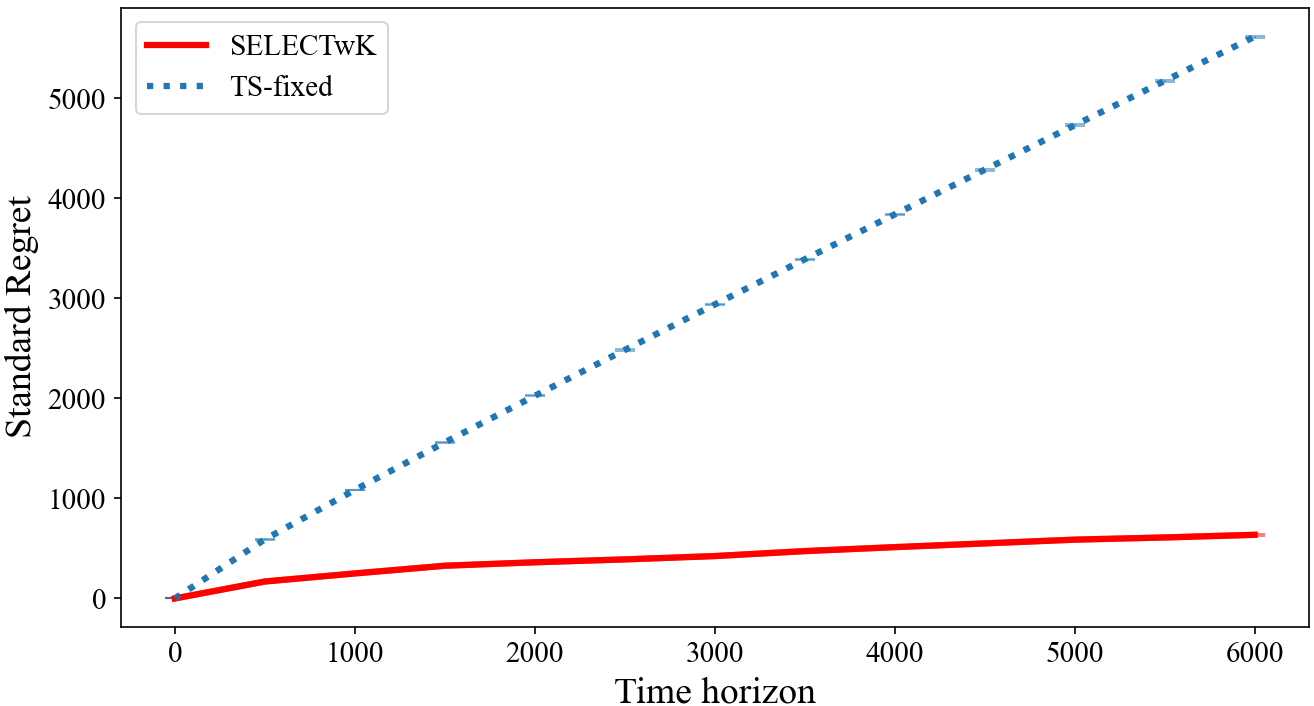}
    \label{fig:BwK_exhaust_time_non}}
    
    \caption{Comparison of \satexwk~and TS-fixed under exhausting time budget case}
\end{figure}

\subsection{Exhausting Resource Budget}
\textbf{Setup:} In this case, we set $B/T=0.4$, then the optimal reward $Z^*(T,\bm{B})=2.4$ is reached by price $p=\arg\min|c(p)-B/T|=6$ which has selling rate $c(p)=0.4$ that will exhaust the resource budget. The length of the time horizon is set to be $500$ to $6000$ with a stepsize of $500$. We conduct experiments for both the realizable case and the non-realizable case. For the realizable case, we set the satisficing level $S=1$; for the non-realizable case, we set the satisficing level $S=5$. In both cases, we compare \satexwk~against TS-fixed in \cite{ferreira2018online}. The experiment is repeated for 1000 times and we report the average results.

\textbf{Results:} The results of the realizable case is presented in Figure~\ref{fig:BwK_exhaust_resource_sat}, and the results of the non-realizable case is presented in Figure~\ref{fig:BwK_exhaust_resource_non}. From Figure~\ref{fig:BwK_exhaust_resource_sat}, one can see that \satexwk~stabilizes at a constant level after $T$ exceeds 1000, implying that \satexwk~attains a constant expected satisficing regret in the realizable case. At the same time, the satisficing regret of TS-fixed keeps increasing with $T$ and fails to be stable at a constant. Comparing \satexwk~and TS-fixed, \satexwk~incurs satisficing regret that grows similarly to that of TS-fixed but then smaller than TS-fixed after being stable.

As for the non-realizable case, one can see from Figure~\ref{fig:BwK_exhaust_resource_non} that \satexwk~achieves better empirical performance than TS-fixed. Similar to the case of $B/T=0.6$, this is because \satexwk~exploits more aggressively on near-optimal arms during forced sampling while TS-fixed keeps exploring with posterior sampling.

\begin{figure}[!ht]
    \centering
    \subfigure[Satisficing Regret in the Realizable Case]{\includegraphics[height=4.3cm]{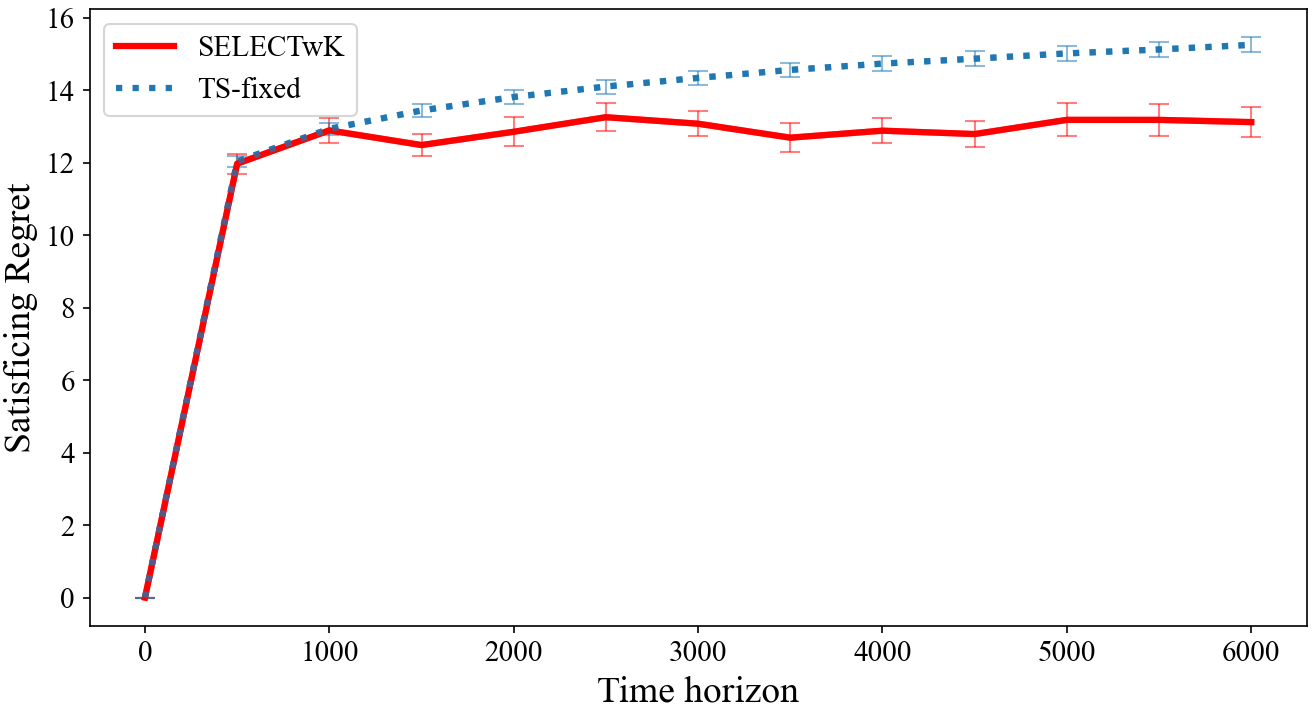}
    \label{fig:BwK_exhaust_resource_sat}}
    \subfigure[Standard Regret in the Non-Realizable Case]{\includegraphics[height=4.3cm]{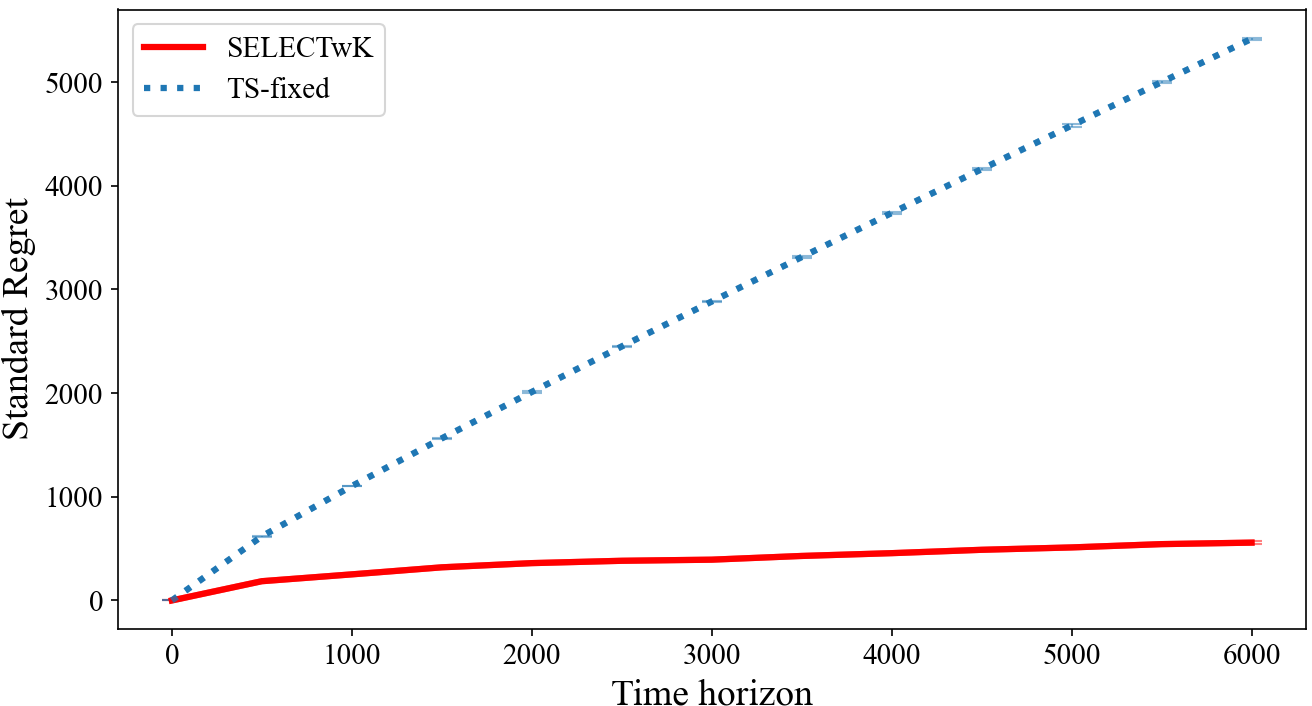}
    \label{fig:BwK_exhaust_resource_non}}
    
    \caption{Comparison of \satexwk~and TS-fixed under exhausting resource budget case}
\end{figure}

\section{Conclusion}
In this paper, we study satisficing regret minimization in bandits with knapsacks and propose \satexwk, a general algorithmic template for maintaining mean rewards above a satisficing threshold while respecting finite resource constraints. 
We show that a bandit-with-knapsack learning oracle with a sub-linear expected standard regret bound is sufficient for \satexwk~to attain constant expected satisficing regret in the realizable case. 
When the threshold is unattainable, \satexwk~preserves the oracle's expected standard regret. 
We further introduce a strong regret condition that controls the expected sum of positive reward deficits at individual time steps relative to the optimal feasible mean reward. 
Under this condition, \satexwk~achieves a sharper dependence on the exceeding gap, together with a satisficing regret bound that is sub-linear up to logarithmic factors in the time horizon.
We instantiate \satexwk~to finite-armed bandits with knapsacks and dynamic pricing with resource constraints.
Finally, numerical experiments on dynamic pricing with limited inventory support our theoretical findings and demonstrate the empirical performance of \satexwk.

Several directions remain for future research. 
First, advances in machine learning, particularly large language models, motivate extending our framework to settings with ML-generated proxy rewards from these models \citep{ji2025multi,ma2026best}, then an important question is how to leverage these proxy rewards to improve the satisficing regret. 
Second, our analysis assumes stationary setting. Future studies could investigate how to deal with non-stationary reward and resource consumption distributions \citep{cheung2022hedging,liu2022non}.

\bibliography{ref}
\bibliographystyle{ormsv080}

\newpage

\begin{appendices}{\Large \noindent\textbf{Appendix}}

\section{Proof of Theorem~\ref{thm:real_bwk_weak}}\label{sec:thm:real_bwk_weak}
In the Appendix, for notational simplicity we denote $Z^*(T,\bm{B})$ as $Z^*$ and denote $Z^*(T,\bm{B}')$ as $Z'$.

We first analyze the expected satisficing regret of this execution, assuming $\Delta_S^{\prime}=Z^{\prime}-S>0$ and conditioned on $\mathcal{E}$. If $T$ is sufficiently large with $\sqrt{T/\log(2dT^2)}\geq20\alpha/(b(1-\alpha))$, term $11\sqrt{2}T/(\min_jB_j')$ can be upper bounded by
\begin{equation}
\begin{aligned}
\frac{11\sqrt{2}T}{\min_jB_j'} &= 11\sqrt{2}\cdot\left(\min_j \frac{B_j}{T} - \frac{10\alpha}{1-\alpha}\sqrt{\frac{\log(2dT^2)}{T}}\right)^{-1}\\
&\leq11\sqrt{2}\cdot\left(b - b/2 \right)^{-1} = \frac{22\sqrt{2}}{b}.
\label{eq:budgetj_bound}
\end{aligned}
\end{equation}

Denote $\beta'=\max\{\beta,1/2\}$ and $K_{\alpha, \beta} = 64\left(\frac{2C_1+2}{(1-\alpha)^{\max\{\beta,1/2\}}} + \frac{22\sqrt{2}}{b} \right) \geq 64\left( \frac{2C_1+2}{(1-\alpha)^{\max\{\beta,1/2\}}} + \frac{11\sqrt{2}T}{\min_jB_j'} \right)$. Round $i_0$ can be expressed as 
$$
i_0=\max\left\{i:64\left(\frac{2C_1+2}{(1-\alpha)^{\max\{\beta,1/2\}}} + \frac{11\sqrt{2}T}{\min_jB_j'} \right)\gamma_i\log(4d/\gamma_i)^{\beta'}>\Delta_S'\right\}.
$$
Next, we prove an upper bound on $i_0$ . Denote $\delta_0\in(0,1)$ the smallest solution to the following
equation
\begin{equation}
\delta_0 \log(4d/\delta_0)^{\beta'} = \frac{\Delta_S'}{64\left(\frac{2C_1+2}{(1-\alpha)^{\max\{\beta,1/2\}}} + \frac{11\sqrt{2}T}{\min_jB_j'} \right)}.
\label{eq:delta_Delta_equation}
\end{equation}
Note that $\delta_0\log(4d/\delta_0)^{\beta'}$ is continuous on $(0,1)$. When $\delta_0\to0$, $\delta_0 \log(4d/\delta_0)^{\beta'} \to 0$, and its maximum is $\max_{\delta_0}\delta_0\log(4d/\delta_0)^{\beta'} \geq \delta_0\log(2/\delta_0)^{\beta'}\big|_{\delta_0=2/e} > 1/2$. On the other hand, since $64\left(\frac{2C_1+2}{(1-\alpha)^{\max\{\beta,1/2\}}} + \frac{11\sqrt{2}T}{\min_jB_j'}\right)\geq64\cdot2=128$ with $1/2\leq\alpha<1$, the right hand side of \eqref{eq:delta_Delta_equation} is upper bounded by $\Delta_S^*/128\leq1/128$ with $\Delta_S^*\leq1$. Therefore there always exists $\delta_0\in(0,1)$ such that \eqref{eq:delta_Delta_equation} holds. Based on \eqref{eq:delta_Delta_equation}, we also have
\begin{equation*}
\begin{aligned}
\frac{d}{\delta_0} &= \frac{d\cdot64\left(\frac{2C_1+2}{(1-\alpha)^{\max\{\beta,1/2\}}} + \frac{11\sqrt{2}T}{\min_jB_j'} \right) \log(4d/\delta_0)^{\beta'} }{\Delta_S'}\\
&\leq \frac{dK_{\alpha,\beta} \log(4d/\delta_0)^{\beta'} }{\Delta_S'}\\
&= \frac{(4\beta')^{\beta'}dK_{\alpha,\beta}\log((4d/\delta_0)^{1/(4\beta')})^{\beta'}}{\Delta_S'}\\
&\leq \frac{(4\beta')^{\beta'}dK_{\alpha,\beta}(4d/\delta_0)^{1/4}}{\Delta_S'} \\
&= \frac{\sqrt{2}d(4\beta')^{\beta'}K_{\alpha,\beta}}{\Delta_S'} \cdot \left(\frac{d}{\delta_0}\right)^{1/4},
\end{aligned}
\end{equation*}
where the first inequality is based on the definition of $K_{\alpha,\beta}$ and \eqref{eq:budgetj_bound}, the second inequality because $\log(x)<x$ for all $x>0$. Therefore, we get
\begin{equation*}
\frac{d}{\delta_0}\leq\left(\frac{\sqrt{2}d(4\beta')^{\beta'}K_{\alpha,\beta}}{\Delta_S'}\right)^{4/3},
\end{equation*}
namely,
\begin{equation}
\frac{4d}{\delta_0} \leq 4\left(\frac{\sqrt{2}d(4\beta')^{\beta'}K_{\alpha,\beta}}{\Delta_S'}\right)^{4/3} \leq \left(\frac{4\sqrt{2}d(4\beta')^{\beta'}K_{\alpha,\beta}}{\Delta_S'}\right)^{4/3}.
\label{eq:4d/delta0_upperbound}
\end{equation}
Plugging \eqref{eq:4d/delta0_upperbound} into the $\log(4d/\delta_0)$ term in \eqref{eq:delta_Delta_equation}, we have
$$
\frac{d}{\delta_0} \leq \frac{dK_{\alpha,\beta}}{\Delta_S'}\cdot (4/3)^{\beta'}\log\left(\frac{4\sqrt{2}d(4\beta')^{\beta'}K_{\alpha,\beta}}{\Delta_S'} \right)^{\beta'}\\
$$
We write $D_{\alpha,\beta}=4\sqrt{2}d(4\beta')^{\beta'}K_{\alpha,\beta}$, $L_{\alpha,\beta}=(4/3)^{\beta'}K_{\alpha,\beta}$, then we have
$$
\delta_0\geq\frac{\Delta_S'}{L_{\alpha,\beta}\log(D_{\alpha,\beta}/\Delta_S')^{\beta'}}.
$$
By definition of $i_0$ we have that $i_0$ is the largest integer that satisfies $\gamma_{i_0}\geq \delta_0$. Therefore we get
$$
\gamma_{i_0}\geq\frac{\Delta_S'}{L_{\alpha,\beta}\log(D_{\alpha,\beta}/\Delta_S')^{\beta'}},
$$
thus
$$
i_0 \leq \frac{\alpha}{1-\alpha}\log_2\left(\frac{L_{\alpha,\beta}\log(D_{\alpha,\beta}/\Delta_S')^{\beta'}}{\Delta_S'}\right),
$$
or equivalently
\begin{equation}
2^{i_0} \leq \left(\frac{L_{\alpha,\beta}\log(D_{\alpha,\beta}/\Delta_S')^{\beta'}}{\Delta_S'}\right)^{\alpha/(1-\alpha)}.
\label{eq:2^i0_upperbound}
\end{equation}

Using the result of Proposition~\ref{prop:round_Sat_regret_weak}, we have that the satisficing regret by round $i_0$ is upper bounded by
\begin{equation}
\begin{aligned}
&\sum_{i=1}^{i_0}\left( 3\left(\frac{2C_1+2}{(1-\alpha)^{\max\{\beta,1/2\}}} + \frac{11\sqrt{2}T}{\min_jB_j'} \right) +2 \right)\gamma_i^{-1 / (1 - \alpha)}\log(4d/\gamma_i)^{\max\{\beta,1/2\}}\\
\leq&\sum_{i=1}^{i_0}\left( 3\left(\frac{2C_1+2}{(1-\alpha)^{\max\{\beta,1/2\}}} + \frac{11\sqrt{2}T}{\min_jB_j'} \right) +2 \right)\gamma_i^{-1 / (1 - \alpha)}\log(4d/\delta_0)^{\max\{\beta,1/2\}}\\
\leq&\sum_{i=1}^{i_0}\left(\frac{3}{64}K_{\alpha,\beta}+2\right)2^{i/\alpha}\log(4d/\delta_0)^{\beta'} \\
\leq& \left(\frac{3}{32}K_{\alpha,\beta} + 4\right)2^{i_0/\alpha}\log(4d/\delta_0)^{\beta'}\\
\leq& \left(\frac{4}{3}\right)^{\beta'}\left(\frac{3}{32}K_{\alpha,\beta} + 4\right) \left( \frac{L_{\alpha, \beta} \log(D_{\alpha, \beta}  / \Delta_S')^{\beta'}}{\Delta_S'} \right)^{1/(1-\alpha)}\log\left(\frac{D_{\alpha,\beta}}{\Delta_S'}\right)^{\beta'}\\
=& \left(\frac{4}{3}\right)^{\beta'}\left(\frac{3}{32}K_{\alpha,\beta} + 4\right) \left( \frac{L_{\alpha, \beta} }{\Delta_S'} \right)^{1/(1-\alpha)}\log\left(\frac{D_{\alpha,\beta}}{\Delta_S'}\right)^{\beta'/(1-\alpha)+\beta'}
\label{eq:sat_regret-before-i0-weak}
\end{aligned}
\end{equation}
where the first inequality is because by definition of $\delta_0$ and $i_0$, we have $\gamma_i=2^{-i(1-\alpha)/\alpha}\geq \gamma_{i_0}$ for all $i\leq i_0$, the second inequality is by plugging into the definition of $\gamma_i=2^{-i(1-\alpha)/\alpha}$ and definition of $K_{\alpha,\beta}$, the third inequality is based on $\sum_{i=1}^{i_0}2^{i/\alpha} \leq 2\cdot2^{i_0/\alpha}$, and the fourth inequality is by plugging into \eqref{eq:4d/delta0_upperbound} and \eqref{eq:2^i0_upperbound}.

By the results of Proposition~\ref{prop:transit_Sat_prob_weak}, for every $k \geq 1$, the probability that round $i_0 + k$ starts within the time horizon is at most $8^{-k+1}$. Therefore the regret incurred after round $i_0$ is upper bounded by
\begin{equation}
\begin{aligned}
&\sum_{k=1}^{\infty}\left( 3\left(\frac{2C_1+2}{(1-\alpha)^{\beta'}} + \frac{11\sqrt{2}T}{\min_jB_j'} \right) +2 \right) \gamma_{i_0+k}^{-1 / (1 - \alpha)}\log(4d/\gamma_{i_0+k})^{\beta'}\cdot8^{-k+1}\\
\leq& \sum_{k=1}^{\infty} \left(\frac{3}{64}K_{\alpha,\beta}+2\right)2^{(i_0+k)/\alpha}\left(\log(4d/\gamma_{i_0}) + k\right)^{\beta'}\cdot8^{-k+1} \\
\leq& \sum_{k=1}^{\infty} 2^{\beta'} \left(\frac{3}{64}K_{\alpha,\beta}+2\right)2^{(i_0+k)/\alpha}\log(4d/\gamma_{i_0})^{\beta'} \cdot k^{\beta'}\cdot8^{-k+1} \\
\leq& 2^{\beta'}\left(\frac{3}{8}K_{\alpha,\beta}+16\right) 2^{i_0/\alpha}\log(4d/\delta_0)^{\beta'}\cdot M_{\beta'}\\
\leq& (8/3)^{\beta'}\left(\frac{3}{8}K_{\alpha,\beta}+16\right) \left( \frac{L_{\alpha, \beta} \log(D_{\alpha, \beta}  / \Delta_S')^{\beta'}}{\Delta_S'} \right)^{1/(1-\alpha)}\log\left(\frac{D_{\alpha,\beta}}{\Delta_S'}\right)^{\beta'}M_{\beta'}\\
=& (8/3)^{\beta'}\left(\frac{3}{8}K_{\alpha,\beta}+16\right) \left( \frac{L_{\alpha, \beta}}{\Delta_S'} \right)^{1/(1-\alpha)} \log\left(\frac{D_{\alpha,\beta}}{\Delta_S'}\right)^{\beta'/(1-\alpha)+\beta'} M_{\beta'}
\label{eq:sat_regret-after-i0-weak}
\end{aligned}
\end{equation}
where $M_{\beta'}=\sum_{k=1}^\infty k^{\beta'}\cdot2^{-k}$. The first inequality is by plugging into the definition of $K_{\alpha,\beta}$ and $\gamma_i=2^{-i(1-\alpha)/\alpha}$, the second inequality is because $\log(4d/\gamma_{i_0})>1$ and $(1-\alpha)/\alpha\leq1$ for $1/2\leq\alpha<1$, and $\log(4d/\gamma_{i_0})+k\leq2k\log(4d/\gamma_{i_0})$ for $k\geq1$, and for any $a,b\geq 1$, we have $2ab\geq ab+1\geq a+b$, the third inequality is based on $\sum_{k=1}^{\infty}2^{k/\alpha}\cdot8^{-k+1} \leq \sum_{k=1}^{\infty}4^{k}\cdot8^{-k+1}=8\sum_{k=1}^{\infty}2^{-k}$. Therefore, combining \eqref{eq:sat_regret-before-i0-weak} and \eqref{eq:sat_regret-after-i0-weak}, the total expected satisficing regret conditioned on $\mathcal{E}$ is upper bounded by
\begin{equation*}
\left(\frac{15\cdot(8/3)^{\beta'}}{32}K_{\alpha,\beta}+20\cdot(8/3)^{\beta'} \right) \left( \frac{L_{\alpha, \beta} }{\Delta_S'} \right)^{1/(1-\alpha)} \log\left(\frac{D_{\alpha,\beta}}{\Delta_S'}\right)^{\beta'/(1-\alpha)+\beta'} (M_{\beta'}+1).
\end{equation*}

In order to bound $\Delta_S'$, we consider two cases, whether $\Delta_S'\geq\delta(T)$, where $\delta(T)={10\alpha}/{(b-b\alpha)}\cdot\sqrt{{\log(2dT^2)}/{T}}$ is the bound on $(Z^* - Z')$ given in Proposition~\ref{prop:optimal_gap}. 

If $\Delta_S'\geq\delta(T)$, recall $\Delta^*_S=Z^*-S$ and $\Delta_S'=Z'-S$, thus Proposition~\ref{prop:optimal_gap} implies $\Delta^*_S-\Delta'_S\leq\delta(T)$. Then we get
$$
\Delta_S' \geq \frac{\delta(T)}{2} + \frac{\Delta^*_S-\delta(T)}{2} = \frac{\Delta^*_S}{2}.
$$
Combining this with $P(\mathcal{E}^c)\leq1/T$ according to Proposition~\ref{prop:clean_event_bound} and total satisficing regret incurred conditioned on negation of $\mathcal{E}$ is at most $T$, total expected satisficing regret is upper bounded by
\begin{equation}
1 + \left(\frac{15\cdot(8/3)^{\beta'}}{32}K_{\alpha,\beta}+20\cdot(8/3)^{\beta'} \right) \left( \frac{2L_{\alpha,\beta} }{\Delta_S^*} \right)^{1/(1-\alpha)} \log\left(\frac{2D_{\alpha,\beta}}{\Delta_S^*}\right)^{\beta'/(1-\alpha)+\beta'} (M_{\beta'}+1).
\end{equation}

If $\Delta_S'<\delta(T)$, namely $T$ is not large enough, we simply bound the satisficing regret with $T$. As $\Delta^*_S-\Delta'_S\leq\delta(T)$, we have $\Delta_S^*\leq2\delta(T)$, namely 
\begin{equation*}
T\leq\frac{400\alpha^2}{b^2(1-\alpha)^2(\Delta_S^*)^2}\log(2dT^2).
\end{equation*}
Let $q=400\alpha^2b^{-2}(1-\alpha)^{-2}{\Delta_S^*}^{-2}$ and $y=T/(2q)$. So, we have
$$
y \leq \log\left( 2q\sqrt{2d}y \right) = \log\left( 2q\sqrt{2d} \right) + \log(y) \leq  \log\left( 2q\sqrt{2d} \right) + y/2,
$$
where $\log(y)\leq y/2$ is used in the second inequality. This implies $y\leq 2\log\left( 2q\sqrt{2d} \right)$. As $\alpha\geq1/2$ and $\Delta_S^*\leq1$, the total expected satisficing regret is upper bounded by
\begin{equation}
\frac{1600\alpha^2}{b^2(1-\alpha)^2}\cdot\frac{1}{{\Delta_S^*}^2} \log\left(\frac{800\sqrt{2d}\alpha^2}{b^2(1-\alpha)^2(\Delta_S^*)^2}\right) \leq \frac{3200\alpha^2}{b^2(1-\alpha)^2}\cdot\left(\frac{1}{{\Delta_S^*}}\right)^{1/(1-\alpha)} \log\left(\frac{40d^{1/4}\alpha}{b(1-\alpha)\Delta_S^*}\right).
\label{eq:Delta'<delta(T)}
\end{equation}

\subsection{High Probability Event of Oracle Running}
As $\hat{\pi}_i$ is based on the running of $\texttt{ALG}(t_i, \{b_{ji}\})$, we first analyze the high probability events in the trajectory of $\texttt{ALG}(t_i, \{b_{ji}\})$, which will be used in future proof. 

Denoting $C_j(\texttt{ALG}, t_i)$ as the realized cumulative consumption in resource $j$ of executing $\texttt{ALG}(t_i,\{b_{ji}\})$, the high probability event of $\texttt{ALG}(t_i, \{b_{ji}\})$ can be expressed as
\begin{lemma}
    The expected cost of distribution $\hat{\pi}_i$ for resource $j$ in round $i$ can be upper bounded by high probability as 
    \begin{align}
        &\Pr\left(c_j(\hat{\pi}_i)\leq \frac{B_j'}{T} + \sqrt{\frac{\log(2dT^2)}{t_i}} \right)\geq1-\frac{1}{2dT^2},\label{eq:pi_hat_cost_1}\\
        &\Pr\left(\forall k\geq0:\,c_j(\hat{\pi}_i)\leq \frac{B_j'}{T} - 2\sqrt{\frac{4\log(d(T_i+k))}{T_i+k}} \right)\geq1-\frac{1}{32d}.\label{eq:pi_hat_cost_2}
    \end{align}
\label{lemma:high_prob_cost}
\end{lemma}
\begin{proof}

We note that the difference between realized resource consumption $C_j(\texttt{ALG},k)$ and expected consumption of the realized trajectory $\sum_{s=1}^kc_j(\pi_s)$ is a martingale because
\begin{equation*}
\begin{aligned}
&\E\left[C_j(\texttt{ALG},k+1) - \sum_{s=1}^{k+1}c_j(\pi_s) \,\bigg|\, C_j(\texttt{ALG},k) - \sum_{s=1}^kc_j(\pi_s) \right]\\
= &C_j(\texttt{ALG},k) - \sum_{s=1}^kc_j(\pi_s) + \E\left[(C_j(\texttt{ALG},k+1)-C_j(\texttt{ALG},k)) - c_j(\pi_{k+1})\right]\\
= &C_j(\texttt{ALG},k) - \sum_{s=1}^kc_j(\pi_s)
\end{aligned}
\end{equation*}
where the final equality is because the realized consumption in resource $j$ at step $k+1$ is $(C_j(\texttt{ALG},k+1)-C_j(\texttt{ALG},k))$ with expectation $c_j(\pi_{k+1})$. Conditioned on the preceding history, each martingale increment belongs to an interval of length $1$, because the cost belongs to $[0,1]$. The Hoeffding--Azuma inequality therefore gives
\begin{equation}
\Pr\left(C_j(\texttt{ALG},k) - \sum_{s=1}^kc_j(\pi_s)<-\epsilon\right)\leq\exp(-\frac{2\epsilon^2}{k}).
\label{eq:azuma-cumulative_cost}
\end{equation} 

In order to analyze the high probability events for $c_j(\hat{\pi}_i)$, we bound the high probability events for realized resource consumption, with confidence interval $\sqrt{t_i\log(2dT^2)}$ and $\sqrt{2t_i\log(32d)}$ respectively. 

For the high probability event with radius $\sqrt{t_i\log(2dT^2)}$, the difference between realized resource consumption and the expected consumption of the trajectory can be upper bounded by Azuma's Inequality \eqref{eq:azuma-cumulative_cost} as
\begin{equation}
    \Pr\left(\frac{C_j(\texttt{ALG},t_i)}{t_i} < \frac{1}{t_i}\sum_{s=1}^{t_i}c_j(\pi_s)-\sqrt{\frac{\log(2dT^2)}{t_i}}\right)\leq \frac{1}{2dT^2}.
\label{eq:high_prob-consumption-1}
\end{equation}
Recall $\hat{\pi}_i$ is derived from the trajectory of running $\texttt{ALG}(t_i,\{b_{ji}\})$, then $C_j(\texttt{ALG},t_i) - \sum_{s=1}^{t_i}c_j(\pi_s)>-\sqrt{t_i\log(2dT^2)}$ with $0\leq w_i\leq1$ implies that with any nonnegative $w_i$, we have that
\begin{equation*}
\begin{aligned}
&c_j(\hat{\pi}_i)=\frac{1-w_i}{t_i}\sum_{s=1}^{t_i}c_j(\pi_s)\\
\leq& \frac{C_j(\texttt{ALG},t_i)}{t_i} + \sqrt{\frac{\log(2dT^2)}{t_i}}\\
\leq& \frac{B_j'}{T} + \sqrt{\frac{\log(2dT^2)}{t_i}},
\end{aligned}
\end{equation*}
happens with probability higher than $1-1/(2dT^2)$, where the first inequality is due to $w_i\geq0$ and the complement of \eqref{eq:high_prob-consumption-1}, and the second inequality is because realized cost $C_j(\texttt{ALG},t_i)$ is constrained by allocated resource budget $b_{ji}=B'_jt_i/T$, and then high probability event \eqref{eq:pi_hat_cost_1} is proved. 

For the high probability event with radius $\sqrt{2t_i\log(32d)}$, the difference between realized resource consumption and the expected consumption of the trajectory can be upper bounded by Azuma's Inequality \eqref{eq:azuma-cumulative_cost} as
\begin{equation}
    \Pr\left(\frac{C_j(\texttt{ALG},t_i)}{t_i} < 
    \frac{1}{t_i}\sum_{s=1}^{t_i}c_j(\pi_s)-\sqrt{\frac{2\log(32d)}{t_i}}\right)\leq \frac{1}{32d}.
\label{eq:high_prob-consumption-2}
\end{equation}
Recall $\hat{\pi}_i$ is derived from the trajectory of running $\texttt{ALG}(t_i,\{b_{ji}\})$, then $C_j(\texttt{ALG},t_i) - \sum_{s=1}^{t_i}c_j(\pi_s) > -\sqrt{2t_i\log(32d)}$ implies that
\begin{equation*}
\begin{aligned}
    & c_j(\hat{\pi}_i)=\frac{1-w_i}{t_i}\sum_{s=1}^{t_i}c_j(\pi_s)\\
    \leq & (1-w_i)\cdot\left(\frac{C_j(\texttt{ALG},t_i)}{t_i}+\sqrt{\frac{2\log(32d)}{t_i}}\right)\\
    \leq & \frac{B_j'}{T}-w_i\frac{B_j'}{T} + \sqrt{\frac{2\log(32d)}{t_i}} \\
    \leq & \frac{B_j'}{T}-2\sqrt{\frac{4\log(d(T_i+k))}{T_i+k}},\quad\forall k\geq0,
\end{aligned}
\end{equation*}
happens with probability higher than $1-1/(32d)$, where the first inequality is due to the complement of \eqref{eq:high_prob-consumption-2}, and the second inequality is because realized cost $C_j(\texttt{ALG},t_i)$ is constrained by allocated resource budget $b_{ji}=B'_jt_i/T$, \ie, $C_j(\texttt{ALG},t_i)/t_i\leq B_j'/T$, and $w_i>0$. The third inequality is based on $2\leq T_i\leq t_i$ and the definition of $w_i=10T/\min_jB_j'\,\cdot\sqrt{\log(8dT_i)/T_i}$ when $w_i<1$ where
\begin{equation*}
\begin{aligned}
&2\sqrt{\frac{4\log(d(T_i+k))}{T_i+k}}+\sqrt{\frac{2\log(32d)}{t_i}}\\
\leq&8\sqrt{\frac{\log(dT_i)}{T_i}}+2\sqrt{\frac{\log(8d)}{t_i}}\\
\leq&8\sqrt{\frac{\log(8dT_i)}{T_i}} + 2\sqrt{\frac{\log(8dT_i)}{T_i}}\\
=& \frac{\min_jB_j'}{T}\cdot w_i, \quad\forall k\geq0.
\end{aligned}
\end{equation*}
The first inequality above is because $\log(x)/x$ is decreasing for $x\geq3$ and $2\log(2)/2\geq\log(3)/3$, and the second inequality is because $T_i\geq1$ and $t_i=\lceil1/\gamma_i^{1/(1-\alpha)} \rceil\geq\lceil1/\gamma_i^2\rceil=T_i$ for $\gamma_i<1$ and $1/2\leq\alpha<1$. Then the high probability event \eqref{eq:pi_hat_cost_2} is proved. 
\end{proof}

For notation simplicity in future proof, we denote
\begin{equation*}
\begin{aligned}
    \mathcal{C}_{ji}^1&:=\left\{c_j(\hat{\pi}_i)\leq \frac{B_j'}{T} + \sqrt{\frac{\log(2dT^2)}{t_i}}\right\}\\
    \mathcal{C}_{ji}^2&:=\left\{\forall k\geq0:\,c_j(\hat{\pi}_i)\leq \frac{B_j'}{T} - 2\sqrt{\frac{4\log(d(T_i+k))}{T_i+k}}\right\}.
\end{aligned}
\end{equation*}
Therefore Lemma~\ref{lemma:high_prob_cost} implies $\Pr\left(\mathcal{C}_{ji}^{1c}\right) \leq 1/(2dT^2)$ and $\Pr\left(\mathcal{C}_{ji}^{2c}\right)\leq 1/(32d)$.

The expected regret of the candidate policy is bounded as follows.
\begin{lemma}
    Under Condition~\ref{condition:std_regret_bound}, the expected regret of a single pull on distribution $\hat{\pi}_i$ in round $i$ can be upper bounded by
    \begin{equation*}
        \E\left[\max\{Z'-r(\hat{\pi}_i), 0\}\right]\leq \left( \frac{2C_1+2}{(1-\alpha)^{\max\{\beta,1/2\}}} + \frac{11\sqrt{2}T}{\min_jB_j'} \right) \gamma_i\log\left(\frac{4d}{\gamma_i} \right)^{\max\{\beta,1/2\}}.
    \end{equation*}
\label{lemma:high_prob_reward_weak}
\end{lemma}
\begin{proof}
Since $\hat{\pi}_i$ is calculated based on $\{\tilde{\pi}_s\}_{s\in[t_i]}$ in round $i$, we have
\begin{equation*}
\begin{aligned}
&\E\left[\max\{Z'-r(\hat{\pi}_i), 0\}\right]\\
= &\frac{1}{t_i}\E\left[\max\left\{t_iZ'-(\sum_{s=1}^{t_i}r(\tilde{\pi}_s))\cdot(1-w_i), 0\right\}\right]\\
\leq &\frac{1}{t_i}\E\left[\max\left\{t_iZ'-\sum_{s=1}^{t_i}r(\tilde{\pi}_s), 0\right\}\right] + \frac{w_i}{t_i}\E\left[\sum_{s=1}^{t_i}r(\tilde{\pi}_s)\right] \\
\leq &\frac{1}{t_i}\E\left[\max\left\{t_iZ'-R, 0 \right\} \right] + w_i
\end{aligned}
\end{equation*}
where $R:=\sum_{s=1}^{t_i}r(\tilde{\pi}_s)$. The first equality is based on the definition of $\hat{\pi}_i$ as a weighted sum of trajectory average and null arm that has zero reward, and the second inequality is based on bounded reward $r\in[0,1]$. By Lemma~\ref{lemma:positive_standard_regret}, the first term is bounded using Condition~\ref{condition:std_regret_bound}, and therefore
\begin{equation}
\E\left[\max\{Z'-r(\hat{\pi}_i),0\}\right]\leq C_1t_i^{\alpha-1}\log^\beta(t_i)+\frac{T}{\min_jB_j'}t_i^{-1/2}+w_i.
\label{eq:single_pull-weak}
\end{equation}

For $t_i^{\alpha-1}\log (t_i)^{\beta}$, recall that $t_i=\lceil \gamma_i^{-1/(1-\alpha)}\rceil\leq 2\gamma_i^{-1/(1-\alpha)}$, then $\log(t_i)\leq\log(2\gamma_i^{-1/(1-\alpha)})\leq \log(2/\gamma_i)/(1-\alpha)$ and $t_i^\alpha=\lceil2^{i/\alpha}\rceil^\alpha\leq2^{\alpha(1+i/\alpha)}\leq2\cdot2^i=2\gamma_i^{-\alpha/(1-\alpha)}$. Similarly, $t_i^{\alpha-1}\leq(\gamma_i^{-1/(1-\alpha)})^{\alpha-1}=\gamma_i\leq2\gamma_i$. So, it can be bounded by
\begin{equation}
C_1t_i^{\alpha-1}\log (t_i)^{\beta} \leq \frac{2C_1}{(1-\alpha)^\beta}\gamma_i\log\left(\frac{2}{\gamma_i}\right)^{\beta} \leq \frac{2C_1}{(1-\alpha)^\beta}\gamma_i\log\left(\frac{4d}{\gamma_i}\right)^{\beta}.
\label{eq:single_pull-weak-sublinear}
\end{equation}
For $w_i$, as we assume $\alpha\geq1/2$, then $T_i \leq t_i$. Recall $\gamma_i^{-\alpha/(1-\alpha)}=2^i$ and $T_i=\lceil1/\gamma_i^2\rceil$, we have $T_i\geq\gamma_i^{-2}$ and $T_i\leq2\gamma_i^{-2}$. $w_i$ can be upper bounded with
\begin{equation}
\begin{aligned}
\sqrt{\frac{\log(8dT_i)}{T_i}} \leq \gamma_i\sqrt{\log\left(\frac{16d}{\gamma_i^2}\right)} \leq \sqrt{2}\gamma_i\log\left(\frac{4d}{\gamma_i}\right)^{1/2}.
\label{eq:single_pull-weak-wi}
\end{aligned}
\end{equation}
For the additional term, since $\alpha\geq1/2$ and $\gamma_i\leq1$,
\begin{equation}
\frac{T}{\min_jB_j'}t_i^{-1/2}\leq\frac{T}{\min_jB_j'}\gamma_i^{1/(2(1-\alpha))}\leq\frac{\sqrt2T}{\min_jB_j'}\gamma_i\log\left(\frac{4d}{\gamma_i}\right)^{1/2}.
\label{eq:single_pull-weak-1/ti}
\end{equation}
Combining \eqref{eq:single_pull-weak-sublinear}, \eqref{eq:single_pull-weak-wi}, and \eqref{eq:single_pull-weak-1/ti} with \eqref{eq:single_pull-weak}, gives
\begin{equation*}
\E\left[\max\{Z'-r(\hat{\pi}_i), 0\}\right]\leq \left( \frac{2C_1+2}{(1-\alpha)^{\max\{\beta,1/2\}}} + \frac{ 11\sqrt{2}T}{\min_jB_j'} \right)\gamma_i\log\left(\frac{4d}{\gamma_i} \right)^{\max\{\beta,1/2\}}.
\end{equation*}

\end{proof}

\subsection{Proof of Lemma~\ref{lemma:positive_standard_regret}}\label{sec:positive_standard_regret}

\begin{proof}
Recall that $b_{ji}=B_j't_i/T$. Since the resource budgets in round $i$ are obtained by scaling $\bm B'$ by $t_i/T$, the LP relaxation corresponding to the $t_i$-step oracle run can be written as
\begin{equation*}
\begin{aligned}
t_i Z' = \max_{\pi\in\Delta(\mathcal X)} \quad& t_i r(\pi)\\
\text{s.t.}\quad& t_i c_j(\pi)\leq b_{ji}, \qquad \forall j\in[d].
\end{aligned}
\end{equation*}
Its dual is
\begin{equation*}
\begin{aligned}
\min_{\lambda_1,\ldots,\lambda_d,\mu} \quad& t_i \mu+\sum_{j=1}^d\lambda_j b_{ji}\\
\text{s.t.}\quad& \mu+\sum_{j=1}^d\lambda_j c_j(X)\geq r(X), \qquad \forall X\in\mathcal X,\\
&\lambda_j\geq0,\qquad \forall j\in[d].
\end{aligned}
\end{equation*}
Let $(\bm\lambda^*,\mu^*)$ be an optimal dual solution. The positive budgets and the null policy provide strict feasibility. By strong duality,
\begin{equation}
t_i Z'=t_i\mu^*+\sum_{j=1}^d\lambda_j^*b_{ji}.
\label{eq:dual_strong_duality}
\end{equation}
Then, we can bound
\begin{equation*}
t_i\overset{\rm(i)}{\geq}t_iZ'=t_i\mu^*+\sum_{j=1}^d\lambda_j^*b_{ji}\overset{\rm(ii)}{\geq}\sum_{j=1}^d\lambda_j^*b_{ji}\overset{\rm(iii)}{\geq}\frac{t_i\min_jB_j'}T\sum_{j=1}^d\lambda_j^*,
\end{equation*}
where step (i) is because $Z'\leq1$, step (ii) is because the null arm has zero reward and zero resource consumption and then its dual constraint implies $\mu^*\geq0$, and step (iii) is based on $b_{ji}=B_j't_i/T$. This further implies
\begin{equation}
\|\bm\lambda^*\|_1=\sum_{j=1}^d\lambda_j^*\leq\frac{T}{\min_jB_j'}.
\label{eq:dual_lambda_bound}
\end{equation}
Since the dual constraint holds for every arm, it also holds for every distribution over arms. Therefore, for every policy $\tilde\pi_s$ selected by the oracle,
\begin{equation*}
r(\tilde\pi_s)\leq\mu^*+\sum_{j=1}^d\lambda_j^*c_j(\tilde\pi_s).
\end{equation*}
Summing this inequality over $s\in[t_i]$, we obtain
\begin{equation}
\begin{aligned}
\sum_{s=1}^{t_i}r(\tilde\pi_s)-t_iZ'&\leq t_i\mu^*+\sum_{j=1}^d\lambda_j^*\sum_{s=1}^{t_i}c_j(\tilde\pi_s)-\left(t_i\mu^*+\sum_{j=1}^d\lambda_j^*b_{ji}\right)\\
&=\sum_{j=1}^d\lambda_j^*\left(\sum_{s=1}^{t_i}c_j(\tilde\pi_s)-b_{ji}\right)\\
&\leq\sum_{s=1}^{t_i}\sum_{j=1}^d\lambda_j^*\left(c_j(\tilde\pi_s)-C_s^j\right),
\end{aligned}
\label{eq:dual_overshoot_martingale}
\end{equation}
where the first inequality uses strong duality \eqref{eq:dual_strong_duality}, and the final inequality is because the realized consumption does not exceed the oracle allocation $\sum_{s=1}^{t_i}C_s^j\leq b_{ji}$ for every $j\in[d]$. 

Define $D_s:=\sum_{j=1}^d\lambda_j^*(c_j(\tilde\pi_s)-C_s^j)$. Conditioned on the history immediately before drawing the arm at step $s$, including the oracle's policy choice, $\tilde\pi_s$ is fixed and $\E[C_s^j\mid\text{history}]=c_j(\tilde\pi_s)$, so $\E[D_s\mid\text{history}]=0$. Hence $\{D_s\}_{s=1}^{t_i}$ is a martingale difference sequence. Furthermore, since $c_j(\tilde\pi_s)$ and $C_s^j$ belong to $[0,1]$, \eqref{eq:dual_lambda_bound} gives
\begin{equation*}
|D_s|\leq\sum_{j=1}^d\lambda_j^*|c_j(\tilde\pi_s)-C_s^j|\leq\sum_{j=1}^d\lambda_j^*=\|\bm\lambda^*\|_1\leq\frac{T}{\min_jB_j'}.
\end{equation*}
Therefore, using the martingale difference property to eliminate the cross terms,
\begin{equation}
\begin{aligned}
\E\left[\left(\sum_{s=1}^{t_i}D_s\right)^2\right]&=\sum_{s=1}^{t_i}\E[D_s^2]+2\sum_{1\leq s<u\leq t_i}\E[D_sD_u]\\
&=\sum_{s=1}^{t_i}\E[D_s^2]\leq t_i\left(\frac{T}{\min_jB_j'}\right)^2.
\end{aligned}
\label{eq:martingale_second_moment}
\end{equation}
Combining \eqref{eq:dual_overshoot_martingale} and \eqref{eq:martingale_second_moment}, followed by Cauchy--Schwarz, gives
\begin{equation}
\begin{aligned}
\E\left[\max\left\{\sum_{s=1}^{t_i}r(\tilde\pi_s)-t_iZ',0\right\}\right]&\leq\E\left[\max\left\{\sum_{s=1}^{t_i}D_s,0\right\}\right]\leq\E\left[\left|\sum_{s=1}^{t_i}D_s\right|\right]\\
&\leq\sqrt{\E\left[\left(\sum_{s=1}^{t_i}D_s\right)^2\right]}\leq\frac{T}{\min_jB_j'}\sqrt{t_i}.
\end{aligned}
\label{eq:positive_overshoot_bound}
\end{equation}
Finally, using $\max\{x,0\}=x+\max\{-x,0\}$ and Condition~\ref{condition:std_regret_bound},
\begin{equation*}
\begin{aligned}
\E\left[\max\left\{t_iZ'-\sum_{s=1}^{t_i}r(\tilde\pi_s),0\right\}\right]&=\E\left[t_iZ'-\sum_{s=1}^{t_i}r(\tilde\pi_s)\right]+\E\left[\max\left\{\sum_{s=1}^{t_i}r(\tilde\pi_s)-t_iZ',0\right\}\right]\\
&\leq C_1t_i^\alpha\log^\beta(t_i)+\frac{T}{\min_jB_j'}\sqrt{t_i},
\end{aligned}
\end{equation*}
which proves the lemma.
\end{proof}

\subsection{Proof of Proposition~\ref{prop:clean_event}}\label{sec:prop:regret_bound_Sat_clean}



We first bound the number of rounds $i_{\max}$. Recall that $t_i=\lceil\gamma_i^{-1/(1-\alpha)}\rceil$ and $T_i=\lceil\gamma_i^{-2}\rceil$, since the length of the time horizon is $T$ time steps, we have $t_i\leq T$. By definition of $t_i=\lceil2^{i/\alpha}\rceil$, the number of rounds that end within the time horizon is bounded by $2^{i/\alpha}\leq T$, \ie, $i\leq\lfloor\alpha\log_2(T) \rfloor$ because $i$ is an integer. Then, the maximum number of rounds is bounded by $i_{\max}=\lceil\alpha\log_2(T)\rceil \leq 2\alpha\log_2(T)$ for $T\geq4$.

Recall that $T_i\leq\gamma_i^{-2}+1\leq 2^{2i(1-\alpha)/\alpha}+1$, then $\sqrt{T_i}\leq2^{i(1-\alpha)/\alpha}+1$ and $T_{i_{\max}}\leq T$. We can bound $\sum_i\sqrt{T_i}$ as
\begin{equation*}
\begin{aligned}
\sum_{i=1}^{i_{\max}}\sqrt{T_i} \leq& \sum_{i=1}^{i_{\max}}2^{i(1-\alpha)/\alpha} + i_{\max}   \\
\leq&\dfrac{2^{(1-\alpha)/\alpha}}{2^{(1-\alpha)/\alpha}-1}\cdot2^{i_{\max}(1-\alpha)/\alpha} + i_{\max}\\
\leq&\dfrac{2}{2^{(1-\alpha)/\alpha}-1}\cdot\sqrt{T_{i_{\max}}} + 2\alpha\log_2(T)\\
\overset{\text{(i)}}{\leq}&\frac{2\alpha}{(1-\alpha)\ln(2)}\sqrt{T} + 4\alpha\sqrt{T}\\
\leq&\frac{5\alpha}{1-\alpha}\sqrt{T}
\end{aligned}
\end{equation*}
where step \text{(i)} is based on $2^x-1\geq x\ln(2)$ for all $x>0$ and $\log_2(T)\leq2\sqrt{T}$ for all $T\geq0$.

Therefore the permissive accumulation of the excess portion of $\hat{c}_{ji}^{\text{tot}}(\tau_i)$ into 
$B_j'\tau_i/T$, conditioned on $\mathcal{E}$, is upper bounded by 
\begin{equation*}
    \sum_i 2\sqrt{T_i\log(2dT^2)}\leq\frac{10\alpha}{1-\alpha}\sqrt{T\log(2dT^2)}.
\end{equation*}

Recall that $b_{ji}={B_j't_i}/{T}$, $\sum_i(t_i+\tau_i)=T$. conditioned on $\mathcal{E}$, realized resource $j$ consumption is upper bounded by
\begin{equation*}
\begin{aligned}
    &\sum_i(b_{ji}+\hat{c}_{ji}^{\text{tot}}(\tau_i))\\
    \leq& \sum_i\left(\frac{B_j't_i}{T}+\frac{B_j'\tau_i}{T} + \bm{1}(\tau_i>0) 2\sqrt{T_i\log(2dT^2)}\right)\\
    \leq&B_j'+\frac{10\alpha}{1-\alpha}\sqrt{T\log(2dT^2)}= B_j.
\end{aligned}
\end{equation*}

\subsection{Proof of Proposition~\ref{prop:clean_event_bound}}\label{sec:prop:regret_bound_Sat_cleanbound}
\begin{lemma}
Let $\tau_i$ be the length of forced sampling and LCB-UCB test of round $i$, then for any $j\in[d]$,
\begin{equation*}
\hat{c}_{ji}^{\text{tot}}(\tau_i)-\dfrac{\tau_iB_j'}{T}\leq \max\left\{ \max_{n\leq\min\{T_i,\tau_i\}} \left\{\hat{c}_{ji}^{\text{tot}}(n)-\dfrac{nB_j'}{T}\right\}, 1\right\}.
\end{equation*}
\label{lemma:cost_bound_forced_sampling}
\end{lemma}

\begin{proof}

If $\tau_i=T_i$, then $$\hat{c}_{ji}^{\text{tot}}(\tau_i)-\dfrac{\tau_iB_j'}{T}=\hat{c}_{ji}^{\text{tot}}(\tau_i)-\dfrac{T_iB_j'}{T}\leq \max\left\{\hat{c}_{ji}^{\text{tot}}(T_i)-\dfrac{T_iB_j'}{T},1\right\}.$$
Therefore it suffices to consider the case when $\tau_i>T_i$. Assume $\tau_i>T_i$, then the termination criterion of the LCB-UCB test is not satisfied at step $\tau_i-1$, then we get
\begin{equation*}
UCB_j(\hat{\pi}_i)=\dfrac{\hat{c}_{ji}^{\text{tot}}(\tau_i-1)}{\tau_i-1}+\sqrt{\frac{4\log(d(\tau_i-1))}{\tau_i-1}}\leq \dfrac{B_j'}{T},
\end{equation*}
which implies that
\begin{equation*}
\hat{c}_{ji}^{\text{tot}}(\tau_i-1)\leq \dfrac{B_j'(\tau_i-1)}{T}.
\end{equation*}
Since we assume that the realized resource consumption is at most $1$, we get
\begin{equation*}
\hat{c}_{ji}^{\text{tot}}(\tau_i)\leq \hat{c}_{ji}^{\text{tot}}(\tau_i-1)+1\leq \dfrac{B_j'\tau_i}{T}+1,
\end{equation*}
therefore in this case
\begin{equation*}
\hat{c}_{ji}^{\text{tot}}(\tau_i)-\dfrac{\tau_iB_j'}{T}\leq 1\leq \max\left\{ \max_{n\leq\min\{T_i,\tau_i\}} \left\{\hat{c}_{ji}^{\text{tot}}(n)-\dfrac{nB_j'}{T}\right\}, 1\right\}.\qedhere
\end{equation*}

\end{proof}

We prove this proposition by proving $\Pr(\mathcal{E}^c)\leq1/T$. Recall that
$$
\mathcal{C}_{ji}^1=\left\{c_j(\hat{\pi}_i)\leq \frac{B_j'}{T} + \sqrt{\frac{\log(2dT^2)}{t_i}}\right\}.
$$
As $\mathcal{E}^c$ can be expressed as $\{\exists j,\,\exists i,\, \hat{c}_{ji}^{\text{tot}}(\tau_i)>{\tau_iB'_j}/{T} + 2\sqrt{T_i\log(2dT^2)} \}$, its probability can be analyzed conditioned on whether $\mathcal{C}_{ji}^1$ holds for each resource $j$ round $i$.

When $\mathcal{C}_{ji}^1$ does not hold, the probability that $\hat{c}_{ji}^{\text{tot}}(\tau_i)$ exceeds the allocated resource can be upper bounded by $1$. And recall Equation~\eqref{eq:pi_hat_cost_1} in Lemma~\ref{lemma:high_prob_cost} implies that $\mathcal{C}_{ji}^{1c}$ happens with probability 
\begin{equation*}
    \Pr(\mathcal{C}_{ji}^{1c})\leq\frac{1}{2dT^2}.
\end{equation*}

When $\mathcal{C}_{ji}^1$ holds, namely $c_j(\hat{\pi}_i)\leq B_j'/T+\sqrt{\log(2dT^2)/t_i}$, then during forced sampling the excess cost can be bounded by
\begin{equation*}
\begin{aligned}
&\Pr\left(\max_{1\leq n\leq \min\{T_i,\tau_i\}}\left\{\hat{c}_{ji}^{\mathrm{tot}}(n)-\frac{nB_j'}{T}\right\}\geq2\sqrt{T_i\log(2dT^2)}\,\bigg|\,\mathcal{C}_{ji}^1\right)\\
\leq&\sum_{n=1}^{T_i}\Pr\left(\frac{\hat{c}_{ji}^{\mathrm{tot}}(n)}{n}\geq c_j(\hat{\pi}_i)-\sqrt{\frac{\log(2dT^2)}{t_i}}+\frac{2\sqrt{T_i\log(2dT^2)}}{n}\,\bigg|\,\mathcal{C}_{ji}^1\right)\\
\leq&\sum_{n=1}^{T_i}\Pr\left(\frac{\hat{c}_{ji}^{\mathrm{tot}}(n)}{n}\geq c_j(\hat{\pi}_i)+\sqrt{\frac{\log(2dT^2)}{n}}\,\bigg|\,\mathcal{C}_{ji}^1\right)\\
\leq&T_i\exp(-2\log(2dT^2))\leq\frac{1}{2dT^2}.
\end{aligned}
\end{equation*}
where the first inequality is based on the definition of $\mathcal{C}_{ji}^1$ and a union bound over $n=1,\ldots,T_i$. The second inequality is because $n\leq T_i\leq t_i$, which implies
\begin{equation*}
\frac{2\sqrt{T_i\log(2dT^2)}}{n}-\sqrt{\frac{\log(2dT^2)}{t_i}}
\geq\sqrt{\frac{\log(2dT^2)}{n}}.
\end{equation*}
For the third inequality, conditioned on the completed oracle trajectory, the forced-sampling observations are independent and bounded in $[0,1]$ with mean $c_j(\hat{\pi}_i)$, so Hoeffding's inequality applies. Finally, $T_i\leq T$ implies $T_i\exp(-2\log(2dT^2))\leq1/(2dT^2)$. Combining this with Lemma~\ref{lemma:cost_bound_forced_sampling}, because $2\sqrt{T_i\log(2dT^2)}>1$, we then know
\begin{equation}
\Pr\left(\hat{c}_{ji}^{\mathrm{tot}}(\tau_i)>\frac{\tau_iB_j'}{T}+2\sqrt{T_i\log(2dT^2)}\,\bigg|\,\mathcal{C}_{ji}^1\right)\leq\frac{1}{2dT^2}.
\label{eq:exceed_prob-C}
\end{equation}

Therefore, the probability of $\mathcal{E}$ negation can be upper bounded by
\begin{equation*}
\begin{aligned}
&\Pr\left(\exists j,\,\exists i,\, \hat{c}_{ji}^{\text{tot}}(\tau_i)>{\tau_iB_j'}/{T} + 2\sqrt{T_i\log(2dT^2)}\right)\\
\leq&\sum_{j=1}^d\sum_i^{i_{\max}}(\frac{1}{2dT^2} + \frac{1}{2dT^2})\\
\leq& \frac{dT}{dT^2} = \frac{1}{T},
\end{aligned}
\end{equation*}
where the first inequality is based on probability condition on $\mathcal{C}_{ji}^{1c}$ \eqref{eq:pi_hat_cost_1} and $\mathcal{C}_{ji}^1$ \eqref{eq:exceed_prob-C}, and the second inequality is based on $T>i_{\max}$. Then this proves the proposition that $\Pr(\mathcal{E})\geq1-1/T$.

\subsection{Proof of Proposition~\ref{prop:optimal_gap}}\label{sec:prop:optimal_gap}

Given one of the optimal distribution under original resource constraint $B_j$ as $\pi^*$, we define distribution 
$$
\pi'= (1-\delta(T))\pi^*+\delta(T)\cdot\bm{1}_0,
$$
where $\delta(T):={10\alpha}/{(b-b\alpha)}\cdot\sqrt{{\log(2dT^2)}/{T}}$. Recall $B_j'=B_j - 10\alpha/(1-\alpha)\cdot\sqrt{T\log(2dT^2)}$, for each resource $j$ we can bound 
\begin{equation*}
\begin{aligned}
    c_j(\pi') &= (1-\delta(T))c_j(\pi^*)\\
    &\leq \left(1 - \frac{10\alpha}{b(1-\alpha)}\sqrt{\frac{\log(2dT^2)}{T}} \right)\cdot\frac{B_j}{T}\\
    &\leq\frac{B_j}{T} - \frac{10\alpha}{1-\alpha}\sqrt{\frac{\log(2dT^2)}{T}}=\frac{B_j'}{T},
\end{aligned}
\end{equation*}
where the first inequality is based on the definition of $\pi^*$ introduced in LP relaxation \eqref{eq:problem_formulation-lp} on the optimal dynamic policy. Therefore $\pi'$ is a feasible distribution under resource constraints $B_j'$. Namely, the optimal distribution reward under constraints $B_j'$ can be upper bounded by $Z'\geq r(\pi')$.

The difference between $Z^*$ and $Z'$ can be upper bounded by 
\begin{equation*}
    Z^*-Z'\leq Z^* - r(\pi') = Z^* - (1-\delta(T))Z^* \leq \delta(T),
\end{equation*}
where the $Z^*=\E_{X\sim\pi^*}[r(X)]\leq1$ is used in the final inequality. 

\subsection{Proof of Proposition~\ref{prop:round_Sat_regret_weak}}\label{sec:prop:round_Sat_regret_weak}
Recall $\alpha\geq1/2$ and $\gamma_i\leq1$. Lemma~\ref{lemma:high_prob_reward_weak} implies that the satisficing regret of a single pull on $\hat{\pi}_i$ can be upper bounded by
\begin{equation}
\begin{aligned}
\E[\max\{S-r(\hat{\pi}_i) ,0\}] &\leq \E\left[\max\{Z'-r(\hat{\pi}_i), 0\}\right] \\
&\leq \left( \frac{2C_1+2}{(1-\alpha)^{\max\{\beta,1/2\}}} + \frac{11\sqrt{2}T}{\min_jB_j'} \right) \gamma_i\log\left(\frac{4d}{\gamma_i} \right)^{\max\{\beta,1/2\}}.
\end{aligned}
\label{eq:sat_regret_step-single-weak}
\end{equation}
The regret in round $i$ can be divided into three parts. Running $\texttt{ALG}(t_i,\{b_{ji}\})$, forced-sampling, and LCB-UCB test.

Recall that $t_i=\lceil\gamma_i^{-1/(1-\alpha)}\rceil\leq2\gamma_i^{-1/(1-\alpha)}$ as $\gamma_i<1$, the satisficing regret incurred when running $\texttt{ALG}(t_i,\{b_{ji}\})$ can simply be upper bounded by linear regret
\begin{equation}
\E\left[\sum_{s=1}^{t_i}\max\{S-r(\tilde{\pi}_s) ,0\} \right]\leq t_i \leq 2\gamma_i^{-1/(1-\alpha)}\leq2\gamma_i^{-1/(1-\alpha)}\log(4d/\gamma_i)^{\max\{\beta,1/2 \}}.
\label{eq:sat_regret_step-1-weak}
\end{equation}

The satisficing regret incurred when pulling $\hat{\pi}_i$ in forced sampling for $T_i=\lceil\gamma_i^{-2}\rceil$ time steps is upper bounded by
\begin{equation}
\begin{aligned}
&\lceil\gamma_i^{-2}\rceil\cdot \left( \frac{2C_1+2}{(1-\alpha)^{\max\{\beta,1/2\}}} + \frac{11\sqrt{2}T}{\min_jB_j'} \right) \gamma_i\log(4d/\gamma_i)^{\max\{\beta,1/2\}}\\
\leq&\frac{2}{\gamma_i^2}\cdot \left( \frac{2C_1+2}{(1-\alpha)^{\max\{\beta,1/2\}}} + \frac{11\sqrt{2}T}{\min_jB_j'} \right) \gamma_i\log(4d/\gamma_i)^{\max\{\beta,1/2\}}\\
\leq&2 \left( \frac{2C_1+2}{(1-\alpha)^{\max\{\beta,1/2\}}} + \frac{11\sqrt{2}T}{\min_jB_j'} \right) \gamma_i^{-1/ (1 - \alpha)}\log(4d/\gamma_i)^{\max\{\beta,1/2\}},
\label{eq:sat_regret_step-2-weak}
\end{aligned}
\end{equation}
where the first inequality is because $\lceil\gamma_i^{-2}\rceil\leq2/\gamma_i^2$ as $\gamma_i<1$, and the second inequality is because $\alpha\geq1/2$ and $\gamma_i\leq1$.

If $r(\hat{\pi}_i)<S$ (otherwise the satisficing regret is $0$), for every $k$, round $i$ is not terminated after pulling $\hat{\pi}_i$ for $k$ time steps in addition to first $T_i$ time steps occurs with probability at most
\begin{equation*}
\begin{aligned}
&\Pr\left(\dfrac{\hat{r}^{\text{tot}}_i}{T_i+k}-\sqrt{\dfrac{4\log(T_i+k)}{T_i+k}}\geq S\right)\\
\leq &\Pr\left(\dfrac{\hat{r}^{\text{tot}}_i}{T_i+k}-\sqrt{\dfrac{4\log(T_i+k)}{T_i+k}}>r(\hat{\pi}_i)\right)\\
\leq&\dfrac{1}{(T_i+k)^8}\leq \dfrac{1}{(k+1)^8}.
\end{aligned}
\end{equation*}
Therefore expected number of arm pulls of $\hat{\pi}_i$ after first pulling it $T_i$ time steps is upper bounded by
\begin{equation}
\begin{aligned}
&\E\left[\text{Number of time steps }\hat{\pi}_i\text{ is pulled in round }i\text{ after the first }T_i \text{ time steps} \,|\, \hat{\pi}_i \right]\\
=&\sum_{s=1}^{\infty}\Pr(k\geq s\,|\, \hat{\pi}_i)
\leq \sum_{s=1}^\infty \dfrac{1}{(T_i+s-1)^8}
\leq \sum_{s=1}^\infty \dfrac{1}{(s+1)^8} \leq 1.
\label{eq:sat_regret_step-3-weak}
\end{aligned}
\end{equation}

Summing up the three parts of satisficing regret, \ie, the satisficing regret from Step 1 \eqref{eq:sat_regret_step-1-weak}, Step 2 \eqref{eq:sat_regret_step-2-weak}, and Step 3 (combination of \eqref{eq:sat_regret_step-3-weak} and \eqref{eq:sat_regret_step-single-weak}), we have that the satisficing regret~of round $i$ is upper bounded by
\begin{equation*}
\left( \frac{6C_1+6}{(1-\alpha)^{\max\{\beta,1/2\}}} + \frac{33\sqrt{2}T}{\min_jB_j'}+2 \right)\gamma_i^{-1 / (1 - \alpha)}\log(4d/\gamma_i)^{\max\{\beta,1/2\}}.
\end{equation*}

\subsection{Proof of Proposition~\ref{prop:transit_Sat_prob_weak}}\label{sec:prop:transit_Sat_prob_weak}
The reason why round $i$ ends within the time horizon $T$ can be divided into two factors, reward LCB-test is triggered or resource UCB-test is triggered. 

\noindent\textbf{LCB-Test}: We first deal with reward LCB-test. Recall that Lemma~\ref{lemma:high_prob_reward_weak} implies that
\begin{equation*}
    \E\left[\max\{Z'-r(\hat{\pi}_i), 0\}\right]\leq \left( \frac{2C_1+2}{(1-\alpha)^{\max\{\beta,1/2\}}} + \frac{11\sqrt{2}T}{\min_jB_j'} \right) \gamma_i\log\left(\frac{4d}{\gamma_i} \right)^{\max\{\beta,1/2\}}.
\end{equation*}

According to the condition of this proposition, we have $i$ satisfies the condition
\begin{equation*}
    64\left( \frac{2C_1+2}{(1-\alpha)^{\max\{\beta,1/2\}}} + \frac{11\sqrt{2}T}{\min_jB_j'} \right)\gamma_i\log(4d/\gamma_i)^{\max\{\beta,1/2\}}\leq\Delta_S',
\end{equation*}
by Markov's inequality, we have
\begin{equation}
\begin{aligned}
\Pr\left(Z'-r(\hat{\pi}_i)>\frac{\Delta_S'}{2}\right) &= \Pr\left(\max\{Z'-r(\hat{\pi}_i),0\}>\frac{\Delta_S'}{2}\right) \leq\frac{\E[\max\{Z'-r(\hat{\pi}_i),0\}]}{\Delta_S'/2}\\
&\leq \frac{\left( \frac{2C_1+2}{(1-\alpha)^{\max\{\beta,1/2\}}} + \frac{11\sqrt{2}T}{\min_jB_j'} \right)\gamma_i\log(4d/\gamma_i)^{\max\{\beta,1/2\}}}{64\left(\frac{2C_1+2}{(1-\alpha)^{\max\{\beta,1/2\}}}+ \frac{11\sqrt{2}T}{\min_jB_j'}\right)\gamma_i\log(4d/\gamma_i)^{\max\{\beta,1/2\}}/2}
\leq\frac{1}{32}
\label{eq:exit-lcb-bad_arm-weak}
\end{aligned}
\end{equation}

Now, if $r(\hat{\pi}_i)< Z'-\Delta_S'/2$, which happens with probability at most $1/32$ according to \eqref{eq:exit-lcb-bad_arm-weak}, we simply bound the probability of round $i$ ends within time horizon $T$ by $1$.

If $r(\hat{\pi}_i)\geq Z'-\Delta_S'/2$, which happens with probability at least $31/32$ and can be upper bounded by $1$, then $r(\hat{\pi}_i)\geq S+\Delta_S'/2$ holds and the confidence interval for any $k\geq0$ can be bounded by
\begin{equation}
\begin{aligned}
\sqrt{\dfrac{{4\log(T_i+k)}}{T_i+k}}&\leq 4\sqrt{\dfrac{\log(T_i)}{T_i}} =  4\sqrt{\frac{\log(\lceil\gamma_i^{-2}\rceil)}{\lceil\gamma_i^{-2}\rceil}} \\
&\leq 4\sqrt{\gamma_i^2\log(4\gamma_i^{-2})}
= 4\sqrt{2}\gamma_i\log(2/\gamma_i)^{1/2} \\
&\leq 4\sqrt{2}\cdot\frac{\Delta_S'}{64} \leq \dfrac{\Delta_S'}{4}
\label{eq:lcb_radius_bound}
\end{aligned}
\end{equation}
where the first inequality is because $T_i>1$, and $f(y)=\log(y)/y$ is monotonically decreasing when $y\geq 3$ and $2\cdot\log(2)/2\geq \log(3)/3$ for all $k\geq 0$, thus $\log(T_i+k)/(T_i+k)\leq 4\log(T_i)/T_i$ for all $k\geq 0$; the second inequality is because $\gamma_i^{-1}\geq1$, thus $\lceil\gamma_i^{-2}\rceil\leq2\gamma_i^{-2}<4\gamma_i^{-2}$; the third inequality is based on the condition of this proposition, where with $C_1>1$, $(1-\alpha)^{\max\{\beta,1/2\}}<1$, and $B_j'\leq T$ for all $j$ we know $64\gamma_i\log(4d/\gamma_i)^{\max\{\beta,1/2\}}\leq\Delta_S'$. Suppose round $i$ is terminated for some $k$ due to reward LCB-test, then
\begin{equation*}
\dfrac{\hat{r}_i^{\text{tot}}}{T_i+k}-\sqrt{\dfrac{4\log(T_i+k)}{T_i+k}}<S=Z'-\Delta_S'\leq r(\hat{\pi}_i)-\dfrac{\Delta_S'}{2}\leq r(\hat{\pi}_i)-2\sqrt{\dfrac{4\log(T_i+k)}{T_i+k}},
\end{equation*}
or equivalently
\begin{equation*}
\dfrac{\hat{r}_i^{\text{tot}}}{T_i+k}\leq r(\hat{\pi}_i)-\sqrt{\dfrac{4\log(T_i+k)}{T_i+k}}.
\end{equation*}
By the Hoeffding's inequality we have
\begin{equation*}
\Pr\left(\dfrac{\hat{r}_i^{\text{tot}}}{T_i+k}\leq r(\hat{\pi}_i)-\sqrt{\dfrac{4\log(T_i+k)}{T_i+k}}\right) \leq \dfrac{1}{(T_i+k)^8}\leq \dfrac{1}{(k+1)^8}.
\end{equation*}
Therefore conditioned on $r(\hat{\pi}_i)\geq Z'-\Delta_S'/2$ the probability that round $i$ is terminated by the end of the horizon is upper bounded by
\begin{equation}
\sum_{k=0}^\infty \dfrac{1}{(k+T_i)^8}\leq \sum_{k=0}^\infty \dfrac{1}{(k+2)^8} \leq \dfrac{1}{32}.
\label{eq:exit-lcb-good_arm-weak}
\end{equation}

Combining \eqref{eq:exit-lcb-bad_arm-weak} and \eqref{eq:exit-lcb-good_arm-weak} and we conclude that the probability that round $i$ terminates by the end of the time horizon due to reward LCB-test is upper bounded by $1/16$ conditioned on $\mathcal{E}$.

\noindent\textbf{UCB-Test}: We then deal with resource UCB-test. Recall that events $\mathcal{C}_{ji}^2$ is defined as
$$
\mathcal{C}_{ji}^2=\left\{\forall k\geq0:\,c_j(\hat{\pi}_i)\leq \frac{B_j'}{T} - 2\sqrt{\frac{4\log(d(T_i+k))}{T_i+k}}\right\}.
$$

If $\mathcal{C}_{ji}^2$ does not hold, which happens with probability at most $1/(32d)$ according to 
Equation~\eqref{eq:pi_hat_cost_2} in Lemma~\ref{lemma:high_prob_cost}, we simply bound the probability of round $i$ ends within time horizon $T$ by $1$.

If $\mathcal{C}_{ji}^2$ holds, and suppose that round $i$ is terminated for some $k$ due to resource UCB-test, then
\begin{equation*}
    \frac{\hat{c}_{ji}^{\text{tot}}(k+T_i)}{k+T_i} + \sqrt{\frac{4\log(d(k+T_i))}{k+T_i}} > \frac{B_j'}{T} \geq c_j(\hat{\pi}_i) + 2\sqrt{\frac{4\log(d(k+T_i))}{k+T_i}}
\end{equation*}
where the first inequality holds due to UCB test, and the second inequality is based on definition of $\mathcal{C}_{ji}^2$. This is equivalent to
\begin{equation*}
    \frac{\hat{c}_{ji}^{\text{tot}}(k+T_i)}{k+T_i} > c_j(\hat{\pi}_i) + \sqrt{\frac{4\log(d(k+T_i))}{k+T_i}}.
\end{equation*}
By the Hoeffding's inequality we have
\begin{equation*}
    \Pr\left(\frac{\hat{c}_{ji}^{\text{tot}}(k+T_i)}{k+T_i} > c_j(\hat{\pi}_i) + \sqrt{\frac{4\log(d(k+T_i))}{k+T_i}}\right)\leq\frac{1}{d^8(T_i+k)^8}\leq\frac{1}{d(k+1)^8}.
\end{equation*}
Therefore conditioned on $\mathcal{C}_{ji}^2$, the probability that round $i$ is terminated by the end of the horizon is upper bounded by
\begin{equation}
\sum_{k=1}^\infty \dfrac{1}{d(k+1)^8}\leq \dfrac{1}{32d}.
\label{eq:exit-ucb-good_arm-weak}
\end{equation}

Combining \eqref{eq:pi_hat_cost_2} and \eqref{eq:exit-ucb-good_arm-weak} and we conclude that the probability that round $i$ terminates by the end of the time horizon due to resource $j$ UCB-test is upper bounded by $1/(16d)$, and sums to $1/16$ among all $d$ types of resources conditioned on the history in the round.

Therefore, combining the terminate probability due to resource UCB-test and that due to reward LCB-test, the probability that round $i$ terminates by the end of the time horizon is upper bounded by $1/16+1/16=1/8$ conditioned on the history in the round.

\section{Proof of Theorem~\ref{thm:real_bwk_sat}}\label{sec:thm:real_bwk_sat}
Following the same pattern as in proof of Theorem~\ref{thm:real_bwk_weak}, we first analyze the satisficing regret conditioned on $\mathcal{E}$. The same as \eqref{eq:budgetj_bound}, if $T$ is sufficiently large with $\sqrt{T/\log(2dT^2)}\geq20\alpha/(b(1-\alpha))$, term $10\sqrt{2}T/(\min_jB_j')$ can be upper bounded by
\begin{equation*}
\begin{aligned}
    \frac{10\sqrt{2}T}{\min_jB_j'} &= 10\sqrt{2}\cdot\left(\min_j \frac{B_j}{T} - \frac{10\alpha}{1-\alpha}\sqrt{\frac{\log(2dT^2)}{T}}\right)^{-1}\\
    &\leq10\sqrt{2}\cdot\left(b - b/2 \right)^{-1} = \frac{20\sqrt{2}}{b}.
\end{aligned}
\end{equation*}

Denote $\beta'=\max\{\beta,1/2\}$ and $K'_{\alpha,\beta}=64\left(\frac{2C_1}{(1-\alpha)^{\beta}} + \frac{20\sqrt{2}}{b}\right) \geq 64 \left(\frac{2C_1}{(1-\alpha)^{\beta}}+\frac{10\sqrt{2}T}{\min_jB'_j}\right)$ (we use $K'_{\alpha,\beta}$ to show a difference with $K_{\alpha,\beta}$). We define $i_0$ as 
$$
i_0=\max\left\{i:64\left(\frac{2C_1}{(1-\alpha)^{\beta}}+ \frac{10\sqrt{2}T}{\min_jB'_j}\right)\cdot\gamma_i\log(4d/\gamma_i)^{\beta'}>\Delta_S'\right\}.
$$
We first prove an upper bound on $i_0$ . Denote $\delta_0\in(0,1)$ the smallest solution to the following
equation
\begin{equation*}
    \delta_0 \log(4d/\delta_0)^{\beta'} = \frac{\Delta_S'}{64\left(\frac{2C_1}{(1-\alpha)^{\beta}}+ \frac{10\sqrt{2}T}{\min_jB_j'}\right)}.
\end{equation*}
We note that, this equation is nearly the same as \eqref{eq:delta_Delta_equation} except a difference in constant term $2C_1$ instead of $2C_1+2$. Therefore, following the same proof pattern in Theorem~\ref{thm:real_bwk_weak}, we get the same bound as \eqref{eq:4d/delta0_upperbound} that
\begin{equation}
\frac{4d}{\delta_0} \leq 4\left(\frac{\sqrt{2}d(4\beta')^{\beta'}K'_{\alpha,\beta}}{\Delta_S'}\right)^{4/3} \leq \left(\frac{4\sqrt{2}d(4\beta')^{\beta'}K'_{\alpha,\beta}}{\Delta_S'}\right)^{4/3},
\label{eq:4d/delta0_upperbound-strong}
\end{equation}
where the difference in constant term is included in $K'_{\alpha,\beta}$ instead of $K_{\alpha,\beta}$. Plugging \eqref{eq:4d/delta0_upperbound-strong} into the $\log(4d/\delta_0)$ term in the definition of $\delta_0$, we write $D'_{\alpha,\beta}=4\sqrt{2}d(4\beta')^{\beta'}K'_{\alpha,\beta}$, $L'_{\alpha,\beta}=(4/3)^{\beta'}K'_{\alpha,\beta}$, and we have
$$
\delta_0\geq\frac{\Delta_S'}{L'_{\alpha,\beta}\log(D'_{\alpha,\beta}/\Delta_S')^{\beta'}}.
$$
By definition of $i_0$ we have that $i_0$ is the largest integer that satisfies $\gamma_{i_0}\geq \delta_0$. Therefore we get
$$
\gamma_{i_0}\geq\frac{\Delta_S'}{L'_{\alpha,\beta}\log(D'_{\alpha,\beta}/\Delta_S')^{\beta'}},
$$
thus
$$
i_0 \leq \frac{\alpha}{1-\alpha}\log_2\left(\frac{L'_{\alpha,\beta}\log(D'_{\alpha,\beta}/\Delta_S')^{\beta'}}{\Delta_S'}\right),
$$
or equivalently, (the same as \eqref{eq:2^i0_upperbound} with difference in constant terms)
\begin{equation}
2^{i_0} \leq \left(\frac{L'_{\alpha,\beta}\log(D'_{\alpha,\beta}/\Delta_S')^{\beta'}}{\Delta_S'}\right)^{\alpha/(1-\alpha)}.
\label{eq:2^i0_upperbound-strong}
\end{equation}

Using the result of Proposition~\ref{prop:round_Sat_regret}, we have that the satisficing regret by round $i_0$ is upper bounded by
\begin{equation}
\begin{aligned}
&\sum_{i=1}^{i_0}4\left(\frac{2C_1}{(1-\alpha)^\beta} + \frac{10\sqrt{2}T}{\min_jB_j'}\right)\gamma_i^{-\alpha / (1 - \alpha)}\log(4d/\gamma_i)^{\beta'}\\
\leq&\sum_{i=1}^{i_0}4\left(\frac{2C_1}{(1-\alpha)^\beta} + \frac{10\sqrt{2}T}{\min_jB_j'}\right)\gamma_i^{-\alpha / (1 - \alpha)}\log(4d/\delta_0)^{\beta'}\\
\leq&\sum_{i=1}^{i_0} \frac{K'_{\alpha,\beta}}{16}\cdot 2^i \log(4d/\delta_0)^{\beta'}\\
\leq&\frac{1}{8}K'_{\alpha,\beta}2^{i_0}\log(4d/\delta_0)^{\beta'}\\
\leq&\frac{(4/3)^{\beta'}}{8}K'_{\alpha,\beta}\left( \frac{L'_{\alpha,\beta} \log(D'_{\alpha,\beta}  / \Delta_S')^{\beta'}}{\Delta_S'} \right)^{\alpha/(1-\alpha)}\log\left(\frac{D'_{\alpha,\beta}}{\Delta_S'}\right)^{\beta'},
\label{eq:sat_regret-before-i0-strong}
\end{aligned}
\end{equation}
where the first inequality is because by definition of $\delta_0$ and $i_0$, we have $\gamma_i=2^{-i(1-\alpha)/\alpha}\geq \gamma_{i_0}$ for all $i\leq i_0$, the second inequality is by plugging into the definition of $\gamma_i=2^{-i(1-\alpha)/\alpha}$ and definition of $K'_{\alpha,\beta}$, the third inequality is based on $\sum_{i=1}^{i_0}2^{i} \leq 2\cdot2^{i_0}$, and the fourth inequality is by plugging into \eqref{eq:4d/delta0_upperbound-strong} and \eqref{eq:2^i0_upperbound-strong}.

By the results of Proposition~\ref{prop:transit_Sat_prob}, for every $k \geq 1$, the probability that round $i_0 + k$ starts within the time horizon is at most $8^{-k+1}$. Therefore the regret incurred after round $i_0$ is upper bounded by
\begin{equation}
\begin{aligned}
&\sum_{k=1}^{\infty}4\left(\frac{2C_1}{(1-\alpha)^\beta} + \frac{10\sqrt{2}T}{\min_jB_j'}\right)\gamma_{i_0+k}^{-\alpha / (1 - \alpha)}\log(4d/\gamma_{i_0+k})^{\beta'}\cdot8^{-k+1}\\
\leq&\sum_{k=1}^{\infty}\frac{K'_{\alpha,\beta}}{16}\cdot 2^{i_0+k} \left(\log(4d/\gamma_{i_0})+k\right)^{\beta'}\cdot8^{-k+1}\\
\leq&\sum_{k=1}^{\infty}\frac{2^{\beta'}K'_{\alpha,\beta}}{16}\cdot 2^{i_0+k} \log(4d/\gamma_{i_0})^{\beta'} \cdot k^{\beta'} \cdot8^{-k+1}\\
\leq&2^{\beta'-1}K'_{\alpha,\beta}2^{i_0}\log(4d/\delta_0)^{\beta'}\cdot M'_{\beta'}\\
\leq&\frac{(8/3)^{\beta'}}{2}K'_{\alpha,\beta}\left( \frac{L'_{\alpha,\beta} \log(D'_{\alpha,\beta}  / \Delta_S')^{\beta'}}{\Delta_S'} \right)^{\alpha/(1-\alpha)}\log\left(\frac{D'_{\alpha,\beta}}{\Delta_S'}\right)^{\beta'}M'_{\beta'}
\label{eq:sat_regret-after-i0-strong}
\end{aligned}
\end{equation}
where $M'_{\beta'}=\sum_{k=1}^\infty k^{\beta'}\cdot4^{-k}$. The first inequality is by plugging into the definition of $K'_{\alpha,\beta}$ and $\gamma_i=2^{-i(1-\alpha)/\alpha}$, the second inequality is because $\log(4d/\gamma_{i_0})>1$ and $(1-\alpha)/\alpha\leq1$ for $1/2\leq\alpha<1$, and for any $a,b\geq 1$, we have $2ab\geq ab+1\geq a+b$, the third inequality is based on $\sum_{k=1}^{\infty}2^k\cdot8^{-k+1}=8\sum_{k=1}^{\infty}4^{-k}$. Therefore, combining \eqref{eq:sat_regret-before-i0-strong} and \eqref{eq:sat_regret-after-i0-strong}, the total satisficing regret conditioned on $\mathcal{E}$ is upper bounded by
\begin{equation*}
\left(\frac{(4/3)^{\beta'}}{8}+\frac{(8/3)^{\beta'}}{2}\right)K'_{\alpha,\beta}\left( \frac{L'_{\alpha,\beta} \log(D'_{\alpha,\beta}  / \Delta_S')^{\beta'}}{\Delta_S'} \right)^{\alpha/(1-\alpha)}\log\left(\frac{D'_{\alpha,\beta}}{\Delta_S'}\right)^{\beta'}(M'_{\beta'}+1).
\end{equation*}

If $\Delta_S'\geq\delta(T)$, recall $\Delta^*_S=Z^*-S$ and $\Delta_S'=Z'-S$. Proposition~\ref{prop:optimal_gap} implies $\Delta^*_S-\Delta'_S\leq\delta(T)$. Then we have
$$
\Delta_S' \geq \frac{\delta(T)}{2} + \frac{\Delta^*_S-\delta(T)}{2} = \frac{\Delta^*_S}{2}.
$$
Combining this with $\text{Pr}(\mathcal{E}^c)\leq1/T$ according to Proposition~\ref{prop:clean_event_bound} and total satisficing regret incurred conditioned on negation of $\mathcal{E}$ is at most $T$, total satisficing regret is upper bounded by
\begin{equation}
1 + \left(\frac{(4/3)^{\beta'}}{8}+\frac{(8/3)^{\beta'}}{2}\right)K'_{\alpha,\beta}\left( \frac{2L'_{\alpha,\beta} \log(2D'_{\alpha,\beta}  / \Delta_S^*)^{\beta'}}{\Delta_S^*} \right)^{\alpha/(1-\alpha)}\log\left(\frac{2D'_{\alpha,\beta}}{\Delta_S^*}\right)^{\beta'}(M'_{\beta'}+1).
\label{eq:sat_regret-const-strong}
\end{equation}

Since by definition of $t_i=\lceil2^{i/\alpha}\rceil$, the number of rounds that end within the time horizon is bounded by $2^{i/\alpha}\leq T$, \ie, $i\leq\lfloor\alpha\log_2(T) \rfloor$ because $i$ is an integer. Then, the maximum number of rounds is bounded by $i_{\max}=\lceil\alpha\log_2(T)\rceil\leq\alpha\log_2(T)+1$. Recall $\gamma_i=2^{-i(1-\alpha)/\alpha}$, for any $i\leq i_{\max}$, we have
\begin{equation}
\begin{aligned}
\gamma_i&\geq \gamma_{i_{\max}}\geq2^{-(1-\alpha)\lceil\alpha\log_2(T)\rceil/\alpha} \\
&\geq 2^{-(1-\alpha)\log_2(T)/\alpha-(1-\alpha)/\alpha}\\
&\geq 2^{-\log_2(T)-1}=1/(2T),
\label{eq:gammai_1/T}
\end{aligned}
\end{equation}
where the third inequality is because $(1-\alpha)/\alpha\leq1$ for $1/2\leq\alpha<1$. So, $2/\gamma_i\leq 4T$ for all $i\leq i_{\max}$. Therefore, the satisficing regret is upper bounded by the sum of cumulative regret with respect to $Z'$ and cumulative difference between $Z'$ and $Z^*$
$$
\begin{aligned}
\texttt{Regret}_S&\leq 1 + T(Z^*-Z') + \sum_{i=1}^{i_{\max}}4\left(\frac{2C_1}{(1-\alpha)^\beta}+ \frac{10\sqrt{2}T}{\min_jB'_j}\right)\gamma_i^{-\alpha / (1 - \alpha)}\log(4d/\gamma_i)^{\max\{\beta,1/2\}}\\
&\leq1 + \frac{10\alpha}{b(1-\alpha)} \sqrt{T\log(2dT^2)} + \sum_{i=1}^{i_{\max}}4\left(\frac{2C_1}{(1-\alpha)^\beta} + \frac{10\sqrt{2}T}{\min_jB_j'}\right)2^{i+1}\log(8dT)^{\max\{\beta,1/2\}}\\
&\leq1 + \frac{10\sqrt{2}\alpha}{b(1-\alpha)} T^{1/2}\log(4dT)^{1/2} + 16\left(\frac{2C_1}{(1-\alpha)^\beta} + \frac{10\sqrt{2}T}{\min_jB_j'}\right)2^{i_{\max}}\log(8dT)^{\max\{\beta,1/2\}}\\
&\leq1 + \left(32\left(\frac{2C_1}{(1-\alpha)^\beta} + \frac{10\sqrt{2}T}{\min_jB_j'}\right)+\frac{10\sqrt{2}\alpha}{b(1-\alpha)} \right)T^{\alpha}\log(8dT)^{\max\{\beta,1/2\}},
\end{aligned}
$$
where the first inequality is based on Proposition~\ref{prop:round_Sat_regret} and $\max\{S-r,0\}\leq(Z^*-Z')+\max\{Z'-r,0\}$, the second inequality is based on $\Delta_S^*-\Delta_S'\leq\delta(T)$ from Proposition~\ref{prop:optimal_gap} and $2/\gamma_i\leq4T$ from \eqref{eq:gammai_1/T}, the third inequality is based on $\sum_{i=1}^{i_{\max}}2^i\leq2^{i_{\max}+1}$, and the fourth inequality is plugging into $2^{i_{\max}}\leq2^{\alpha\log_2(T)+1}\leq2T^\alpha$. Recall that if $T$ is sufficiently large with $\sqrt{T/\log(2dT^2)}\geq20\alpha/(b(1-\alpha))$, term $10\sqrt{2}T/(\min_jB_j')$ can be upper bounded by ${20\sqrt{2}}/{b}$, and the satisficing regret can further be upper bounded by
\begin{equation}
\texttt{Regret}_S\leq 1 + \left(\frac{64C_1}{(1-\alpha)^\beta}+ \frac{640\sqrt{2}}{b}+\frac{10\sqrt{2}\alpha}{b(1-\alpha)} \right)T^{\alpha}\log(8dT)^{\max\{\beta,1/2\}}.
\label{eq:sat_regret-sublinear-strong}
\end{equation}

Finally we deal with $\Delta_S'<\delta(T)$ for constant regret bound, where we bound the satisficing regret with $T^\alpha\cdot\text{polylog}(T)$. $\Delta_S^*-\Delta_S'\leq\delta(T)$ implies $\Delta_S^*<2\delta(T)$. Similar to \eqref{eq:Delta'<delta(T)}, we still have
\begin{equation}
T\leq\frac{1600\alpha^2}{b^2(1-\alpha)^2(\Delta_S^*)^2}\log\left(\frac{800\sqrt{2d}\alpha^2}{b^2(1-\alpha)^2(\Delta_S^*)^2}\right).
\label{eq:Delta'<delta(T)-strong}
\end{equation}
Using the sub-linear satisficing regret bound in \eqref{eq:sat_regret-sublinear-strong} and \eqref{eq:Delta'<delta(T)-strong}, we have
\begin{equation*}
\begin{aligned}
\texttt{Regret}_S\leq&1+\left(\frac{64C_1}{(1-\alpha)^\beta}+\frac{640\sqrt{2}}{b}+\frac{10\sqrt{2}\alpha}{b(1-\alpha)}\right)\left(\frac{1600\alpha^2}{b^2(1-\alpha)^2}\right)^\alpha\frac{1}{(\Delta_S^*)^{2\alpha}}\\
&\cdot\log\left(\frac{800\sqrt{2d}\alpha^2}{b^2(1-\alpha)^2(\Delta_S^*)^2}\right)^\alpha\log\left(\frac{12800d\alpha^2}{b^2(1-\alpha)^2(\Delta_S^*)^2}\log\left(\frac{800\sqrt{2d}\alpha^2}{b^2(1-\alpha)^2(\Delta_S^*)^2}\right)\right)^{\beta'}\\
\leq&1+\left(\frac{64C_1}{(1-\alpha)^\beta}+\frac{640\sqrt{2}}{b}+\frac{10\sqrt{2}\alpha}{b(1-\alpha)}\right)\left(\frac{1600\alpha^2}{b^2(1-\alpha)^2}\right)^\alpha \left(\frac{1}{\Delta_S^*}\right)^{\alpha/(1-\alpha)}\\
&\cdot\log\left(\frac{800\sqrt{2d}\alpha^2}{b^2(1-\alpha)^2(\Delta_S^*)^2}\right)^\alpha\log\left(\frac{12800d\alpha^2}{b^2(1-\alpha)^2(\Delta_S^*)^2}\log\left(\frac{800\sqrt{2d}\alpha^2}{b^2(1-\alpha)^2(\Delta_S^*)^2}\right)\right)^{\beta'},
\end{aligned}
\end{equation*}
where the second inequality is based on $(\Delta_S^*)^{-2\alpha}\leq(\Delta_S^*)^{-\alpha/(1-\alpha)}$ implied by $\alpha\geq1/2$ and $\Delta_S^*\leq1$. 

Combining the two parts of results, \ie, the constant satisficing regret \eqref{eq:sat_regret-const-strong} and the sub-linear satisficing regret \eqref{eq:sat_regret-sublinear-strong}, we have completed the proof of Theorem~\ref{thm:real_bwk_sat} by putting all constant terms into the $\text{polylog}$ term.

\subsection{Reward Bound under Strong Regret Condition} 

Replacing the standard regret condition by the strong regret condition does not change the algorithm or the resource-consumption analysis. Therefore, the resource-consumption guarantees in Lemma~\ref{lemma:high_prob_cost} remain valid. Under Condition~\ref{condition:strong_regret_bound}, we can directly control the expected sum of positive reward deficits along the oracle trajectory, which yields the following sharper bound on $\E\left[\max\{Z'-r(\hat{\pi}_i), 0\}\right]$.

\begin{lemma}
    Under Condition~\ref{condition:strong_regret_bound}, the expected regret of a single pull on distribution $\hat{\pi}_i$ in round $i$ can be upper bounded in expectation as
    \begin{equation*}
        \E\left[\max\{Z'-r(\hat{\pi}_i), 0\}\right]\leq\left(\frac{2C_1}{(1-\alpha)^\beta} + \frac{10\sqrt{2}T}{\min_jB_j'} \right)\gamma_i\log\left(\frac{4d}{\gamma_i} \right)^{\max\{\beta,1/2\}}.
    \end{equation*}
\label{lemma:high_prob_reward}
\end{lemma}
\begin{proof}
Recall that $\{\tilde{\pi}_s\}_{s\in[t_i]}$ is realized trajectory of running $\texttt{ALG}(t_i,\{b_{ji}\})$. Since $\hat{\pi}_i$ is calculated based on $\{\tilde{\pi}_s\}_{s\in[t_i]}$ in round $i$, we have
\begin{equation*}
\begin{aligned}
    &\E\left[\max\{Z'-r(\hat{\pi}_i), 0\}\right]=\frac{1}{t_i}\E\left[\max\left\{t_iZ'-(\sum_{s=1}^{t_i}r(\tilde{\pi}_s))\cdot(1-w_i), 0\right\}\right]\\
    \leq &\frac{1}{t_i}\E\left[\sum_{s=1}^{t_i}\max\{Z'-r(\tilde{\pi}_s),0\}\right] + \frac{w_i}{t_i}\E\left[\sum_{s=1}^{t_i}r(\tilde{\pi}_s)\right]\\
    \leq &C_1t_i^{\alpha-1}\log (t_i)^{\beta} + w_i.
\end{aligned}
\end{equation*}
For $t_i^{\alpha-1}\log (t_i)^{\beta}$, it can be simply upper bounded by $\frac{2}{(1-\alpha)^\beta}\gamma_i\log(\frac{2}{\gamma_i})^{\beta}$. For $w_i$, it can be upper bounded by $\frac{10\sqrt{2}T}{\min_jB_j'}\gamma_i\log\left({4d}/{\gamma_i}\right)^{1/2}$.

Therefore, the expected satisficing regret of a single pull can be upper bounded by
\begin{equation*}
    \E\left[\max\{Z'-r(\hat{\pi}_i), 0\}\right]\leq\left(\frac{2C_1}{(1-\alpha)^\beta} + \frac{10\sqrt{2}T}{\min_jB_j'} \right)\gamma_i\log\left(\frac{4d}{\gamma_i} \right)^{\max\{\beta,1/2\}}.\qedhere
\end{equation*}
\end{proof}

\subsection{Proof of Proposition~\ref{prop:round_Sat_regret}}\label{sec:prop:round_Sat_regret}
Recall Lemma~\ref{lemma:high_prob_reward} implies that the satisficing regret of a single pull on $\hat{\pi}_i$ can be upper bounded by
\begin{equation}
\begin{aligned}
\E[\max\{S-r(\hat{\pi}_i) ,0\}] &\leq \E\left[\max\{Z'-r(\hat{\pi}_i), 0\}\right] \\
&\leq\left(\frac{2C_1}{(1-\alpha)^\beta} + \frac{10\sqrt{2}T}{\min_jB_j'} \right)\gamma_i\log\left(\frac{4d}{\gamma_i} \right)^{\max\{\beta,1/2\}}.
\label{eq:sat_regret_step-single-strong}
\end{aligned}
\end{equation}
The regret in round $i$ can be divided into three parts. Running $\texttt{ALG}(t_i,\{b_{ji}\})$, forced-sampling, and LCB-UCB test.

Recall that $t_i=\lceil \gamma_i^{-1/(1-\alpha)}\rceil\leq 2\gamma_i^{-1/(1-\alpha)}$, then $\log(t_i)\leq\log(2\gamma_i^{-1/(1-\alpha)})\leq \log(2/\gamma_i)/(1-\alpha)$ and $t_i^\alpha = \lceil2^{i/\alpha}\rceil^\alpha \leq 2^{\alpha(1+i/\alpha)} \leq 2\cdot2^i = 2\gamma_i^{-\alpha/(1-\alpha)}$. The satisficing regret incurred when running $\texttt{ALG}(t_i,\{b_{ji}\})$ is upper bounded by the expected strong regret bound
\begin{equation}
\begin{aligned}
\E\left[\sum_{s=1}^{t_i}\max\{S-r(\tilde{\pi}_s) ,0\} \right] &\leq \E\left[\sum_{s=1}^{t_i} \max\{Z'-r(\tilde{\pi}_s),0\}\right]\leq C_1t_i^\alpha\log(t_i)^\beta\\
&\leq\frac{2C_1}{(1 - \alpha)^{\beta}} \gamma_i^{-\alpha / (1 - \alpha)} \log \left( {2}/{\gamma_i} \right)^{\beta}\\
&\leq\left(\frac{2C_1}{(1-\alpha)^\beta} + \frac{10\sqrt{2}T}{\min_jB_j'} \right)\gamma_i^{-\alpha / (1 - \alpha)}\log(4d/\gamma_i)^{\max\{\beta,1/2\}}.
\label{eq:sat_regret_step-1-strong}
\end{aligned}
\end{equation}

Similar to \eqref{eq:sat_regret_step-2-weak}, the satisficing regret incurred when pulling $\hat{\pi}_i$ in forced sampling for $T_i=\lceil\gamma_i^{-2}\rceil$ time steps is upper bounded by
\begin{equation}
\begin{aligned}
&\lceil\gamma_i^{-2}\rceil\cdot \left(\frac{2C_1}{(1-\alpha)^\beta} + \frac{10\sqrt{2}T}{\min_jB_j'} \right)\gamma_i\log(4d/\gamma_i)^{\max\{\beta,1/2\}}\\
\leq&\frac{2}{\gamma_i^2}\cdot\left(\frac{2C_1}{(1-\alpha)^\beta} + \frac{10\sqrt{2}T}{\min_jB_j'} \right)\gamma_i\log(4d/\gamma_i)^{\max\{\beta,1/2\}}\\
\leq&2\left(\frac{2C_1}{(1-\alpha)^\beta} + \frac{10\sqrt{2}T}{\min_jB_j'} \right)\gamma_i^{-\alpha / (1 - \alpha)}\log(4d/\gamma_i)^{\max\{\beta,1/2\}}
\label{eq:sat_regret_step-2-strong}
\end{aligned}
\end{equation}

If $r(\hat{\pi}_i)<S$ (otherwise the satisficing regret is $0$), for every $k$, round $i$ is not terminated after pulling $\hat{\pi}_i$ for $k$ time steps in addition to first $T_i$ time steps occurs with probability at most
\begin{equation*}
\begin{aligned}
&\Pr\left(\dfrac{\hat{r}^{\text{tot}}_i}{T_i+k}-\sqrt{\dfrac{4\log(T_i+k)}{T_i+k}}>S\right)\\
\leq &\Pr\left(\dfrac{\hat{r}^{\text{tot}}_i}{T_i+k}-\sqrt{\dfrac{4\log(T_i+k)}{T_i+k}}>r(\hat{\pi}_i)\right)\\
\leq&\dfrac{1}{(T_i+k)^8}\leq \dfrac{1}{(k+1)^8}.
\end{aligned}
\end{equation*}
Therefore expected number of arm pulls of $\hat{\pi}_i$ after first pulling it $T_i$ time steps is upper bounded by
\begin{equation}
\begin{aligned}
&\E\left[\text{Number of time steps }\hat{\pi}_i\text{ is pulled in round }i\text{ after the first }T_i \text{ time steps} \,|\, \hat{\pi}_i\right]\\
=&\sum_{s=1}^{\infty}\Pr(k\geq s\,|\, \hat{\pi}_i)
\leq \sum_{s=1}^\infty \dfrac{1}{(T_i+s-1)^8}
\leq \sum_{s=1}^\infty \dfrac{1}{(s+1)^8} \leq 1.
\label{eq:sat_regret_step-3-strong}
\end{aligned}
\end{equation}

Summing up the three parts of satisficing regret, \ie, the satisficing regret from Step 1 \eqref{eq:sat_regret_step-1-strong}, Step 2 \eqref{eq:sat_regret_step-2-strong}, and Step 3 (combination of \eqref{eq:sat_regret_step-3-strong} and \eqref{eq:sat_regret_step-single-strong}), we have that the satisficing regret~of round $i$ is upper bounded by
\begin{equation*}
4\left(\frac{2C_1}{(1-\alpha)^\beta} + \frac{10\sqrt{2}T}{\min_jB_j'} \right)\gamma_i^{-\alpha / (1 - \alpha)}\log(4d/\gamma_i)^{\max\{\beta,1/2\}}.
\end{equation*}

\subsection{Proof of Proposition~\ref{prop:transit_Sat_prob}}\label{sec:prop:transit_Sat_prob}
The reason why round $i$ ends within the time horizon $T$ can be divided into two factors, reward LCB-test is triggered or resource UCB-test is triggered. Since the difference in regret bound (Condition \ref{condition:std_regret_bound} and \ref{condition:strong_regret_bound}) does not affect the costs and, consequently, does not impact the UCB-test, we can still bound the probability that round $i$ terminates by $1/16$. Therefore we only focus on LCB-test, and follow similar proof pattern in Proposition~\ref{prop:transit_Sat_prob_weak} with difference in constant terms.

By Lemma~\ref{lemma:high_prob_reward} we have
\begin{equation*}
\E[\max\{Z'-r(\hat{\pi}_i), 0\}] \leq \left(\frac{2C_1}{(1-\alpha)^\beta} + \frac{10\sqrt{2}T}{\min_jB_j'}\right) \gamma_i \log(4d/\gamma_i)^{\max\{\beta,1/2\}}.
\end{equation*}

According to the condition of this proposition, we have $i$ satisfies the condition
\begin{equation*}
    64\left(\frac{2C_1}{(1-\alpha)^\beta} + \frac{10\sqrt{2}T}{\min_jB_j'} \right)\gamma_i\log(4d/\gamma_i)^{\max\{\beta,1/2\}}\leq\Delta_S',
\end{equation*}
by Markov's inequality, we have
\begin{equation*}
\begin{aligned}
&\Pr\left(Z'-r(\hat{\pi}_i)>\frac{\Delta_S'}{2}\right) = \Pr\left(\max\{Z'-r(\hat{\pi}_i),0\}>\frac{\Delta_S'}{2}\right)\\
\leq&\frac{\E[\max\{Z'-r(\hat{\pi}_i),0\}]}{\Delta_S'/2}\leq\frac{\left(\frac{2C_1}{(1-\alpha)^\beta} + \frac{10\sqrt{2}T}{\min_jB_j'} \right)\gamma_i\log(4d/\gamma_i)^{\max\{\beta,1/2\}}}{64\left(\frac{2C_1}{(1-\alpha)^\beta} + \frac{10\sqrt{2}T}{\min_jB_j'} \right)\gamma_i\log(4d/\gamma_i)^{\max\{\beta,1/2\}}/2}
\leq\frac{1}{32}
\end{aligned}
\end{equation*}

If $r(\hat{\pi}_i)\geq Z'-\Delta_S'/2$, which happens with probability at least $31/32$, then $r(\hat{\pi}_i)\geq S+\Delta_S'/2$ holds. Similar to \eqref{eq:lcb_radius_bound}, and given that $64\gamma_i\log(4d/\gamma_i)^{\max\{\beta,1/2\}}\leq\Delta_S'$ still holds despite a different constant term in the proposition condition, we know for all $k\geq 0$,
\begin{equation*}
\sqrt{\dfrac{{4\log(T_i+k)}}{T_i+k}}\leq 4\sqrt{\dfrac{\log(T_i)}{T_i}}\leq \dfrac{\Delta_S'}{4}.
\end{equation*}
Suppose round $i$ is terminated for some $k$ due to reward LCB-test, then
\begin{equation*}
\dfrac{\hat{r}_i^{\text{tot}}}{T_i+k}-\sqrt{\dfrac{4\log(T_i+k)}{T_i+k}}<S=Z'-\Delta_S'\leq r(\hat{\pi}_i)-\dfrac{\Delta_S'}{2}\leq r(\hat{\pi}_i)-2\sqrt{\dfrac{4\log(T_i+k)}{T_i+k}},
\end{equation*}
or equivalently
\begin{equation*}
\dfrac{\hat{r}_i^{\text{tot}}}{T_i+k}\leq r(\hat{\pi}_i)-\sqrt{\dfrac{4\log(T_i+k)}{T_i+k}}.
\end{equation*}
By the Hoeffding's inequality we have
\begin{equation*}
\Pr\left(\dfrac{\hat{r}_i^{\text{tot}}}{T_i+k}\leq r(\hat{\pi}_i)-\sqrt{\dfrac{4\log(T_i+k)}{T_i+k}}\right)\leq \dfrac{1}{(T_i+k)^8}\leq \dfrac{1}{(k+1)^8}.
\end{equation*}
Therefore conditioned on $r(\hat{\pi}_i)\geq Z'-\Delta_S'/2$ the probability that round $i$ is terminated by the end of the horizon is upper bounded by
\begin{equation*}
\sum_{k=1}^\infty \dfrac{1}{(k+1)^8}\leq \dfrac{1}{32}.\label{eq:prob_bound_2}
\end{equation*}

Combining above and we conclude that the probability that round $i$ terminates by the end of the time horizon due to reward LCB-test is upper bounded by $1/16$ conditioned on $\bigcap_{i'<i}\mathcal{E}_{i'}$. Also, recall in Proposition~\ref{prop:transit_Sat_prob_weak}, the probability that round $i$ terminates by the end of the time horizon due to resource UCB-test is upper bounded by $1/16$ conditioned on $\bigcap_{i'<i}\mathcal{E}_{i'}$. Therefore, combining resource UCB-test and reward LCB-test, the probability that round $i$ terminates by the end of the time horizon is upper bounded by $1/16+1/16=1/8$ conditioned on $\bigcap_{i'<i}\mathcal{E}_{i'}$.

\section{Proof of Theorem~\ref{thm:real_bwk_std}}\label{sec:thm:real_bwk_std}
As $\hat{\pi}_i$ is derived from oracle running $\texttt{ALG}$ under adjusted resource capacities $\bm{B}'$, the cumulative standard regret can be divided into
\begin{equation}
\texttt{Regret} \leq 1+ \E\left[\sum_{t=1}^T(Z'-r(\pi_t)) \right] + T(Z^*-Z').
\label{eq:std_regret_bwk_decompose}
\end{equation}

For the first term in \eqref{eq:std_regret_bwk_decompose}, which shows the cumulative regret with respect to $Z'$, we first bound the standard regret of a single pull on $\hat{\pi}_i$ with respect to $Z'$ by
\begin{equation*}
\begin{aligned}
\E\left[Z'-r(\hat{\pi}_i)\right]=&\frac{1}{t_i}\E\left[t_iZ'-(\sum_{s=1}^{t_i}r(\tilde{\pi}_s))\cdot(1-w_i)\right]\\
\leq &\frac{1}{t_i}\E\left[\sum_{s=1}^{t_i}(Z'-r(\tilde{\pi}_s))\right] + \frac{w_i}{t_i}\E\left[\sum_{s=1}^{t_i}r(\tilde{\pi}_s)\right]\\
\leq &C_1t_i^{\alpha-1}\log (t_i)^{\beta} + w_i,
\end{aligned}
\end{equation*}
where the final inequality is based on Condition~\ref{condition:std_regret_bound} and $r(\pi)\in[0,1]$ for any $\pi$. Following the same calculation in Lemma~\ref{lemma:high_prob_reward_weak}, we can combine \eqref{eq:single_pull-weak-sublinear} and \eqref{eq:single_pull-weak-wi} and bound the standard regret of a single pull on $\hat{\pi}_i$ with respect to $Z'$ by
\begin{equation*}
    \E\left[Z'-r(\hat{\pi}_i)\right]\leq\left(\frac{2C_1}{(1-\alpha)^\beta} + \frac{10\sqrt{2}T}{\min_jB_j'} \right)\gamma_i\log\left(\frac{4d}{\gamma_i} \right)^{\max\{\beta,1/2\}}.
\end{equation*}

The standard regret with respect to $Z'$ in round $i$ can be divided into three parts. Running $\texttt{ALG}(t_i,\{b_{ji}\})$, forced-sampling, and LCB-UCB test.

The standard regret incurred when running $\texttt{ALG}(t_i,\{b_{ji}\})$ can directly bounded as
\begin{equation}
\begin{aligned}
\E\left[\sum_{s=1}^{t_i}(Z'-r(\tilde{\pi}_s)) \right] \leq& C_1t_i^\alpha\log(t_i)^\beta\\
\leq& \frac{2C_1}{(1 - \alpha)^{\beta}} \gamma_i^{-\alpha / (1 - \alpha)} \log \left( {2}/{\gamma_i} \right)^{\beta}\\
\leq& \left(\frac{2C_1}{(1-\alpha)^\beta} + \frac{10\sqrt{2}T}{\min_jB_j'} \right)\gamma_i^{-\alpha / (1 - \alpha)}\log(4d/\gamma_i)^{\max\{\beta,1/2\}},
\label{eq:std_regret_bwk_step1}
\end{aligned}
\end{equation}
where the second inequality is because $t_i=\lceil \gamma_i^{-1/(1-\alpha)}\rceil\leq 2\gamma_i^{-1/(1-\alpha)}$, then $\log(t_i)\leq\log(2\gamma_i^{-1/(1-\alpha)})\leq \log(2/\gamma_i)/(1-\alpha)$ and $t_i^\alpha = \lceil2^{i/\alpha}\rceil^\alpha \leq 2^{\alpha(1+i/\alpha)} \leq 2\cdot2^i = 2\gamma_i^{-\alpha/(1-\alpha)}$.

The standard regret incurred when pulling $\hat{\pi}_i$ in forced sampling for $T_i=\lceil\gamma_i^{-2}\rceil$ time steps is upper bounded by
\begin{equation}
\begin{aligned}
&\lceil\gamma_i^{-2}\rceil\cdot (C_1t_i^{\alpha-1}(\log t_i)^{\beta} + w_i)\\
\leq&\frac{2}{\gamma_i^2}\cdot\left(\frac{2C_1}{(1-\alpha)^\beta} + \frac{10\sqrt{2}T}{\min_jB_j'} \right)\gamma_i\log(4d/\gamma_i)^{\max\{\beta,1/2\}}\\
\leq&2\left(\frac{2C_1}{(1-\alpha)^\beta} + \frac{10\sqrt{2}T}{\min_jB_j'} \right)\gamma_i^{-\alpha / (1 - \alpha)}\log(4d/\gamma_i)^{\max\{\beta,1/2\}}.
\label{eq:std_regret_bwk_step2}
\end{aligned}
\end{equation}

As $r(\hat{\pi}_i)< Z' <Z^* <S$, round $i$ is not terminated after pulling $\hat{\pi}_i$ for $k$ time steps in addition to first $T_i$ time steps occurs with probability at most
\begin{equation*}
\begin{aligned}
&\Pr\left(\dfrac{\hat{r}^{\text{tot}}_i}{T_i+k}-\sqrt{\dfrac{4\log(T_i+k)}{T_i+k}}>S\right)\\
\leq &\Pr\left(\dfrac{\hat{r}^{\text{tot}}_i}{T_i+k}-\sqrt{\dfrac{4\log(T_i+k)}{T_i+k}}>r(\hat{\pi}_i)\right)\\
\leq&\dfrac{1}{(T_i+k)^8}\leq \dfrac{1}{(k+1)^8}.
\end{aligned}
\end{equation*}
Therefore expected number of arm pulls of $\hat{\pi}_i$ after first pulling it $T_i$ time steps is upper bounded by
\begin{equation}
\begin{aligned}
&\E\left[\text{Number of time steps }\hat{\pi}_i\text{ is pulled in round }i\text{ after the first }T_i \text{ time steps} \,|\, \hat{\pi}_i \right]\\
=&\sum_{s=1}^{\infty}\Pr(k\geq s\,|\, \hat{\pi}_i)
\leq \sum_{s=1}^\infty \dfrac{1}{(T_i+s-1)^8}
\leq \sum_{s=1}^\infty \dfrac{1}{(s+1)^8} \leq 1.
\label{eq:std_regret_bwk_step3}
\end{aligned}
\end{equation}

Summing up the three parts of standard regret, \ie, \eqref{eq:std_regret_bwk_step1}, \eqref{eq:std_regret_bwk_step2}, and \eqref{eq:std_regret_bwk_step3}, we have that the standard regret with respect to $Z'$ of round $i$ is upper bounded by
\begin{equation*}
4\left(\frac{2C_1}{(1-\alpha)^\beta} + \frac{10\sqrt{2}T}{\min_jB_j'} \right)\gamma_i^{-\alpha / (1 - \alpha)}\log(4d/\gamma_i)^{\max\{\beta,1/2\}}.
\end{equation*}

For the second term in \eqref{eq:std_regret_bwk_decompose}, which shows the cumulative difference in optimal reward between resource constraints $B_j$ and $B_j'$. Recall the difference between $Z^*$ and $Z'$ is bounded in Proposition~\ref{prop:optimal_gap}, then the cumulative difference can be bounded as 
\begin{equation*}
    T(Z^* - Z') \leq \frac{10\alpha}{b(1-\alpha)} \sqrt{T\log(2dT^2)}.
\end{equation*}

Since by definition of $t_i=\lceil2^{i/\alpha}\rceil$, the number of rounds that end within the time horizon is bounded by $2^{i/\alpha}\leq T$, \ie, $i\leq\lfloor\alpha\log_2(T) \rfloor$ because $i$ is an integer. Then, the maximum number of rounds is bounded by $i_{\max}=\lceil\alpha\log_2(T)\rceil\leq\alpha\log_2(T)+1$. Recall \eqref{eq:gammai_1/T} that $\gamma_i\geq1/(2T)$, total standard regret can be bounded as
$$
\begin{aligned}
\texttt{Regret}&\leq 1 + T(Z^*-Z') + \sum_{i=1}^{i_{\max}} 4\left(\frac{2C_1}{(1-\alpha)^\beta} + \frac{10\sqrt{2}T}{\min_jB'_j}\right) \gamma_i^{-\alpha / (1 - \alpha)}\log(4d/\gamma_i)^{\max\{\beta,1/2\}}\\
&\leq1 + \frac{10\alpha}{b(1-\alpha)} \sqrt{T\log(2dT^2)} + \sum_{i=1}^{i_{\max}}4\left(\frac{2C_1}{(1-\alpha)^\beta} + \frac{10\sqrt{2}T}{\min_jB_j'} \right)2^{i+1}\log(8dT)^{\max\{\beta,1/2\}}\\
&\leq1 + \frac{10\sqrt{2}\alpha}{b(1-\alpha)} T^{1/2}\log(4dT)^{1/2} + 32\left(\frac{2C_1}{(1-\alpha)^\beta} + \frac{10\sqrt{2}T}{\min_jB_j'} \right)T^{\alpha}\log(8dT)^{\max\{\beta,1/2\}}\\
&\leq1 + \left(32\left(\frac{2C_1}{(1-\alpha)^\beta} + \frac{10\sqrt{2}T}{\min_jB_j'} \right)+\frac{10\sqrt{2}\alpha}{b(1-\alpha)} \right)T^{\alpha}\log(8dT)^{\max\{\beta,1/2\}}.
\end{aligned}
$$
where the first inequality follows from the standard regret bound for each round derived above, the second inequality is based on $\Delta_S^*-\Delta_S'\leq\delta(T)$ from Proposition~\ref{prop:optimal_gap} and $2/\gamma_i\leq4T$ from \eqref{eq:gammai_1/T}, the third inequality is based on $\sum_{i=1}^{i_{\max}}2^i\leq2^{i_{\max}+1}$, and the fourth inequality is plugging into $2^{i_{\max}}\leq2^{\alpha\log_2(T)+1}\leq2T^\alpha$. Recall that if $T$ is sufficiently large with $\sqrt{T/\log(2dT^2)}\geq20\alpha/(b(1-\alpha))$, term $10\sqrt{2}T/(\min_jB_j')$ can be upper bounded by ${20\sqrt{2}}/{b}$, and the satisficing regret can further be upper bounded by
$$
\texttt{Regret}\leq1 + \left(\frac{64C_1}{(1-\alpha)^\beta} + \frac{640\sqrt{2}}{b} + \frac{10\sqrt{2}\alpha}{b(1-\alpha)} \right)T^{\alpha}\log(4dT)^{\max\{\beta,1/2\}}.
$$

\section{Strong Regret Bound for Existing BwK Algorithms}\label{sec:strong_regret_bound_existing_algorithm}

In this section, we show that two existing algorithms for bandits with resource constraints satisfy Condition~\ref{condition:strong_regret_bound}. For both algorithms, let $Z^*(T,\bm B)$ denote the corresponding per-round full-information optimal reward, and define
\[
\mathrm{OPT}:=T Z^*(T,\bm B), \qquad
\mathrm{REW}_S:=\E\left[\sum_{t=1}^T \min\left\{r(\pi_t),Z^*(T,\bm B)\right\}\right].
\]
Here, $\min\{r(\pi_t),Z^*(T,\bm B)\}$ is the strong reward collected at time $t$, which truncates any reward above the LP optimal reward. Therefore,
\[
\mathrm{OPT}-\mathrm{REW}_S
=\E\left[\sum_{t=1}^T\max\left\{Z^*(T,\bm B)-r(\pi_t),0\right\}\right],
\]
which is exactly the expected strong regret in Condition~\ref{condition:strong_regret_bound}. We first establish this property for \texttt{BalancedExploration} in \cite{BKS18}, where $Z^*(T,\bm B)=\mathrm{OPT}_{\mathrm{LP}}/T$, and then for a discrete version of \texttt{DPA} in \cite{wang2014close}, where $Z^*(T,B)=p^D\lambda(p^D)$.

\subsection{Bandits with Knapsacks}\label{sec:balanced_exploration_strong}
We consider the bandit setting given in \cite{BKS18}, where there is a finite arm set denoted as $\mathcal{X}$, $d$ resources that will be consumed under a constraint $B$ for each $i\in[d]$, and a total budget $T$ rounds. Again, we assume $B\geq bT$. As in \cite{BKS18}, after the uniform-budget normalization the time horizon is included among the $d$ resources: its budget is $B$ and its deterministic per-round consumption is $B/T$. In each time $t\in[T]$, if we pick an arm $X_t\in \mathcal{X}$, we will receive a reward $r(X_t)$ and consume $c_{i}(X_t)$ amount resources for each resource $i$, both are in expectation and will be observed with a small noise after each time step $t$. Namely for each arm $X_t$, the reward and resource consumption form an outcome vector $(Y_t;C_t^1(X_t),\cdots,C_t^d(X_t)) \in [0,1]^{d+1}$, which is an i.i.d sample from a fixed distribution with mean $(r(X_t);c_{1}(X_t),\cdots,c_{d}(X_t))\in[0,1]^{d+1}$ unknown to the algorithm. We also write $\mu=((r(X);c_{1}(X),\cdots,c_{d}(X)):X\in\mathcal{X})$ as the latent structure of a problem instance.

As the optimal dynamic policy given foreknowledge of outcome distribution is difficult to derive, we introduce a fractional LP-relaxation. That is, we relax the time steps $t$ from discrete set $[T]$ to continuous interval $[0,T]$, and relax the policy to be a time-invariant distribution denoted as $\pi:\,\mathcal{X}\to[0,1]$. Namely, the LP relaxation can be expressed as an optimization problem as
\begin{equation*}
\begin{aligned}
\max_t \quad& t \, r(\pi)\quad \text{in }t\in\mathbb{R} \\
\text{s.t.} \quad & t \, c_i(\pi) \leq B \quad \text{for each resource } i \\
& 0\leq t \leq T,
\end{aligned}
\end{equation*}
where $r(\pi)=\sum_{X\in \mathcal{X}}\pi(X)r(X)$ and $c_i(\pi)=\sum_{X\in \mathcal{X}}\pi(X)c_i(X)$. The optimal reward is called LP-value of $\pi$ as 
\begin{equation*}
    \text{LP}(\pi) = r(\pi)\min_i\frac{B}{c_i(\pi)},
\end{equation*}
which is linear in $T$ according to the optimization problem. So, we can write $\text{LP}(\pi)=C_{\pi}T$ where $C_{\pi}$ is a fixed constant given $\pi$ and $\mu$, and define $\text{OPT}_{\text{LP}}=\sup_{\pi}\text{LP}(\pi)=T\sup_{\pi}C_{\pi}$, which is the optimal value of the fractional relaxation of the optimal dynamic policy. It is easy to check (also proved in \cite{BKS18}) that $\text{OPT}_{\text{LP}}\geq \text{OPT}$. Also, the distribution $\pi^*\in\arg\max_{\pi}\text{LP}(\pi)$ is called LP-optimal, and holds following properties

\begin{proposition}[Claim 3.4 in \cite{BKS18}]
\textit{For any latent structure $\mu$, there exists a distribution $\pi$ over arms which is LP-optimal for $\mu$ and moreover satisfies the following three properties:}

\begin{enumerate}
    \item $c_i(\pi) \leq B/T$ for each resource $i$.
    \item $\pi$ has a support of size at most $d$.
    \item If $\pi$ has a support of size exactly 2 then for some resource $i$ we have $c_i(\pi) = B/T$.
\end{enumerate}

\noindent Such distribution $\pi$ will be called LP-perfect for $\mu$.
\end{proposition}

We also follow the same high-probability event given in \cite{BKS18} as
\begin{proposition}[Theorem 2.1 in \cite{BKS18}]
    With confidence radius defined to be
    \begin{equation*}
        \text{rad}(\nu,N) = \sqrt{\frac{C_{\text{rad}}\nu}{N}} + \frac{C_{\text{rad}}}{N},
    \end{equation*}
    with $C_{\text{rad}}=\Theta(\log(dT|\mathcal{X}|))$. Then, let $\hat{\nu}$ be the average of $N$ independent samples from a distribution with expectation $\nu$, then
    \begin{equation*}
        \Pr \left[ |\nu - \hat{\nu}| \leq \text{rad}(\hat{\nu}, N) \leq 3 \text{rad}(\nu, N) \right] \geq 1 - e^{-\Omega(C_{\text{rad}})}, \quad \text{for each } C_{\text{rad}} > 0.
    \end{equation*}
\end{proposition}

We recall \texttt{BalancedExploration} given in \cite{BKS18}, which is summarized in Algorithm~\ref{alg:balanced-exploration}.  $\Delta_q$'s are sets of all distribution $\pi$'s in each phase $q$ that are LP-perfect for some latent structure $\mu'\in I_q$. For phase $q=0$, let $I_0$ be the set of feasible latent structures; for each $q\geq1$, the confidence interval $I_q$ is the set of all $\mu'\in I_{q-1}$ that satisfy
\begin{equation*}
\begin{aligned}
|r(x,\mu')-\hat{r}_q(x)| &\leq rad(\hat{r}_q(x),N_q(x)),\\
|c_i(x,\mu')-\hat{c}_{q,i}(x)| &\leq rad(\hat{c}_{q,i}(x),N_q(x)),\quad\forall i\in[d].
\end{aligned}
\end{equation*}

\begin{algorithm}
\caption{\texttt{BalancedExploration}}
\label{alg:balanced-exploration}
\begin{algorithmic}
\For{each phase $q = 0, 1, 2, \ldots$}
    \State Recompute the set $\Delta_q$ of potentially perfect distributions over arms.
    \For{each $X\in \mathcal{X}$, over the next $|\mathcal{X}|$ rounds}
        \State Pick any distribution $\pi_{q,X} \in \Delta_q$ such that $\pi_{q,X}(X) \geq \frac{1}{2} \max_{\pi' \in \Delta_q} \pi'(X)$.
        \State Pull an arm as an independent sample from $\pi_{q,X}$.
        \State \textbf{Halt} if time horizon is met or one of the resources is exhausted.
    \EndFor
\EndFor
\end{algorithmic}
\end{algorithm}

For this standalone subsection, the expected strong reward is $\text{REW}_S=\E[\sum_{t=1}^T\min\{r(\pi_t),\text{OPT}_{\text{LP}}/T\}]$.
\begin{theorem}
Algorithm \texttt{BalancedExploration} given in \cite{BKS18} has an expected strong regret bound the same as its standard regret bound in expectation to be
\begin{equation*}
    \text{OPT}_{\text{LP}} - \text{REW}_S \leq C_1 T^{1/2}\log(T)^2,
\end{equation*}
where $C_1=O((1+b^{-1/2})\sqrt{d|\mathcal{X}|}\sup_\pi C_{\pi})$ is a constant with respect to $T$ for fixed $d,b$.
\label{thm:strong_BwK}
\end{theorem}

\begin{proof}
We first work on a clean execution in the sense of \cite{BKS18}. The core of our proof is based on the fact that LP value of each phase $q$ designated for arm $x$ can be lower bounded by $\text{OPT}_{\text{LP}}$ as 
\begin{proposition}[Claim 4.13 in \cite{BKS18}]
    On clean execution events, for every distribution $\pi_{q,X}$ chosen by the algorithm in phase $q$,
    \begin{equation*}
    \text{LP}(\pi_{q,X}) \geq \text{OPT}_{\text{LP}} - O\left(O(T) \, \text{rad} \left(\text{OPT}_{\text{LP}} / T, \frac{q}{d} \right) + O \left( \text{OPT}_{\text{LP}} \frac{T}{B} \right) \text{rad} \left( \frac{B}{T}, \frac{q}{d} \right)\right),
    \end{equation*}
    \label{Proposition:BwK_Dpx}
\end{proposition}

We then introduce some notations to simplify our proof: $\Phi_q := O(T) \, \text{rad} \left(\text{OPT}_{\text{LP}} / T, \frac{q}{d} \right) + O \left( \text{OPT}_{\text{LP}} \frac{T}{B} \right) \text{rad} \left( \frac{B}{T}, \frac{q}{d} \right)$, \ie, Proposition~\ref{Proposition:BwK_Dpx} implies $\text{LP}(\pi_{q,X}) \geq \text{OPT}_{\text{LP}} - O(\Phi_q)$. Let $p$ be the last phase in the execution and $T_0$ be the stopping time. Letting $m:=|\mathcal{X}|$, we have $pm<T_0\leq(p+1)m$.

For strong reward $\text{REW}_S$, it can be bounded by neglecting the final incomplete phase and only focusing on the first $p$ rounds as
\begin{equation}
\sum_{t=1}^T\min\left\{r(\pi_t),\frac{1}{T}\text{OPT}_{\text{LP}}\right\} \geq \sum_{q<p} \sum_{X\in \mathcal{X}}\min\left\{r(\pi_{q,X}), \frac{1}{T}\text{OPT}_{\text{LP}} \right\}
\label{eq:BwK_REW_def}
\end{equation}
where $\E_{X'\sim \pi_{q,X}}\left[r(X')\right]=\sum_{X'\in \mathcal{X}}\pi_{q,X}(X')r(X')=r(\pi_{q,X})$ and $T_0 \leq T$ is embedded in the second inequality.

Recall that $\text{LP}(\pi_{q,X})=r(\pi_{q,X})\min_iB/c_i(\pi_{q,X})\leq r(\pi_{q,X})\cdot T$, then
\begin{equation*}
    r(\pi_{q,X})\geq\frac{1}{T}\text{LP}(\pi_{q,X})\geq\frac{1}{T}(\text{OPT}_{\text{LP}} - O(\Phi_q)),
\end{equation*}
where the second inequality is based on Proposition~\ref{Proposition:BwK_Dpx}. Note that RHS is irrelevant to which arm $X$ is designated for. We plug this into Equation~\ref{eq:BwK_REW_def}, and get
\begin{equation}
\begin{aligned}
    \sum_{t=1}^T\min\left\{r(\pi_t),\frac{1}{T}\text{OPT}_{\text{LP}}\right\} &\geq \sum_{q<p} \sum_{X\in \mathcal{X}}\min\left\{\frac{1}{T}(\text{OPT}_{\text{LP}} - O(\Phi_q)), \frac{1}{T}\text{OPT}_{\text{LP}} \right\}\\
    &= m\sum_{q<p} \left(\frac{1}{T}\text{OPT}_{\text{LP}} - \frac{1}{T}O(\Phi_q)\right)\\
    &\geq \frac{pm}{T}(\text{OPT}_{\text{LP}} - O(\Phi_p\log p))
    \label{eq:BwK_REW_1}
\end{aligned}
\end{equation}
where the last inequality follows from Claim 4.6(fg) in \cite{BKS18}, which is used to average the confidence radii in $\Phi_q$ over all phases $q<p$.

Till now, we get exactly the same lower bound expression on $\text{REW}_S$ as the one in proving standard reward in Equation (22), in \cite{BKS18}. As the difference between strong and standard reward has nothing to do with the consumption part, we directly use the bound on $pm/T$ given in \cite{BKS18} as
\begin{proposition}[Proof of Lemma 4.17 with Corollary 4.16 in \cite{BKS18}]
    The number of complete phases $p$ can be lower bounded by
    \begin{equation*}
        \frac{pm}{T} \geq 1 - O \left( \frac{pm \log p}{B} \right) \text{rad} \left( \frac{B}{T}, \frac{p}{d}\right).
    \end{equation*}
    \label{Proposition:BwK_pm_T}
\end{proposition}

Plug Proposition~\ref{Proposition:BwK_pm_T} into Equation~\eqref{eq:BwK_REW_1}, we conclude:
\begin{equation*}
\begin{aligned}
\sum_{t=1}^T\min\left\{r(\pi_t),\frac{1}{T}\text{OPT}_{\text{LP}}\right\} &\geq \text{OPT}_{\text{LP}} - O \left( \frac{pm \log p}{B} \right) \text{rad} \left( \frac{B}{T}, \frac{p}{d} \right) \text{OPT}_{\text{LP}} - O \left( \frac{pm \log p}{T} \right) \cdot O(\Phi_p) \\
&\geq \text{OPT}_{\text{LP}} - O \left( \frac{pm \log p}{T} \right) \left( O(\Phi_p) + \frac{T}{B} \text{rad} \left( \frac{B}{T}, \frac{p}{d} \right) \text{OPT}_{\text{LP}} \right) \\
&\geq \text{OPT}_{\text{LP}} - \frac{pm \log p}{T} O(\Phi_p),
\end{aligned}
\end{equation*}
where the final inequality is because $\Phi_q = O(T) \, \text{rad} \left(\text{OPT}_{\text{LP}} / T, \frac{q}{d} \right) + O \left( \text{OPT}_{\text{LP}} \frac{T}{B} \right) \text{rad} \left( \frac{B}{T}, \frac{q}{d} \right)\leq O(\Phi_q)$.

Note that $p \Phi_p \log p$ is increasing in $p$ (by definition of $\Phi_p$), and plug in the upper bound $p \leq \frac{T}{m}$, and we have
\begin{equation*}
\sum_{t=1}^T\min\left\{r(\pi_t),\frac{1}{T}\text{OPT}_{\text{LP}}\right\} \geq \text{OPT}_{\text{LP}} - O \left( \log \frac{T}{|\mathcal{X}|} \right) \Phi_{T / |\mathcal{X}|}.
\end{equation*}

Plugging into the definition of $\Phi_{T/|\mathcal{X}|}$, on the clean event we obtain
\begin{equation}
\begin{aligned}
\text{OPT}_{\text{LP}} - \sum_{t=1}^T \min\left\{r(\pi_t),\frac{1}{T}\text{OPT}_{\text{LP}}\right\}
&\leq O(\log(T)\log(T/m)) \left(\sqrt{dm\,\text{OPT}_{\text{LP}}} + \text{OPT}_{\text{LP}}\sqrt{\frac{dm}{B}} \right)\\
&\leq O\left(\sqrt{dm}\left(1+b^{-1/2}\right) T^{1/2}\log(T)^2 \right),
\end{aligned}
\end{equation}
where the second inequality follows from $\text{OPT}_{\text{LP}}\leq T$ and $B\geq bT$. Finally, the clean event fails with probability at most $T^{-2}$, and the strong regret is always at most
$\text{OPT}_{\text{LP}}\leq T$. Therefore, taking expectation gives
\begin{equation}
\begin{aligned}
\text{OPT}_{\text{LP}}-\text{REW}_S 
&\leq O\left(\sqrt{dm}\left(1+b^{-1/2}\right) T^{1/2}\log(T)^2 \right) +\frac{1}{T}\\
&= O\left(\sqrt{dm}\left(1+b^{-1/2}\right) T^{1/2}\log(T)^2 \right).
\end{aligned}
\end{equation}
\end{proof}

\subsection{Dynamic Pricing with Resource Constraints}\label{sec:discrete_dpa_strong}

We consider a discrete-time version of bandit setting given in \cite{wang2014close}. Similar to \cite{wang2014close}, we consider a seller selling a single product with a fixed inventory at beginning to be $B$ within a time budget $T$. Instead of customers arriving according to a Poisson process with an instantaneous demand rate $\lambda_t$ at any time $t\in[0,T]$ in \cite{wang2014close}, we consider customers arriving according to a Bernoulli process with an arriving probability $\lambda_t$ at each time step $t\in[T]$. 

The same as \cite{wang2014close}, the arriving probability is solely determined by the price offered at $t$, namely we can express $\lambda_t=\lambda(p_t)$ where $p_t$ is the offered price at time $t$. Then, the revenue rate function can be defined as $r_t = r(p_t) = p_t\lambda(p_t)$. We also denote the inverse function of $\lambda$ as $p=\gamma(\lambda)$, then revenue can also be expressed in $\lambda$ as $r(\lambda)=\lambda\gamma(\lambda)$. We assume feasible prices exist in an interval $[\underline{p},\bar{p}]$, $\lambda(p)$ is decreasing in $p$, and $r$ is concave in $\lambda$. The stop price $p_{\infty}$ corresponds to stopping selling, namely $\lambda(p_{\infty})=0$. We also follow assumptions on $\lambda(p)$ and $r(\lambda)$ in \cite{wang2014close} as
\begin{proposition}[Assumption 1 in \cite{wang2014close}]
    For some positive constants $M$, $K$, $m_L$ and $m_U$,
    \begin{enumerate}
        \item Boundedness: $|\lambda(p)| \leq M$ for all $p \in [\underline{p}, \overline{p}]$;
        \item \textit{Lipschitz continuity:} $\lambda(p)$ and $r(\lambda(p))$ are Lipschitz continuous with respect to $p$ with a factor $K$. Also, the inverse demand function $p = \gamma(\lambda)$ is Lipschitz continuous in $\lambda$ with a factor $K$;
        \item \textit{Strict concavity and differentiability:} $r''(\lambda)$ exists and $-m_L \leq r''(\lambda) \leq -m_U < 0$ for all $\lambda$ in the range of $\lambda(p)$ for $p \in [\underline{p}, \overline{p}]$.
    \end{enumerate}
    \label{proposition:CtG_assumption}
\end{proposition}

In order to evaluate the algorithm, \cite{wang2014close} continues to use the expected revenue introduced by \cite{besbes2009dynamic}, defined as $\int_0^Tp(t)\lambda(p(t))\mathrm{d}t$ subject to the inventory constraint $\int_0^T\lambda(p(t))\mathrm{d}t\leq B$. This expected revenue is compared to a full-information deterministic optimization problem where demand $\lambda$ is known, and the optimal (deterministic) solution is given as
\begin{proposition}[According to \cite{wang2014close}, \cite{besbes2009dynamic}]
    \begin{equation*}
        p(t) = p^D = \max\{p^u, p^c \}
    \end{equation*}
    where $p^u = \arg\max_{p\in[\underline{p},\bar{p}]}\{r(\lambda(p)) \}$ and $p^c = \arg\min_{p\in[\underline{p},\bar{p}]}|\lambda(p) - B/T |$
\end{proposition}

Different to \cite{wang2014close}, we define a discrete version of strong reward function for each policy $\pi$ as $\text{REW}_S^{\pi}(B,T;\lambda) = \E[\sum_{t=1}^T\min\{p(t)\lambda(p(t)), p^D\lambda({p^D})\}]$ where the expectation is over price trajectory (for a stochastic policy), and it will be compared to the optimal (strong) reward $\text{REW}_S^D(B,T;\lambda)=Tp^D\lambda(p^D)$, which is reward of applying full-information price $p^D$ at each time step $t$. While in \cite{wang2014close} the asymptotic analysis is based on $B_n=nB$ and $\lambda_n(p)=n\lambda(p)$ with large $n$ and fixed $T=1$, we consider inventory $B$ that grows linearly with total budget $T$ following Assumption~\ref{assum:linear-resource}, \ie, there exists $b$ such that $B\geq bT$, and $\lambda(p)$ remains the same. For notational simplicity, we define $\text{REW}_S := \text{REW}_{S}^\pi(B,T;\lambda)$, $\text{OPT}:=\text{REW}_{S}^D(B,T;\lambda)$, and we want to bound the regret $\text{OPT} - \text{REW}_S$. In particular, ${\rm OPT}=Tp^D\lambda(p^D)$. Therefore
\[
\text{OPT}-\text{REW}_S
=\E\!\left[\sum_{t=1}^T\max\{Z^*(T,B)-r(p(t)),0\}\right],
\]
which is exactly the expected strong regret in Condition~\ref{condition:strong_regret_bound}.

We will use and analyze the discrete version of Algorithm DPA given in \cite{wang2014close}. Algorithm DPA can be divided into three steps. In the first step, learn $p^u$ or determine $p^c>p^u$. If $p^c>p^u$, explore $p^c$ in the second step. Otherwise, skip the second step. After exploration in $p^u$ and $p^c$ (if needed), repeat the estimated best price in the final step until exhausting time budget or resources.

To be specific, in both exploration stages, price $p^u$ or $p^c$ are estimated round by round. In each round, learning period $\tau_i^u$ or $\tau_i^c$ and price confidence interval $I_i^u:=[\underline{p}_i^u,\bar{p}_i^u]$ or $I_i^c:=[\underline{p}_i^c,\bar{p}_i^c]$ are both equally divided into $\kappa_i^u$ or $\kappa_i^c$ interval. Because we propose a discrete version, the length of divided time period is rounded up as $\Delta_i^u:=\lceil T\tau_i^u/\kappa_i^u \rceil$. The left endpoints of each price intervals are repeated for corresponding learning period. Namely, at round $i$ for $j=1,2,\ldots,\kappa_i^u$ in exploring $p^u$, we set price and time endpoints to be
\begin{equation}
\begin{aligned}
    p_{i,j}^u &= \underline{p}_i^u + (j-1)\cdot\frac{\bar{p}_i^u-\underline{p}_i^u}{\kappa_i^u}\\
    t_{i,j}^u &= t_{i-1}^u + j \Delta_i^u,
    \label{eq:CtG_interval_divide}
\end{aligned}
\end{equation}
and at round $i$ for $j=1,2,\ldots,\kappa_i^c$ in exploring $p^c$, we set price and time endpoints to be
\begin{equation}
\begin{aligned}
    p_{i,j}^c &= \underline{p}_i^c + (j-1)\cdot\frac{\bar{p}_i^c-\underline{p}_i^c}{\kappa_i^c}\\
    t_{i,j}^c &= t_{i-1}^c + j \Delta_i^c + t_{i_0}^u.
    \label{eq:CtG_interval_divide_c}
\end{aligned}
\end{equation}
In each round $i$, after exploring $\kappa_i^u$ or $\kappa_i^c$ prices, the demand of $j$-th price is estimated to be
\begin{equation*}
     \hat{d}(p_{i,j}^u) = \frac{\text{total demand over } [t_{i,j-1}^u, t_{i,j}^u]}{\Delta_i^u}.
\end{equation*}

Then, estimate $p^u$ (and $p^c$), check if explore enough, and update price interval for round $i+1$. This is slightly different between exploring $p^u$ and $p^c$. In exploring $p^u$, $p^u$ and $p^c$ are estimated to be
\begin{equation}
\begin{aligned}
     \hat{p}_i^u &= \arg\max_{p\in\{p_{i,j}^u:j\in[\kappa_i^u]\}} p \hat{d}(p) \\
     \hat{p}_i^c &= \arg\min_{p\in\{p_{i,j}^u:j\in[\kappa_i^u]\}} \left|\hat d(p)-B/T\right|,
     \label{eq:CtG_price_estimate_u}
\end{aligned}
\end{equation}
and price interval $I_{i+1}^u$ for round $i+1$ is updated to be
\begin{equation}
\begin{aligned}
    \underline{p}_{i+1}^u &= \hat{p}_i - \frac{\log(T)}{3} \cdot \frac{\bar{p}_i^u - \underline{p}_i^u}{\kappa_i^u}\\
    \bar{p}_{i+1}^u &= \hat{p}_i + \frac{2 \log(T)}{3} \cdot \frac{\bar{p}_i^u - \underline{p}_i^u}{\kappa_i^u}.
    \label{eq:CtG_interval_update_u}
\end{aligned}
\end{equation}
In exploring $p^c$, $p^c$ is estimated to be
\begin{equation}
    \hat q_i \in \arg\min_{p\in\{p_{i,j}^c:j\in[\kappa_i^c]\}} \left|\hat d(p)-B/T\right|,
    \label{eq:CtG_price_estimate_c}
\end{equation}
and price interval $I_{i+1}^c$ for round $i+1$ is updated to be
\begin{equation}
\begin{aligned}
    \underline{p}_{i+1}^c &= \hat{q}_i - \frac{\log(T)}{2} \cdot \frac{\bar{p}_i^c - \underline{p}_i^c}{\kappa_i^c}\\
    \bar{p}_{i+1}^c &= \hat{q}_i + \frac{\log(T)}{2} \cdot \frac{\bar{p}_i^c - \underline{p}_i^c}{\kappa_i^c}.
    \label{eq:CtG_interval_update_c}
\end{aligned}
\end{equation}

We summarize the Discrete DPA in Algorithm~\ref{alg:DPA}.
\begin{algorithm}[!ht]
\caption{\texttt{Discrete DPA}}
\label{alg:DPA}
\begin{algorithmic}
\State \textbf{input:} a sequence of $\tau_i^u$, $\kappa_i^u$ for $i=1,2,\ldots,N^u$ and $\tau_i^c$, $\kappa_i^c$ for $i=1,2,\ldots,N^c$
\State \textit{Step 1:} \Comment{Learn $p^u$ or determine $p^c>p^u$}
\For{$i=1,2,\ldots,N^u$}
    \State Set $p_{i,j}^u$ and $t_{i,j}^u$ for $j = 0, 1, \ldots, \kappa_i^u$ according to Equations~\eqref{eq:CtG_interval_divide}
    \For{$j=1,2,\ldots,\kappa_i^u$}
        \State Apply $p_{i,j}^u$ for $\Delta_i^u=\lceil T\tau_i^u/\kappa_i^u \rceil$ steps; Or apply $p_{\infty}$ until time $T$ if inventory runs out
    \EndFor
    \State Estimate $\hat{p}_i^u$ and $\hat{p}_i^c$ according to Equations~\eqref{eq:CtG_price_estimate_u}
    \If{$ \hat{p}_i^c > \hat{p}_i^u + 2 \sqrt{\log (T)} \cdot \frac{\bar{p}_i^u - \underline{p}_i^u}{\kappa_i^u}$}
    \State Set $i_0=i$. Go to Step 2.
    \Else
    \State Set $\hat{p}_i = \max\{\hat{p}_i^c, \hat{p}_i^u\}$ and update $I_{i+1}^u$ according to Equations~\eqref{eq:CtG_interval_update_u}
    \EndIf
    \If{$i=N^u$}
    \State Go to Step 3 with learned price $\tilde{p} = \hat{p}_{N^u} + 2 \sqrt{\log (T)} \cdot \frac{\bar{p}_{N^u}^u - \underline{p}_{N^u}^u}{\kappa_{N^u}^u}$
    \EndIf
\EndFor
\State \textit{Step 2:} \Comment{Learn $p^c$ when $p^c>p^u$}
\For{$i=1,2,\ldots,N^c$}
    \State Set $p_{i,j}^c$ and $t_{i,j}^c$ for $j = 0, 1, \ldots, \kappa_i^c$ according to Equations~\eqref{eq:CtG_interval_divide_c}
    \For{$j=1,2,\ldots,\kappa_i^c$}
        \State Apply $p_{i,j}^c$ for $\Delta_i^c=\lceil T\tau_i^c/\kappa_i^c \rceil$ steps; Or apply $p_{\infty}$ until time $T$ if inventory runs out
    \EndFor
    \State Estimate $\hat{q}_i$ according to Equations~\eqref{eq:CtG_price_estimate_c} and update $I_{i+1}^c$ according to Equations~\eqref{eq:CtG_interval_update_c}
    \If{$i=N^c$}
    \State Go to Step 3 with learned price $\tilde{q} = \hat{q}_{N^c}$
    \EndIf
\EndFor
\State \textit{Step 3:} \Comment{Apply learned price}
\State Apply the learned price ($\tilde{p}$ or $\tilde{q}$) until inventory or time budget run out
\end{algorithmic}
\end{algorithm}

Algorithm \ref{alg:DPA} depends on pre-designed sequences of $\tau_i^u$, $\kappa_i^u$, $\tau_i^c$, $\kappa_i^c$, and value $N^u$, $N^c$. We use the same design as the one given in \cite{wang2014close} (Equations 18 to 31) so that some of the results can be directly borrowed. Note that an important difference between Algorithm \ref{alg:DPA} and \texttt{DPA} given in \cite{wang2014close} is that each price $p_{i,j}^u$ is applied for $\lceil T\tau_i^u/\kappa_i^u\rceil$ time steps instead of a continuous time interval $\tau_i^u/\kappa_i^u$. In the implementation, each displayed $\kappa_i^u$ and $\kappa_i^c$ is rounded to a positive integer. This changes all bounds below by constant factors only, so we keep the same notation.
\[
\tau_1^u = T^{-\frac{1}{2}}\cdot \log (T)^{3.5} \quad \text{and} \quad \tau_i^u = T^{-\frac{1}{2} \cdot \left(\frac{3}{5}\right)^{i-1}} \cdot \log (T)^5, \quad \forall i = 2, \ldots, N^u,
\]
\[
\kappa_i^u = T^{\frac{1}{10} \cdot \left(\frac{3}{5}\right)^{i-1}} \cdot \log (T), \quad \forall i = 1, 2, \ldots, N^u,
\]
\[
\bar{p}_i^u - \underline{p}_i^u = T^{-\frac{1}{4} \left(1 - \left(\frac{3}{5}\right)^{i-1}\right)}, \quad \forall i = 1, 2, \ldots, N^u,
\]
$$
\text{where } N^u:=\min_l\left\{l:\left(\frac{\bar{p}_l^u - \underline{p}_l^u}{\kappa_l^u} \right)^2\sqrt{\log (T)}<\tau_1^u  \right\}.
$$

Similarly, for $p^c$, we have
\[
\tau_1^c = T^{-\frac{1}{2}} \cdot \log (T)^{2.5} \quad \text{and} \quad \tau_i^c = T^{-\frac{1}{2} \cdot \left(\frac{2}{3}\right)^{i-1}} \cdot \log (T)^3, \quad \forall i = 2, \ldots, N^c, 
\]
\[
\kappa_i^c = T^{\frac{1}{6} \cdot \left(\frac{2}{3}\right)^{i-1}} \cdot \log (T), \quad \forall i = 1, 2, \ldots, N^c.
\]
\[
\bar{p}_i^c - \underline{p}_i^c = T^{-\frac{1}{2} \left( 1 - \left( \frac{2}{3} \right)^{i-1} \right)}, \quad \forall i = 1, \ldots, N^c.
\]
$$
\text{where } N^c:=\min_l\left\{l:\frac{\bar{p}_l^c - \underline{p}_l^c}{\kappa_l^c} \sqrt{\log (T)}<\tau_1^c  \right\}.
$$

For $T<3$, define \texttt{Discrete DPA} to post $p_\infty$ throughout; below we analyze $T\geq3$.

\begin{theorem}
    Algorithm~\ref{alg:DPA} \texttt{Discrete DPA}  has an expected strong regret bound as
    \begin{equation}
    \text{OPT} - \text{REW}_S\leq C\sqrt T\log(T)^6.
    \end{equation}
\label{thm:strong_dpa}
\end{theorem}
\begin{proof}
We first bound the round number $N^u$ and $N^c$ by
\begin{proposition}[Lemma 2 in \cite{wang2014close}]
    $N^u$ and $N^c$ defined above satisfy: $N^u\leq3\log (T)$ and $N^c\leq3\log (T)$, for $T\geq3$.
\end{proposition}

In order to make sure that $p^D$ is always contained in the price interval with high probability, a concentration inequality should first be derived as
\begin{proposition}[Lemma 11 in \cite{wang2014close}]
    The arrival rate $\hat{\lambda}$ for each price $p_{i,j}^u$ can be estimated with high accuracy to be
    \begin{equation*}
        \Pr \left( \left| \hat{\lambda}(p_{i,j}^u) - \lambda(p_{i,j}^u) \right| > C \sqrt{\log (T)} \cdot \sqrt{\frac{\kappa_i^u}{T \tau_i^u}} \right) \leq \frac{1}{T^2}.
    \end{equation*}
    The same inequality holds for $p_{i,j}^c$ with $\kappa_i^c$ and $\tau_i^c$.
\end{proposition}
Although this has exactly the same form as the one in \cite{wang2014close}, it is worth mentioning that its derivation is slightly different. In original paper, the analysis is based on estimating $n\lambda(p_{i,j}^u)$ within a time period $\tau_i^u/\kappa_i^u$ asymptotic in $n$. But in our algorithm, we estimate $\lambda(p_{i,j}^u)$ within a time $\lceil T\tau_i^u/\kappa_i^u\rceil$ asymptotic in $T$, which can lead to a stronger bound than estimating within a time $T\tau_i^u/\kappa_i^u$ (which is equivalent to the one in original paper). We loosen this stronger bound and retain the same form. Then, we propose the high probability event as
\begin{proposition}[Lemma 3 in \cite{wang2014close}]
    Assume $p^D$ is the optimal price for the deterministic problem and Assumption 1 holds for $\Gamma = \Gamma(M, K, m_L, m_U)$. Define $A_0$ to be the event satisfying the following conditions:
\begin{enumerate}
    \item \textit{If we never enter Step 2, then $p^D \in I_i^u$ for all $i = 1, 2, \ldots, N^u$;}
    \item \textit{If Step 1 stops at $i_0$ and the algorithm enters Step 2, then $p^c \geq p^u$, $p^D \in I_i^u$ for all $i = 1, 2, \ldots, i_0$ and $p^D \in I_j^c$ for all $j = 1, 2, \ldots, N^c$.}
\end{enumerate}
Then $\Pr(A_0) = 1 - O\left(\frac{1}{T}\right)$.
\end{proposition}
\noindent We also use the events defined in \cite{wang2014close}: define $B_k:=\{i_0=k \}$ for $k\in[N^u]$ to denote the event that Step 1 jumps to Step 2 at round $k$, and define $B_{N^u+1}=\{\text{Algorithm never enters Step }2 \}$; $A_1 := \left\{ \omega \in A_0 : \sum_{i=1}^{N^u} \sum_{j=1}^{\kappa_i^u} \lambda(p_{i,j}^u)\Delta_i^u< B \right\}$, $A_2 = \left\{ \omega \in A_0 : \sum_{i=1}^{N^u} \sum_{j=1}^{\kappa_i^u} \lambda(p_{i,j}^u)\Delta_i^u + \sum_{i=1}^{N^c} \sum_{j=1}^{\kappa_i^c} \lambda(p_{i,j}^c)\Delta_i^c < B \right\}$.

We also provide some important relationships among the parameters
\begin{proposition}[Equations 18 to 25 in \cite{wang2014close}]
    When entering Step 3, the length of updated price intervals can be bounded by 
    \begin{equation*}
    \begin{aligned}
        &(\bar{p}_{N^u+1}^u - \underline{p}_{N^u+1}^u)^2\leq C\tau_1^u\\
        &\bar{p}_{N^c+1}^c - \underline{p}_{N^c+1}^c\leq C\tau_1^c
    \end{aligned}
    \end{equation*}
    \label{prop:bounded_final_interval}
\end{proposition}

Denote the clipped strong reward $g(p):=\min\{p\lambda(p),p^D\lambda(p^D)\}$ and then $\rho(p):=g(p)/\lambda(p)\cdot\bm{1}(\lambda(p)>0)$. Then $0\leq\rho(p)\leq p\leq\bar p$ and $\rho(p)\lambda(p)=g(p)$. For each exploration block, let $Y_{i,j}^u$ and $Y_{i,j}^c$ denote the corresponding total Bernoulli arrivals. Let $\widehat Y^u$ denote the Bernoulli arrivals that would occur during the final $\tilde p$ period, and let $\widehat Y_l^c$ denote the arrivals that would occur during the final $\tilde q$ period on $B_l$. Then, the total number of exploration steps is upper bounded by $T \sum_i(\tau_i^u+\tau_i^c) + \sum_i(\kappa_i^u+\kappa_i^c) = o(T)$. Hence, because each Bernoulli arrival consumes at most one unit and $B\geq bT$, for all sufficiently large $T$ inventory cannot be exhausted during exploration. We prove the bound below for such $T$ and all smaller horizons are absorbed by increasing the final constant $C_1$.

On $A_0$, the deterministic comparisons used in Lemmas 4--7 of \cite{wang2014close} give
\begin{proposition}[Lemmas 4--7 in \cite{wang2014close}] 
On $A_0$, the deterministic comparisons give
\begin{equation}
\begin{aligned}
g(p_{i,j}^u) &\geq p^D\lambda(p^D)-C(\bar p_i^u-\underline p_i^u)^2,
&&\text{for every reached $u$-round},\\
g(\tilde p) &\geq p^D\lambda(p^D)-C\tau_1^u,
&&\text{on }B_{N^u+1},\\
g(p_{i,j}^c) &\geq p^D\lambda(p^D)-C(\bar p_i^c-\underline p_i^c),
&&\text{after switching to Step 2},\\
g(\tilde q) &\geq p^D\lambda(p^D)-C\tau_1^c,
&&\text{on }\bigcup_{l=1}^{N^u}B_l.
\end{aligned}
\label{eq:CtG_local_strong_bounds}
\end{equation}
\end{proposition}

\begin{proposition}[Discrete counterparts of Lemmas 5 and 7 in \cite{wang2014close}]
The inventory-overrun terms satisfy
\begin{equation}
\begin{aligned}
&\E\left[\left( \widehat Y^u + \sum_{i=1}^{N^u}\sum_{j=1}^{\kappa_i^u}Y_{i,j}^u-B \right)^ + I(B_{N^u+1})\right]
\leq C\sqrt{T}\log^6 T,\\
&\sum_{l=1}^{N^u} \E\left[\left(\widehat Y_l^c + \sum_{i=1}^{l}\sum_{j=1}^{\kappa_i^u}Y_{i,j}^u + \sum_{i=1}^{N^c}\sum_{j=1}^{\kappa_i^c}Y_{i,j}^c-B \right)^+ I(B_l) \right]
\leq C\sqrt{T}\log^6 T.
\end{aligned}
\label{eq:CtG_discrete_overrun}
\end{equation}
\end{proposition}
The predictable-mean part is the same as in the proofs of Lemmas 5 and 7, because it only uses the demand probabilities and the DPA price path. Replacing Poisson arrivals by Bernoulli arrivals changes the stochastic part to a martingale with conditional variance at most $T$, so its expected absolute fluctuation is at most $\sqrt T$. Finally, the ceiling operation contributes at most $C\sum_i(\kappa_i^u+\kappa_i^c)\leq C T^{1/6}\log^2T$. All three terms are bounded by the displayed right-hand side.

Similar to Equation (32) in \cite{wang2014close}, the strong reward can be bounded by
\begin{equation}
\begin{aligned}
\text{REW}_S
\geq& \E \left[ \sum_{i=1}^{N^u} \sum_{j=1}^{\kappa_i^u} g(p_{i,j}^u)\Delta_i^u I \left( \bigcup_{l=i}^{N^u+1} B_l \right) \right] \\
+& \E \left[ \rho(\tilde p)\min\left\{\widehat Y^u,\left(B-\sum_{i=1}^{N^u}\sum_{j=1}^{\kappa_i^u}Y_{i,j}^u\right)^+\right\}I(B_{N^u+1})\right]\\
+& \E \left[ \sum_{i=1}^{N^c} \sum_{j=1}^{\kappa_i^c} g(p_{i,j}^c)\Delta_i^c I \left( \bigcup_{l=1}^{N^u} B_l \right) \right] \\
+& \sum_{l=1}^{N^u}\E \left[ \rho(\tilde q)\min\left\{\widehat Y_l^c,\left(B-\sum_{i=1}^{l}\sum_{j=1}^{\kappa_i^u}Y_{i,j}^u-\sum_{i=1}^{N^c}\sum_{j=1}^{\kappa_i^c}Y_{i,j}^c\right)^+\right\}I(B_l)\right].
\end{aligned}
\label{eq:Close_the_Gap}
\end{equation}
The first term of Equation~\eqref{eq:Close_the_Gap} is the cumulative strong reward in exploring $p^u$, denoted as $\E[\text{Explore }p^u\text{ Reward}]$. The second term is reward in repeating $\tilde{p}$, denoted as $\E[\text{Exploit }\tilde{p}\text{ Reward}]$. The third term is reward in exploring $p^c$, denoted as $\E[\text{Explore }p^c\text{ Reward}]$. The fourth term is reward in repeating $\tilde{q}$, denoted as $\E[\text{Exploit }\tilde{q}\text{ Reward}]$. Namely, the total reward can be divided into
\begin{equation}
\begin{aligned}
    \text{REW}_S \geq& \E[\text{Explore }p^u\text{ Reward}] + \E[\text{Exploit }\tilde{p}\text{ Reward}]\\
    &+ \E[\text{Explore }p^c\text{ Reward}] + \E[\text{Exploit }\tilde{q}\text{ Reward}].
\label{eq:Close_the_Gap_simple}
\end{aligned}
\end{equation}
We will bound these four terms one by one.

\noindent\textbf{First term of Equation~\eqref{eq:Close_the_Gap_simple}} is the cumulative strong reward in exploring $p^u$. We have
\begin{equation}
\begin{aligned}
\E[\text{Explore }p^u\text{ Reward}]
&\geq \sum_{i=1}^{N^u}T\tau_i^u p^D\lambda(p^D)\Pr\left(\bigcup_{l=i}^{N^u+1}B_l\right)-CT\sum_{i=1}^{N^u}\tau_i^u(\bar p_i^u-\underline p_i^u)^2-C\\
&=\sum_{i=1}^{N^u}T\tau_i^u p^D\lambda(p^D)\Pr\left(\bigcup_{l=i}^{N^u+1}B_l\right)-C\left(\sqrt{T}\log^{3.5}T+(N^u-1)\sqrt{T}\log^5T\right)-C\\
&\geq \sum_{i=1}^{N^u}T\tau_i^u p^D\lambda(p^D)\Pr\left(\bigcup_{l=i}^{N^u+1}B_l\right)-C\sqrt{T}\log^6T.
\end{aligned}
\label{eq:Close_the_Gap_1_bound}
\end{equation}
The first inequality follows from $\Delta_i^u\geq T\tau_i^u/\kappa_i^u$, Equation~\eqref{eq:CtG_local_strong_bounds}, and $\Pr(A_0^c)=O(1/T)$. The equality follows from $T\tau_1^u(\bar p_1^u-\underline p_1^u)^2=\sqrt{T}\log^{3.5}T$ and $T\tau_i^u(\bar p_i^u-\underline p_i^u)^2=\sqrt{T}\log^5T$ for $i\geq2$. The last inequality follows from $N^u\leq3\log T$.

\noindent\textbf{Second term of Equation~\eqref{eq:Close_the_Gap_simple}} is the cumulative strong reward in exploiting $\tilde p$. We have
\begin{equation}
\begin{aligned}
\E[\text{Exploit }\tilde p\text{ Reward}]
&\geq \E\left[\rho(\tilde p)\widehat Y^u I(B_{N^u+1})\right]-\bar p\,\E\left[\left(\widehat Y^u+\sum_{i,j}Y_{i,j}^u-B\right)^+I(B_{N^u+1})\right]\\
&=\E\left[g(\tilde p)(T-t_{N^u}^u)I(B_{N^u+1})\right]-\bar p\,\E\left[\left(\widehat Y^u+\sum_{i,j}Y_{i,j}^u-B\right)^+I(B_{N^u+1})\right]\\
&\geq (p^D\lambda(p^D)-C\tau_1^u)\left(T\left(1-\sum_{i=1}^{N^u}\tau_i^u\right)-\sum_{i=1}^{N^u}\kappa_i^u\right)\Pr(B_{N^u+1})\\
&\quad-C-\bar p\,\E\left[\left(\widehat Y^u+\sum_{i,j}Y_{i,j}^u-B\right)^+I(B_{N^u+1})\right]\\
&\geq T\left(1-\sum_{i=1}^{N^u}\tau_i^u\right)p^D\lambda(p^D)\Pr(B_{N^u+1})-C\sqrt T\log^6T.
\end{aligned}
\label{eq:Close_the_Gap_2_bound}
\end{equation}
The first inequality follows from $\min\{a,(B-c)^+\}\geq a-(a+c-B)^+$ and $\rho(\tilde p)\leq\bar p$. The equality follows by conditioning on the history before the final period and using $\rho(\tilde p)\lambda(\tilde p)=g(\tilde p)$. The second inequality follows from Equation~\eqref{eq:CtG_local_strong_bounds}, $\Pr(A_0^c)=O(1/T)$, and $t_{N^u}^u\leq\sum_{i=1}^{N^u}(T\tau_i^u+\kappa_i^u)$. The last inequality follows from Equation~\eqref{eq:CtG_discrete_overrun}, $T\tau_1^u=\sqrt T\log^{3.5}T$, and $\sum_{i=1}^{N^u}\kappa_i^u\leq CT^{1/10}\log^2T$.

\noindent\textbf{Third term of Equation~\eqref{eq:Close_the_Gap_simple}} is the cumulative strong reward in exploring $p^c$. We have
\begin{equation}
\begin{aligned}
\E[\text{Explore }p^c\text{ Reward}]
&\geq \sum_{i=1}^{N^c}T\tau_i^c\left(p^D\lambda(p^D)-C(\bar p_i^c-\underline p_i^c)\right)\Pr\left(\bigcup_{l=1}^{N^u}B_l\right)-C\\
&=\sum_{i=1}^{N^c}T\tau_i^cp^D\lambda(p^D)\Pr\left(\bigcup_{l=1}^{N^u}B_l\right)\\
&\quad-C\left(\sqrt T\log^{2.5}T+(N^c-1)\sqrt T\log^3T\right)-C\\
&\geq \sum_{i=1}^{N^c}T\tau_i^cp^D\lambda(p^D)\Pr\left(\bigcup_{l=1}^{N^u}B_l\right)-C\sqrt T\log^4T.
\end{aligned}
\label{eq:Close_the_Gap_3_bound}
\end{equation}
The first inequality follows from $\Delta_i^c\geq T\tau_i^c/\kappa_i^c$, Equation~\eqref{eq:CtG_local_strong_bounds}, and $\Pr(A_0^c)=O(1/T)$. The equality follows from $T\tau_1^c(\bar p_1^c-\underline p_1^c)=\sqrt T\log^{2.5}T$ and $T\tau_i^c(\bar p_i^c-\underline p_i^c)=\sqrt T\log^3T$ for $i\geq2$. The last inequality follows from $N^c\leq3\log T$.

\noindent\textbf{Fourth term of Equation~\eqref{eq:Close_the_Gap_simple}} is the cumulative strong reward in exploiting $\tilde q$. We have
\begin{equation}
\begin{aligned}
\E[\text{Exploit }\tilde q\text{ Reward}]
&\geq \sum_{l=1}^{N^u}\E\left[\rho(\tilde q)\widehat Y_l^c I(B_l)\right]-\bar p\sum_{l=1}^{N^u}\E\left[\left(\widehat Y_l^c+\sum_{i=1}^{l}\sum_jY_{i,j}^u+\sum_{i=1}^{N^c}\sum_jY_{i,j}^c-B\right)^+I(B_l)\right]\\
&=\sum_{l=1}^{N^u}\E\left[g(\tilde q)(T-t_{N^c}^c-t_l^u)I(B_l)\right] \\
&\quad - \bar p\sum_{l=1}^{N^u}\E\left[\left(\widehat Y_l^c+\sum_{i=1}^{l}\sum_jY_{i,j}^u+\sum_{i=1}^{N^c}\sum_jY_{i,j}^c-B\right)^+I(B_l)\right]\\
&\geq \sum_{l=1}^{N^u}(p^D\lambda(p^D)-C\tau_1^c)\left(T\left(1-\sum_{i=1}^{l}\tau_i^u-\sum_{i=1}^{N^c}\tau_i^c\right)-\sum_{i=1}^{l}\kappa_i^u-\sum_{i=1}^{N^c}\kappa_i^c\right)\Pr(B_l)-C\\
&\qquad-\bar p\sum_{l=1}^{N^u}\E\left[\left(\widehat Y_l^c+\sum_{i=1}^{l}\sum_jY_{i,j}^u+\sum_{i=1}^{N^c}\sum_jY_{i,j}^c-B\right)^+I(B_l)\right]\\
&\geq \sum_{l=1}^{N^u}Tp^D\lambda(p^D)\left(1-\sum_{i=1}^{l}\tau_i^u-\sum_{i=1}^{N^c}\tau_i^c\right)\Pr(B_l)-C\sqrt T\log^6T.
\end{aligned}
\label{eq:Close_the_Gap_4_bound}
\end{equation}
The first inequality follows from $\min\{a,(B-c)^+\}\geq a-(a+c-B)^+$ and $\rho(\tilde q)\leq\bar p$. The equality follows by conditioning on the history before the final period and using $\rho(\tilde q)\lambda(\tilde q)=g(\tilde q)$. The second inequality follows from Equation~\eqref{eq:CtG_local_strong_bounds}, $\Pr(A_0^c)=O(1/T)$, and $t_l^u+t_{N^c}^c\leq T\left(\sum_{i=1}^{l}\tau_i^u+\sum_{i=1}^{N^c}\tau_i^c\right)+\sum_{i=1}^{l}\kappa_i^u+\sum_{i=1}^{N^c}\kappa_i^c$. The last inequality follows from Equation~\eqref{eq:CtG_discrete_overrun}, $T\tau_1^c=\sqrt T\log^{2.5}T$, and $\sum_i(\kappa_i^u+\kappa_i^c)\leq CT^{1/6}\log^2T$.

\noindent\textbf{Combining all terms:} Combining Equations~\eqref{eq:Close_the_Gap_1_bound}--\eqref{eq:Close_the_Gap_4_bound}, we have
\begin{equation}
\begin{aligned}
\text{REW}_S
&\geq Tp^D\lambda(p^D)\Bigg[
\sum_{i=1}^{N^u}\tau_i^u\Pr\left(\bigcup_{l=i}^{N^u+1}B_l\right)
+\left(1-\sum_{i=1}^{N^u}\tau_i^u\right)\Pr(B_{N^u+1})\\
&\quad
+\left(\sum_{i=1}^{N^c}\tau_i^c\right)\Pr\left(\bigcup_{l=1}^{N^u}B_l\right)
+\sum_{l=1}^{N^u}\left(1-\sum_{i=1}^{l}\tau_i^u-\sum_{i=1}^{N^c}\tau_i^c\right)\Pr(B_l)
\Bigg]-C\sqrt T\log^6T\\
&=Tp^D\lambda(p^D)\sum_{l=1}^{N^u+1}\Pr(B_l)-C\sqrt T\log^6T\\
&=\text{OPT}-C\sqrt T\log^6T.
\end{aligned}
\label{eq:Close_the_Gap_final_reward}
\end{equation}
The first equality follows by collecting the coefficient of each $\Pr(B_l)$, which is one for every $l=1,\ldots,N^u+1$, and the second equality follows because $B_1,\ldots,B_{N^u+1}$ form a partition and $\text{OPT}=Tp^D\lambda(p^D)$. Therefore,
\begin{equation}
\text{OPT}-\text{REW}_S\leq C\sqrt T\log^6T.
\label{eq:Close_the_Gap_final}
\end{equation}

\end{proof}

\end{appendices}

\end{document}